\documentclass[twoside]{article}
\usepackage[accepted]{aistats2025}

\usepackage{tabularx}
\usepackage[utf8]{inputenc} % allow utf-8 input
\usepackage[T1]{fontenc}    % use 8-bit T1 fonts
\usepackage{hyperref}       % hyperlinks
\usepackage{url}            % simple URL typesetting
\usepackage{booktabs}       % professional-quality tables
\usepackage{amsfonts}       % blackboard math symbols
\usepackage{nicefrac}       % compact symbols for 1/2, etc.
\usepackage{microtype}      % microtypography
\usepackage{xcolor}         % colors

\usepackage{graphicx} % Required for inserting images
\usepackage{xspace} % for adding space or not based on requirement
\usepackage{amsmath}
\usepackage{amsthm}
\usepackage{amssymb}
\usepackage{mathtools}
\usepackage{hyperref}
\usepackage{bm}
\usepackage{setspace} % for "\setstretch" macro
\usepackage{cleveref}
\usepackage{algorithmic}
\usepackage{algorithm}
\usepackage[justification=centering]{subfig}
\usepackage{multirow}
\usepackage{array}
\usepackage{xcolor}
\usepackage{cases}
\usepackage{soul}
\usepackage{bbm}

\usepackage{tabularx}
\usepackage{arydshln} %added for dotted line in table

\graphicspath{ {./images/} }
\newcommand{\mf}[1]{\mathbf{#1}} %math font
\newcommand{\mc}[1]{\mathcal{#1}} %mathcal
\newcommand{\lw}[2]{\mathbf{x}_{#1}^{#2}} %local weight (x_t,k ^i) 
\newcommand{\gw}[1]{\mathbf{x}_{#1}} % global weight (x_t)
\newcommand{\h}[2]{h_{#1}(\mathbf{#2})}
\newcommand{\alp}[1]{\alpha_{#1}} % alpha_{h_t}
\newcommand{\bet}[1]{\beta_{#1}} % beta_{h_t}
\newcommand{\lmh}[1][\tilde]{#1{h}_{i}} %# local memory h_i on M^i U C^i
\newcommand{\mh}[1][\tilde]{#1{h}} %# memory h on M U C
\newcommand{\grad}[2]{\nabla {#1}(\gw{#2})} % f_i(x_t)
\newcommand{\mgrad}[4]{\nabla {#1}_{#2}^{\dagger}(\lw{#3}{#4})} % \grad f_i^{\dagger}(X_t,k^i) , f or g  - memory grad - {f}{i}{t,k}{}
\newcommand{\dgrad}[4]{\nabla {#1}_{#2}^{\prime}(\lw{#3}{#4})} % delayed grad \grad g_i^{'}(X_t,k^i) , g  - delayed grad - {g}{i}{t,k}{}
\newcommand{\lgrad}[3]{\nabla {#1_{#3}}{(\lw{#2}{#3})} } % \nabla g_i(X_t,k ^i) - 1/2/3 - g, (t,k), i

\newcommand{\lmhgrad}[2]{\nabla {\lmh}^{\prime}(\lw{#1}{#2})}
\newcommand{\bias}[1][t]{B(#1)}
\newcommand{\error}[2]{\mathbf{e}_{#1}^{#2}} % e_{t,k}^i
\newcommand{\avgerror}[1]{\bar{\mathbf{e}}_{#1}}

\newcommand\norm[1]{\left\lVert#1\right\rVert} %norm
\newcommand{\ub}[2]{{\underbrace{#1}_{{#2}}}} % gives underbraces
\newcommand{\sumk}[2]{\sum\limits_{#1}^{#2}} % sum symbol. if upper is not required keep blank while calling
\newcommand{\tsumk}[2]{{\textstyle\sum}_{#1}^{#2}} % text style sum symbol. if upper is not required keep blank while calling
\newcommand{\newstuff}[1]{{\color{blue}#1}}

\DeclarePairedDelimiter\inner{\langle}{\rangle} % for norm, use * to adjust it's size automatically
\DeclareMathOperator*{\argmax}{argmax}

\newtheorem{theorem}{Theorem}
\newtheorem{lemma}[theorem]{Lemma}
\newtheorem{assump}{Assumption}

\newtheorem{corollary}[theorem]{Corollary}
\usepackage[round]{natbib}

\begin{document}

% If your paper is accepted and the title of your paper is very long,
% the style will print as headings an error message. Use the following
% command to supply a shorter title of your paper so that it can be
% used as headings.
%
%\runningtitle{I use this title instead because the last one was very long}

% If your paper is accepted and the number of authors is large, the
% style will print as headings an error message. Use the following
% command to supply a shorter version of the authors names so that
% they can be used as headings (for example, use only the surnames)
%
%\runningauthor{Surname 1, Surname 2, Surname 3, ...., Surname n}

\twocolumn[

\aistatstitle{On the Convergence of Continual Federated Learning Using Incrementally Aggregated Gradients }

\aistatsauthor{ Satish Kumar Keshri \And Nazreen Shah \And  Ranjitha Prasad }

\aistatsaddress{ IIIT Delhi \And  IIIT Delhi \And IIIT Delhi } 
]
\begin{abstract}
% \vspace{-5mm}
  The holy grail of machine learning is to enable  Continual Federated Learning (CFL) to enhance the efficiency, privacy, and scalability of AI systems while learning from streaming data. The primary challenge of a CFL system is to overcome global catastrophic forgetting, wherein the accuracy of the global model trained on new tasks declines on the old tasks. In this work, we propose \emph{Continual Federated Learning with Aggregated Gradients} (C-FLAG), a novel replay-memory based federated strategy consisting of edge-based gradient updates on memory and aggregated gradients on the current data. We provide convergence analysis of the C-FLAG approach, which addresses forgetting and bias while converging at a rate of $\mathcal{O}(1/\sqrt{T})$ over $T$ communication rounds. We formulate an optimization sub-problem that minimizes catastrophic forgetting, translating CFL into an iterative algorithm with adaptive learning rates that ensure seamless learning across tasks. We empirically show that C-FLAG outperforms several state-of-the-art baselines on both task and class-incremental settings with respect to metrics such as accuracy and forgetting.
\end{abstract}
% \vspace{-5mm}
\section{INTRODUCTION}
 % \vspace{-4mm}
The concept of lifelong learning in AI is inspired by the basic human nature of learning and adapting to new experiences and knowledge continuously throughout one's life. Continual learning (CL) is an important aspect of lifelong learning, where the key is to gain knowledge of new tasks without forgetting the previously gained expertise. Centralized lifelong learners are well-known \cite{10444954}. However,  increasing privacy concerns, the volume and complexity of data generated by various sources such as sensors, IoT devices, online platforms, and communication bottlenecks have led to the advent of continual federated learning (CFL) mechanisms. 
%Existing federated learning (FL) techniques such as FedAvg \cite{FedAvg}, Scaffold \cite{scaffold}, FedProx \cite{fedprox}, etc., are not well-suited for real-world streaming data as they are designed for static block datasets. 

A popular use-case of CFL is edge streaming analytics, where a stream of private data is analyzed continuously at the edge-user \cite{georgiou2018streamsight,9732409}, enabling organizations to extract insights for data-driven decisions without transmitting the data to a centralized server. Edge streaming analytics is well-suited for autonomous decision-making memory-constrained applications such as industrial IoT \cite{sen2023replay, husom2023reptile}, smart cities \cite{ul2021incremental}, autonomous systems \cite{shaheen2022continual} and remote monitoring \cite{doshi2020continual}. Conventional ML techniques necessitate retraining in order to adapt to the non-stationary streaming data \cite{zenke2017continual} while computational and memory constraints restrict the simultaneous processing of previous and current datasets, rendering retraining impossible. Further, edge-based analytics without federation results in models that can only learn from its \emph{direct experience} \cite{Fedweit}. A privacy-preserving strategy that allows continuous learning at the global level while circumventing all the above-mentioned issues is Continual Federated Learning (CFL) \cite{yang2024federated}.

In CFL, clients train on private data streams and communicate their local parameters to a server, which subsequently shares the global model with the clients. Several issues in CFL, such as inter-client interference \cite{Fedweit}, dynamic arrival of new classes into FL training \cite{FCIL}, and local and global catastrophic forgetting \cite{FederatedOrthoTraining}, have been studied. Memory-based replay techniques \cite{dupuy2023quantifying,good2023coordinated,li2024towards} offer a direct mechanism to revisit past experiences without requiring access to the historical data of other clients. In particular, episodic replay, wherein a small, fixed-size replay buffer of past data is stored along with new data at each client, has proven to be effective in reducing forgetting \cite{dupuy2023quantifying}. However, the replay buffer permits limited access to the dataset from past tasks, resulting in sampling bias~\cite{Chrysakis2023OnlineBC}. Therefore, it is essential to jointly manage the bias and federation when developing replay-based CFL methods. 

Consider a federated real-time surveillance use-case as depicted in Fig.~\ref{fig:diminishing_rate} (right), where the edge analytics task is the continuous monitoring and analysis of data streams to respond to events as they occur. Let $\mathcal{P}^i$ comprise data from all the previous tasks at the $i$-th client until a given observation point $t = 0$. The $i$-th client samples data from $\mathcal{P}^i$ and stores it in the buffer $\mathcal{M}^i \subset \mathcal{P}^i$ as depicted in Fig.~\ref{fig:diminishing_rate}. Gradient updates on the data in $\mathcal{M}^i$ lead to biased gradients since $\mathcal{M}^i$ consists of a subset of data of all the previous tasks. At $t = 0$, the server transitions from the previous to the current task, and the FL model starts to learn from the current non-stationary dataset $\mathcal{C}^i$. The goal of the CFL framework is to learn from $\mathcal{C}^i, \forall i$, while mitigating the effect of catastrophic forgetting on $\mathcal{P}^i$.

\textbf{Contributions:} We propose the novel \emph{Continual Federated Learning with Aggregated Gradients} (C-FLAG) technique, which is a replay-based CFL strategy consisting of local learning steps and global aggregation at the server. We consider the incremental aggregated gradient (IAG) approach, which significantly reduces computation costs and comes with the benefits of variance reduction techniques \cite{fedtrack}. To jointly mitigate the issues of client drift, bias, and forgetting, an \emph{effective gradient}, which is a combination of a single gradient on the memory buffer and multiple gradients on the current task, is proposed, as depicted in Fig.~\ref{fig:diminishing_rate} (left). Our contributions are as follows:

- We formulate the CFL problem as a smooth, non-convex finite-sum optimization problem and theoretically demonstrate that the proposed C-FLAG approach converges to a stationary point at a rate of $\mc{O}(\frac{1}{\sqrt{T}})$ over $T$ communication rounds.\\
- We formulate an optimization sub-problem parameterized by the learning rate to minimize catastrophic forgetting, allowing the translation of C-FLAG into an iterative algorithm with adaptive learning rates for seamless learning across tasks.

\begin{figure*}
    \includegraphics[scale=0.25]{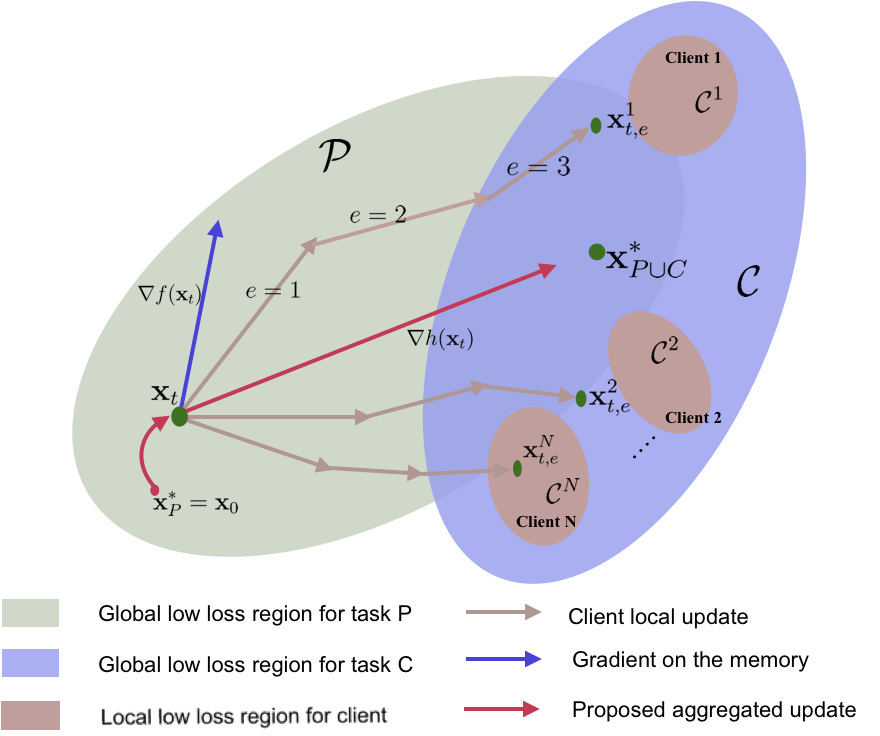} \hspace{5mm}
    \includegraphics[scale=0.25]{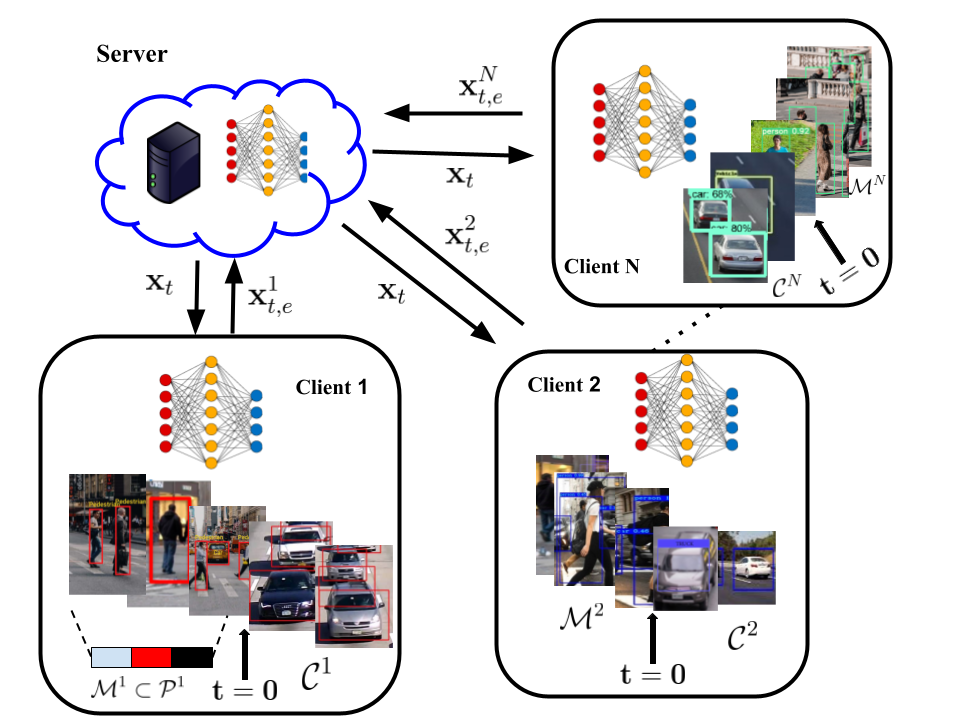}
    % \caption{(Left) Illustration of C-FLAG: Initialised at the optimal point of the previous tasks $\mathbf{x}^*_\mathcal{P} = \mathbf{x}_0$, at the $t$-th iteration, $i$-th client takes $E$ local steps towards its local optimal regions (pink regions). To balance learning and forgetting, C-FLAG takes a single step towards local memory and $E$ steps on the local current data. The global aggregated model moves towards a common global minima $\mathbf{x}^*_{\mathcal{P}\cup\mathcal{C}}$. (Right) Real-time surveillance where a subset of previous tasks are stored in memory until $t=0$. Data arriving thereafter is the current task $\mathcal{C}^i$.
    \caption{(Left) Illustration of C-FLAG: Initialised at the optimal point of the previous tasks $\mathbf{x}^*_\mathcal{P} = \mathbf{x}_0$, at the $t$-th iteration, $i$-th client takes $E$ local steps towards its local optimal regions (pink regions). To balance learning and forgetting, C-FLAG takes a single step towards local memory and $E$ steps on the local current data. The global aggregated model moves towards a common global minima $\mathbf{x}^*_{\mathcal{P}\cup\mathcal{C}}$. (Right) Real-time surveillance where a subset of previous tasks are stored in memory until $T=0$. Data arriving thereafter is the current task $\mathcal{C}^i$.
    }
    \label{fig:diminishing_rate} 
    % \vspace{-5mm}
\end{figure*}
% Review - Left image has wrong subscripts for X
% \vspace{5cm}
We evaluate C-FLAG on task-incremental FL setups, where it consistently outperforms baseline methods in terms of both average accuracy and forgetting. We also perform ablation studies on data heterogeneity, varying number of clients, and the size/type of replay buffer. The results show that     C-FLAG outperforms well-known and state-of-the-art baselines in mitigating forgetting and enhancing overall model performance. 

To the best of our knowledge, this work is the first of its kind to propose a replay-based CFL framework with convergence guarantees. The crux of the theoretical analysis deals with the handling of bias due to memory constraints and characterizing catastrophic forgetting. While prior FL works typically rely on convex or strongly convex assumptions, our analysis extends to the more general non-convex setting. 
%In the sequel, we ***

\section{PROBLEM FORMULATION}
\label{sec:problemStatement}

In a typical FL setting, $N$ edge devices collaboratively solve the finite-sum optimization problem given as:
\begin{align}
    \min_{\mathbf{x} \in \mathbb{R}^d} \phi(\mathbf{x}) = \min_{\mathbf{x} \in \mathbb{R}^d} {\textstyle\sum}_{i = 1}^N p_i \phi_i(\mathbf{x})
\end{align}
where $\phi_i(\mf{x})$ is the local loss function at the $i$-th client, $\phi(\mf{x})$ is the global loss function and $p_i$ is the weight assigned to $i$-th client. Each client independently computes gradients on a local private dataset, and the central server receives and aggregates these updates using a predefined strategy \cite{FedAvg}. In contrast, the edge devices in a CFL framework observe private streaming data, and the goal is to adapt the models as the data arrives, without forgetting the knowledge gained from past experiences. Given the global memory and current datasets $\mathcal{C} = \{\mathcal{C}^1, \mathcal{C}^2, \ldots,  \mathcal{C}^N\}$ and $\mathcal{P} = \{\mathcal{P}^1, \mathcal{P}^2, \ldots,  \mathcal{P}^N \}$, where $\mathcal{P}^i$ and $\mathcal{C}^i$ represents the past and the current task at the $i$-th client, the continual learning problem is defined as a smooth finite-sum optimization problem given as
\begin{flalign}
    \min\limits_{\mf{x}\in{\mathbb{R}^d}} \h{}{x} := {\textstyle\sum}_{i=1}^{N} p_i \h{i}{x},
\end{flalign}
where $h(\cdot)$ is a smooth, non-convex function which decomposes into $f(\mf{x})$ and $g(\mf{x})$ as
\begin{flalign}
    \h{}{x} = \frac{|\mc{P}|}{|\mc{P}| + |\mc{C}|} f(\mf{x}) + \frac{|\mc{C}|}{|\mc{P}| + |\mc{C}|} g(\mf{x}).
\end{flalign}
Here, $f(\mf{x})$ and $g(\mf{x})$ represents the restriction of $\h{}{\mf{x}}$ on the datasets $\mathcal{P} = \{\mc{P}^1, \mc{P}^2, \hdots,\mc{P}^N \}$ and $\mathcal{C} = \{\mc{C}^1, \mc{C}^2, \hdots,\mc{C}^N \}$, respectively, and $ |\mc{P}| = {\textstyle\sum}_{i=1}^{N} |\mc{P}^i|$, $ |\mc{C}| = {\textstyle\sum}_{i=1}^{N} |\mc{C}^i|$. Additionally, we also define the global functions $f(\mf{x})$ and $g(\mf{x})$ in terms of the local optimization objectives as $f(\mf{x}) := \h{}{x}|_\mc{P} :={\textstyle\sum}_{i=1}^{N}p_i f_i(\mf{x})$ and $g(\mf{x}) := \h{}{x}|_\mc{C} := {\textstyle\sum}_{i=1}^{N}p_i g_i(\mf{x})$. Here, $f_i(\mf{x})$ and $g_i(\mf{x})$ are the restrictions on the local previous and current datasets, respectively. Episodic memory $\mc{M}^i$ consisting of a fixed-size buffer of size at most $m_0$ stores a subset of the data that arrives prior to the start of the current task ($t=0$) at the $i$-th client \cite{dupuy2023quantifying}. After sampling, it remains fixed until it trains on the current task over $T$ communication rounds, i.e., $\mc{M}^i = \mc{M}_t^i \; \forall t \in \{0,1, \cdots, T-1 \} \text{ and } i \in [N]$. We define the global memory dataset as $\mc{M} := \{\mc{M}^1,\mc{M}^2, \ldots, \mc{M}^N\}$. Since the replay buffer at the edge, $\mc{M}^i$, permits limited access to the dataset from past tasks, gradient-based approaches result in bias ~\cite{Chrysakis2023OnlineBC}.

Streaming data tends to be non-stationary, while the conventional FL framework always assumes stationary data. Consequently, the convergence of any CFL framework is non-intuitive, and hence, it is essential to theoretically prove that the proposed strategy converges for the previous as well as on the new task. 

\section{C-FLAG ALGORITHM}
% \vspace{-4mm}
\label{sec:algorithm}
C-FLAG consists of local learning steps, where an effective gradient is computed at each client, and a global aggregation step is performed at the server. At the $i$-th client, C-FLAG computes an effective gradient which is a combination of a gradient on the local memory buffer $\mc{M}^i$, and $E$ gradient steps on the dataset of the current task $\mc{C}_i$. The $E$ gradient steps is obtained using the update rule as follows:
\begin{flalign}
    \lw{t,k+1}{i} = \lw{t,k}{i} - \bet{t} \left( \grad{g}{t} - \grad{g_i}{t} + \dgrad{g}{i}{t,k}{i} \right) \label{eq: LUR},
\end{flalign}
where $(\grad{g}{t} - \grad{g_i}{t})$ represents the one-time computation of the difference between the global and the local full gradient of the loss function, respectively. Essentially, $\grad{g}{t}$ provides an insight into the global descent direction at the beginning of round $t$. This descent direction is computed at a stale iterate $\mathbf{x}^i_{t,0} = \mathbf{x}_t$, and not at the current iterate. To account for this deviation, \cite{fedtrack} proposes IAG as an approximation to $\nabla g_i(\mathbf{x}^i_{t,k})$, which we represent as  $\dgrad{g}{i}{t,k}{i}$ defined as follows:
\begin{flalign}
    \dgrad{g}{i}{t,k}{i} := \tfrac{1}{|\mc{C}^i|} {\textstyle\sum}_{j\in \mc{C}^i}{} \nabla g_{i,j}(\mf{x}_{t,\tau_{k,j}^i}^{i}).
    \label{eq: def-memory-g-grad}
\end{flalign} 
At each local step $k$, the $i$-th client computes the gradient of a single, say the $j$-th component function of $g_i(\mathbf{x}^i_{t,k})$ given by $g_{i,j}(.)$ for $j \in \mc{C}^i$. For the remaining $r \in \mathcal{C}^i, r \backslash j$ component functions, the most recently computed gradients are retained. In order to track the component function that is updated, we use an index $\tau_{k,j}^i$ which denotes the most recent local step. At the beginning of the first local epoch ($k=0$), the $i$-th client computes the full gradient and sets $\tau_{0,j}^i = 0 \; \forall j \in \mc{C}^i $. 

The evolution of the index $\tau_{k,j}^i$ can be understood using the following toy example. Suppose, the current dataset $\mc{C}^i$ has $3$ data points, i.e., $\mc{C}^i = \{c_1, c_2, c_3 \}$. At the beginning of the $t$-th round $\dgrad{g}{i}{t,0}{i}$ is computed as $\dgrad{g}{i}{t,0}{i} = \frac{1}{3}(\grad{g_{i,c_1}}{0} + \grad{g_{i,c_2}}{0} + \grad{g_{i,c_3}}{0}) = \grad{g_{i}}{0}$. Hence, $\tau_{t,j}^i = 0$ for 
 all $j$. At $k=1$, the client randomly samples $c_1$, computes $\nabla{g_{i,c_1}}(\lw{1}{i})$ and updates $\tau_{t,d_1}^i = 1$. For the other components the previously computed gradients are used, and hence, $\dgrad{g}{i}{t,1}{i} = \frac{1}{3}( \nabla{g_{i,c_1}}(\lw{t,1}{i}) + \nabla{g_{i,c_2}}(\lw{t,0}{i}) +\nabla{g_{i,c_3}}(\lw{t,0}{i}) )$. At $k=2$, suppose the client randomly samples$c_3$, computes the gradient on $c_3$ and updates $\tau_{t,c_3}^i = 2$, which leads to $\dgrad{g}{i}{t,1}{i} = \frac{1}{3}( \nabla{g_{i,c_1}}(\lw{t,1}{i}) + \nabla{g_{i,c_2}}(\lw{t,0}{i}) +\nabla{g_{i,c_3}}(\lw{t,2}{i}))$. This process continues until $k= E-1$. Using delayed gradients at each local epoch, the gradient computation cost drastically decreases as only one component is updated at a time.

% \lgrad{g_{i,d_1}}{}{} +  \lgrad{g_{i,d_2}}{0} + \lgrad{g_{i,d_3}}{0})
To mitigate catastrophic forgetting while transitioning from one task to another, we update the model on both, the memory buffer data and the current data. However, frequent training on the memory data may lead to overfitting and impede learning from the current dataset. Hence, we propose to take one step towards the memory data for every $E$ local step on current data utilizing the memory as a guide for future learning. Accordingly, for each client $i \in [N]$, we obtain the biased gradient on the memory buffer $\mc{M}^i$ as
\begin{flalign}
    \mgrad{f}{i}{}{} = \grad{f_i}{} + {b_i}(\mf{x}),
    \label{eq:biasedGrad}
\end{flalign}
where $b_i(\cdot)$ quantifies the bias introduced due to the sampling of the memory data $\forall i \in [N]$. Additionally, we define the average bias as ${{b}}(\gw{}) = \tsumk{i=1}{N}p_i {b_i}(\mf{x})$.
After $E$ local epochs, the $i$-th client communicates $\Delta \lw{t}{i} =  \lw{t,E}{i} - \alp{t}\mgrad{f}{}{t}{}$, where $\lw{t,E}{i}$ is obtained using \eqref{eq: LUR}. The server updates the global model as; 
% $\gw{t+1} = {\textstyle\sum}_{i=1}^{N}p_i \Delta \lw{t}{i}$ as
\begin{flalign}
\gw{t+1} = \gw{t} - \alp{t} {\textstyle\sum}_{i=1}^{N}p_i \mgrad{f}{}{t}{}  -  \bet{t}{\textstyle\sum}_{i=1}^{N}p_i \lw{t,E}{i},\label{eq:GUR} 
\end{flalign}
where, $\alpha_t$ and $\beta_t$ are the learning rates on the memory and current data, respectively. The steps of the algorithm are presented in Alg.~\ref{alg: A}, and the pseudocode for \texttt{AdapLR} is provided in the supplementary.

\begin{algorithm}[]
	\caption{C-FLAG: Continual Federated Learning with Aggregated Gradients}
	\label{alg: A}
	\begin{algorithmic}[1]
		\REQUIRE Step-size $\alp{}$, $\bet{}$, initial model $\gw{0}$, $Adap Flag$
  %and  initial full gradient $\grad{g}{0}$
		\ENSURE $\gw{t}$ for $t = 1,\hdots, T$
		\FOR{$t = 0, \ldots, T-1$}
            \STATE For client $i = 1, \hdots, N$, compute $\grad{g_i}{t}$ and $\mgrad{f}{i}{t}{}$, and transmit to the server. 
            \STATE Server computes $\grad{g}{t}$ and $\mgrad{f}{}{t}{}$, and broadcasts to each client.
            \FOR{client $i = 1, 2, \ldots N $ in parallel}
                % \STATE compute $\grad{g_i}{t}$ and transmit to the server
                
                \STATE Set $\lw{t,0}{i} = \gw{t} $%, compute $\mgrad{f}{i}{t,0}{i}$.
                \FOR{$k = 0, \hdots, E-1$}
                    \STATE Compute $\dgrad{g}{i}{t,k}{i}$ using \eqref{eq: def-memory-g-grad} and $\lw{t,k+1}{i}$ using \eqref{eq: LUR}.
                \ENDFOR
                \STATE $\Delta \lw{t}{i} = \texttt{AdapLR}(\lw{t,E}{i},\mgrad{f}{}{t}{}, \alp{},\bet{}, AdapFlag ) $.
                \STATE Transmit $\Delta \lw{t}{i}$ to the server.
            \ENDFOR
            \STATE Server computes and broadcasts $\gw{t+1}$ using $\gw{t+1} = \gw{t}- \sumk{i=1}{N}p_i\Delta \lw{t}{i}$.
        \ENDFOR
	\end{algorithmic}
\end{algorithm}
% \vspace{-4mm}
\section{CONVERGENCE ANALYSIS}
% \vspace{-4mm}
\label{sec:convergenceAnalysis}
In this section, we present a theoretical convergence analysis of the proposed memory-based continual learning framework in a non-convex setting. For purposes of brevity, proofs have been delegated to the supplementary. The assumptions are as follows:

\begin{assump} \label{asp: L-smoothness}
(L-smoothness). For all $i \in [N]$, $f_i, g_i, h_i$ are L-smooth. 
%For all $\mf{x}, \mf{y} \in \mathbb{R}^d, f_i(\mf{x}) \leq f_i(\mf{y}) + \inner{\grad{f_i}{}, \mf{x}-\mf{y}} + \frac{L}{2} \norm{\mf{x-y}}^2 $. 
\end{assump}
\begin{assump} \label{asp: bounded-bias}
    (Bounded Bias). There exists constants $0 \leq m_i < 1 $ for all $\mf{x} \in \mathbb{R}^d$ such that $\norm{{b_i}(\mf{x})}^{2} \leq m_i\norm{\grad{f_i}{}}^2, \quad \forall i \in [N]$.
\end{assump}
\begin{assump} (Bounded memory gradient).
    There exists $r_{i} \in \mathbb{R}^{+}$, such that $\lVert{\mgrad{f}{i}{t}{}}\rVert \leq r_i \norm{\grad{g}{t}}$ for all $i \in [N]$. \label{Asp:bounded-memory}
    %This is guaranteed by Archimedean property of $\mathbb{R}$ since $\lVert{\mgrad{f}{i}{t}{}}\rVert, \norm{\grad{g}{t}} \in \mathbb{R}^{+}$.  
\end{assump}
Since the biased gradient, $\mgrad{f}{i}{t}{}$, on the memory data is correlated with the true gradients, $\grad{f_i}{t}$, Assumption~\ref{asp: bounded-bias} is similar to Assumption~4 in \cite{ajalloeian2020convergence}. We also define the expectation over memory datasets till $t$-th global iteration as $\mathbb{E}_{\mc{M}_t} = \mathbb{E}_{[{\mc{M}_0}:{\mc{M}_t}]}$. Further we denote  $\mathbb{E}$ as $\mathbb{E}_{[{\mc{M}_0}:{\mc{M}_T}]}$ over $T$ global iterations. 

As the iterations progress according to algorithm~\ref{alg: A}, two additional quantities, namely an overfitting term, and a catastrophic forgetting term, vary with time. In the following lemma and theorem, we formally introduce these terms and provide a convergence guarantee of C-FLAG on the previous task $\mc{P}$.

%Discussion
\begin{lemma}
\label{thm:grad-f}
    Suppose that the assumptions \ref{asp: L-smoothness}, \ref{asp: bounded-bias} hold,  $\alp{t} < \frac{2}{L(1+m)}$ and $m\in\mathbb{R}^{+}$. For the sequence $\{\gw{t}\}_{t=1}^{T}$ generated by algorithm \ref{alg: A}, we have
    \begin{flalign}
        &\norm{\grad{f}{t}}^2  \leq \frac{1}{\alp{t}[1-\frac{L}{2}\alp{t}(1+m)]} \big( f(\gw{t}) - f(\gw{t+1}) \nonumber\\
        &+ \bias{} +\Gamma(t) \big),
    \end{flalign} 
where $B(t)$ is the overfitting term defined as 
\begin{flalign}
    B(t) &= (L\alp{t}^2 - \alp{t})\langle{\grad{f}{t}, {b}(\gw{t})}\rangle \nonumber\\
    &+ \bet{t}\langle{{b}(\gw{t}), {\textstyle\sum}_{i=1}^{N}{\textstyle\sum}_{k=1}^{E-1}p_i\dgrad{g}{i}{t,k}{i}}\rangle.
\end{flalign}
Further, $\Gamma(t)$ is the forgetting term defined as
\begin{flalign}
   &\Gamma(t) = L\bet{t}^2\lVert{\textstyle\sum}_{i=1}^{N}{\textstyle\sum}_{k=1}^{E-1}p_i\dgrad{g}{i}{t,k}{i}\rVert^2 - \bet{t}(1-L\alp{t}) \nonumber\\&\langle{\mgrad{f}{}{t}{}, {\textstyle\sum}_{i=1}^{N}{\textstyle\sum}_{k=1}^{E-1}p_i\dgrad{g}{i}{t,k}{i}}\rangle. %\nonumber\\
   % &\bet{t}(1-L\alp{t}) \langle{\mgrad{f}{}{t}{}, {\textstyle\sum}_{k}^{} \avgerror{t,k}}\rangle+L\bet{t}^2\lVert{\textstyle\sum}_{i}^{N}{\textstyle\sum}_{k=1}^{E-1}p_i\error{t,k}{i}\rVert^2.
   \label{eq:Lem1ForgettingTerm}
\end{flalign}
% \vspace{-5mm}
\end{lemma}

As the quantities $B(t)$ and $\Gamma(t)$ accumulate over time, they significantly degrade the performance of the continual learning framework \cite{nccl}. In particular, the term $\langle{\mgrad{f}{}{t}{}, {\textstyle\sum}_{i=1}^{N}{\textstyle\sum}_{k=1}^{E-1}p_i\dgrad{g}{i}{t,k}{i}}\rangle$ is a key factor that determines interference and transference \cite{a-gem, nccl} between locally aggregated gradients and the gradient on the memory. We provide a detailed discussion on the aspects of interference and transference as related to deriving adaptive learning rates in Sec.~\ref{sec:adapLR}.

To derive the convergence of Alg.~\ref{alg: A} we telescope over the training iterations for the current dataset, and using Lem.~\ref{thm:grad-f} we obtain the following theorem.
 
\begin{theorem}\label{lemma:min-grad-f}
 Suppose that the assumptions \ref{asp: L-smoothness}, \ref{asp: bounded-bias} hold. Given $F = f(\gw{0}) - f(\gw{T})$, the sequence $\{\gw{t}\}_{t=1}^{T}$ generated by algorithm \ref{alg: A} with $\alp{t} = \alp{} = \frac{1}{L(m+1)} \; \forall t \in \{0,1, \cdots, T-1 \} $, and $m\in\mathbb{R}^{+}$, satisfies 
\begin{flalign*}
    \min\limits_{t}  \mathbb{E} \left[\norm{\grad{f}{t}}^2 \right] \leq \frac{2L(1+m)}{T} \bigg(F + \sumk{t=0}{T-1}\mathbb{E}[\Gamma(t)] \bigg).
\end{flalign*}
\end{theorem}

We observe that the expected convergence of $f(\cdot)$ depends on the initialization and the forgetting term. The overfitting term tends to zero in expectation, as shown in the supplementary material. On the other hand, the cumulative forgetting term accumulates with each server iteration. Hence, a tight upper bound on the convergence of the previous task necessitates an upper bound on $\Gamma(t)$, which we do in the sequel.

It is also essential to obtain the convergence guarantee on the current task if the updates are obtained using Alg.~\ref{alg: A}, where both the memory and the current dataset are used together. This leads to deriving the convergence rate of the global loss function, $h(\cdot)$,  which learns jointly with respect to the current task and the replay memory at each client.

\begin{lemma} \label{lemma:h_rate}
Given assumption \ref{asp: L-smoothness}, the sequence $\{\gw{t}\}_{t=1}^{T}$ generated with learning rates $ \alp{t} = \bet{t} = \alp{} = \frac{1}{30LE}$, for the restriction of $h(\cdot)$ on $\mc{M}\cup \mc{C}$ given as $\mh(\gw{t}) = h_{|_{\mc{M}\cup \mc{C}}}(\gw{t})$ we have
\begin{flalign}
    \min\limits_{t} \lVert{\grad{\mh}{t}}\rVert^2 \leq \tfrac{60L}{T}\left(\mh(\gw{0}) - \mh(\gw{T})\right).
\end{flalign}
%\vspace{-3mm}
\end{lemma}

From the above lemma, we observe that using the proposed update rule, the overall global loss function reaches a stationary point in sub-linear  $\mc{O}(\tfrac{1}{T})$ time complexity. Since the global loss function converges on $\mc{M}\cup \mc{C}$, the convergence on current task $C$ is evident from the fact that $\mc{C} \subset \mc{M}\cup \mc{C}$. In the next lemma, we will provide an upper bound on the cumulative forgetting term.
\begin{lemma} \label{lemma:gammabound}
Suppose that the assumptions \ref{asp: L-smoothness}, \ref{asp: bounded-bias}, \ref{Asp:bounded-memory} hold and the step-sizes satisfy $\alp{t} = \alp{} < \frac{2}{L(1+m)}$ and {$\bet{t} = \bet{} < \frac{c}{\sqrt{T}} \forall t \in \{0,1,\cdots, T-1\} $} and for some $ c,m \in \mathbb{R}^{+}$. Then the following holds for the forgetting term $\Gamma(t)$:
\begin{flalign}
    \frac{1}{T}\sumk{t=0}{T-1}\mathbb{E}[\Gamma(t)] < \mc{O}\bigg(\frac{1}{T} + \frac{1}{\sqrt{T}} \bigg). 
\end{flalign}
% \vspace{-5mm}
\end{lemma}
One key observation from the above lemma is that using constant step sizes for updates on the memory data (via $\alp{}$) and diminishing step sizes for updates on the current data (via $\bet{}$) is crucial for a tight bound on cumulative forgetting. If we use constant step-sizes for updates on both memory and current data instead of diminishing rates, then the cumulative forgetting term is only guaranteed to converge with constant time complexity $\mc{O}(1)$. Finally, in the next theorem, we show that C-FLAG converges on the previous task, alleviating the problem of catastrophic forgetting.
\begin{theorem}\label{theorem:f-rate}
Let $\{\gw{t}\}_{t=1}^T$ be the sequence generated by algorithm~\ref{alg: A}, and the step-sizes satisfy $\alp{t} = \frac{1}{L(m+1)}$ and $\bet{t} < \frac{c}{\sqrt{T}} \forall t \in \{ 0,1, \cdots, T-1 \}$ and for some $ c,m \in \mathbb{R}^{+}$. Then, we obtain the following rate of convergence
\begin{flalign}
    \min\limits_{t}  \mathbb{E} \left[\norm{\grad{f}{t}}^2 \right] < \mc{O}\bigg(\frac{1}{\sqrt{T}}\bigg)
\end{flalign}
% \vspace{-5mm}
\end{theorem}
From the above, we obtain a sub-linear $\mc{O}(\tfrac{1}{\sqrt{T}})$ convergence rate for the proposed CFL problem. Theoretically, we show that diminishing step sizes for the current task is important to have a better upper bound on the convergence on $f$. In the centralized continual learning setup, \cite{nccl} also provides $\mc{O}(\tfrac{1}{\sqrt{T}})$ convergence rate, but they impose additional constraints on the learning rates. Without these constraints, convergence on the past task is not guaranteed in their work and may diverge at the rate $\mc{O}(\sqrt{T})$.

% \vspace{-4mm}
\section{C-FLAG: ADAPTIVE LEARNING RATES}
% \vspace{-4mm}
\label{sec:adapLR}
In the previous sections, we provided theoretical convergence guarantees and the rate of convergence of the proposed C-FLAG framework. We observed that improved convergence on the previous task is possible if the cumulative forgetting term is as small as possible. In particular, in Theorem~\ref{lemma:min-grad-f} we discussed that the convergence rate is dependent on the cumulative forgetting term $\tsumk{t=0}{T-1}\mathbb{E}[\Gamma(t)]$, and additionally, Lem.~\ref{lemma:gammabound} specifies the constraints on $\alp{t}$ and $\bet{t}$ for obtaining an upper bound on the average of the cumulative forgetting terms over $T$ iterations. In this section, we translate the analysis into an easily implementable solution where learning rates $\alp{t}$ and $\bet{t}$ are adapted to achieve lower forgetting at each iteration. Ideally, the adaptable learning rates solve the following constrained optimization problem
\begin{align}
    &\min\limits_{\alp{t}, \bet{t}} \mathbb{E}[\Gamma(t)] \quad \textrm{subj. to } \alp{t}< \frac{2}{L(1+m)} \forall t < T. \label{eq:gamma-min} 
\end{align}
This is a popular strategy in centralized continual learning, where adjusting learning rates trades off between learning new information and retaining previous knowledge \cite{nccl}. The above formulation allows us to adapt the learning rates to effectively balance learning and forgetting at the global level, leading to a practically implementable solution. Since the learning rate is within the constraint, it also enjoys theoretical convergence guarantees on both $f(\gw{})$ and $g(\gw{})$.

To derive the adaptive rates, we analyze the catastrophic forgetting term at the end of the $t$-th communication round (\eqref{eq:Lem1ForgettingTerm} in Lem.~\ref{thm:grad-f}) given as
{\small{ \begin{align}
     \Gamma(t) = \frac{L\bet{t}^2}{2} \norm{\sumk{i=1}{N}\sumk{k=0}{E-1}p_i\dgrad{g}{i}{t,k}{i}}^2 - \bet{t}(1-L\alp{t}) \Lambda_t ,
      \label{eq:OptiGammatSimpli}
\end{align}}}
where we denote $\Lambda_{t}= {\textstyle\sum}_{i=1}^{N}p_i \Lambda_{t,i} $ and $\Lambda_{t,i} = \langle{\mgrad{f}{}{t}{}, {\textstyle\sum}_{k=0}^{E-1}\dgrad{g}{i}{t,k}{i}}\rangle$. We observe from \eqref{eq:OptiGammatSimpli} that the first term is always non-negative, $(1-L \alpha_t) > 0$ (for $m>1$) and $\beta_t > 0$ and hence, the term $\Lambda_t$ is crucial to optimize $\Gamma(t)$. If $\Lambda_t>0$, it results in a favorable decrease in $\Gamma(t)$ while enhancing learning on the previous and current task, and this case is termed as \emph{transference}. Further, $\Lambda_t \leq 0$ leads to an increase in  $\Gamma(t)$ resulting in forgetting, and this case is termed as \emph{interference}. Since $\Gamma(t)$ is a quadratic polynomial in $\bet{t}$, the optimal value of $\bet{t}$ and $\mathbb{E}[\Gamma(t)]$ is obtained as 

\begin{align}
  \bet{t}^* &= \tfrac{(1-L\alp{t}) \Lambda_t}{L \norm{\sumk{i=1}{N}\sumk{k=0}{E-1}p_i\dgrad{g}{i}{t,k}{i}}^2}, \text{and}~\nonumber\\ \mathbb{E}[\Gamma^{*}(t)] &= -\tfrac{(1-L\alp{t})^2\Lambda_t^2}{2L \norm{\sumk{i=1}{N}\sumk{k=0}{E-1}p_i\dgrad{g}{i}{t,k}{i}}^2}.
\end{align} 
A subtle trick that can be used here is as follows: at the beginning of the training at $t=0$, the server fixes $\alp{t} = \bet{t} = \alp{} \forall t \in \{0,1,\cdots, T-1\}$ and clients train on their datasets for $E$ local epochs. Then at the end of each global communication round $t$, if the server observes a case of transference ($\Lambda_t >0)$ it rescales the client updates by $\tfrac{\bet{t}^*}{\alp{}}$ to get the effective learning rate as $\bet{t}^*$.  However, the above analysis handles transference on an average basis and not a per-client basis. This implies that some clients may observe interference locally ($\Lambda^i_t \leq 0)$, but their contribution is nullified while computing $\Lambda_t$. Hence, we consider the approach where we minimize the clients' contribution to $\Gamma(t)$ individually to better control cumulative forgetting. 

In order to provide the client specific analysis, let $\mf{a}_i = \tsumk{k=0}{E-1}\dgrad{g}{i}{t,k}{i}$ and the client interaction terms $C_{i,j} = \inner{\mf{a}_i, \mf{a_j}}$. The first term in $\Gamma(t)$ can be written as:
{\small\begin{align*}
    ||{\sumk{i=1}{N}p_i \mf{a}_i||^2 = \sumk{i=1}{N}p_i^2 \norm{\mf{a}_i}}^2 + \underbrace{\sumk{i=1}{N}\sumk{j=1,j \neq i}{N} p_i p_j C_{i,j}}_{A_{i,j}} \label{eq:c_i,j-expand}.
\end{align*}}
The client interaction terms play an important role in the analysis. If $C_{i,j} \leq 0$, it can further reduce $\Gamma(t)$ and this leads to alleviated forgetting, while $C_{i,j} >0$ leads to an increase in $\Gamma(t)$ and increased forgetting. Due to privacy concerns, these inter-client interaction terms cannot be utilised directly. In our study, we consider two cases; (a) the \emph{average case} where $A_{i,j} =0$, and (b) the \emph{worst case} where $C_{i,j} > 0\; \forall i,j \in [N] \text{ and } i \neq j$. We derive adaptive learning rates by analyzing the average and the worst case individually. As mentioned earlier, we may fix the learning rate at the beginning of the training on the current task and rescale based on the interference or transference case. Additionally, we denote the catastrophic forgetting obtained using the adapted rates as $\Gamma_{i, ad}(t)$ and $\Gamma_{ad}(t)$ for the $i$-th client and the server, respectively. The adaptive rates obtained after tackling the interference and transference are presented in Table~\ref{tab:adaptive_rates}. In both the average and the worst case when the $i$-th client interferes with the past learning, we adapt $\alp{t}$ to $\alp{t,i}$ while retaining $\bet{t}$ as $\alp{}$. On the other hand, in the case of transferring clients, we adapt $\bet{t}$ to $\bet{t,i}$ and retain $\alp{t}$ as $\alp{}$. We summarise the proposed adaptive rates in Table~\ref{tab:adaptive_rates}.
\begin{table}[h]
\centering
\caption{Table for adaptive learning rates to optimize catastrophic forgetting. In the table, I stands for interference ($\Lambda_{t,i} \leq 0$), and  T stands for transference ($\Lambda_{t,i} > 0$).}
\begin{tabular}{|c|c|c|c|}
\hline
\textbf{Case} &
  \textbf{Type} &
  \textbf{$\alp{t,i}$} &
  \textbf{$\bet{t,i}$} \\ \hline
\multirow{2}{*}{Average} &
  I &
  $\alp{}(1- \frac{\Lambda_{t,i}}{\norm{\mgrad{f}{}{t}{}}^2})$ &
  $\alp{}$ \\ \cline{2-4}
 &
  T &
  $\alp{}$ &
  $\frac{(1-L\alp{})\Lambda_{t,i}}{L p_i\norm{\mf{a}_i}^2}$ 
   \\ \cline{1-4}
\multirow{2}{*}{Worst} &
  I &
  $\alp{}(1- \frac{\Lambda_{t,i}}{\norm{\mgrad{f}{}{t}{}}^2})$ &
  $\alp{}$ 
   \\ \cline{2-4}
 &
  T &
  $\alp{}$ &
  $\frac{(1-L\alp{})\Lambda_{t,i}}{LN p_i\norm{\mf{a}_i}^2}$ \\
   \hline
\end{tabular}
\label{tab:adaptive_rates}
% \vspace{-5mm}
\end{table}

The following lemmas show that our proposed choice of adaptive $\alp{t,i}$ and $\bet{t,i}$ in the case of interference and transference, respectively, leads to lower forgetting, i.e., $\mathbb{E}[\Gamma_{i, ad}(t) ] \leq \mathbb{E}[\Gamma_{i}(t)]$. 

\begin{lemma} \label{lemma:less-inter}
     In the case of interference at the $i$-th client, for both the average and worst cases, adaptive rates $\alp{t,i} = \alp{}(1- \frac{\Lambda_{t,i}}{\norm{\mgrad{f}{}{t}{}}^2})$ and $\bet{t,i} = \alp{}$ lead to smaller forgetting, that is $\mathbb{E}[\Gamma_{i, ad}(t) ] \leq \mathbb{E}[\Gamma_{i}(t)]$.
\end{lemma}

\begin{lemma} \label{lemma:less-gamma}
    Client-wise adaptive rates lead to improved forgetting, $\mathbb{E}[\Gamma_{ad}(t)] \leq \mathbb{E}[\Gamma(t)] $.
\end{lemma}
% \vspace{-4mm}
\section{RELATED WORKS AND DISCUSSIONS}
% \vspace{-4mm}

\textbf{Memory Based Approaches:}\label{par:replay_memory}
In order to mitigate catastrophic forgetting~(~\cite{french1999catastrophic}), memory-based approaches store a subset of data to be leveraged during training. One way to utilise the memory data is to impose constraints while optimizing loss while learning new tasks. Centralized continual learning approaches such as GEM~\cite{gem} and A-GEM~\cite{a-gem} impose these constraints through gradient projection such that the loss on the previous tasks does not increase. Episodic replay (ER-replay) memory-based approaches directly use the gradients on the memory during training ~\cite{tiny-er}, \cite{er}. Achieving the right balance in learning the past tasks and new ones, known as the stability-plasticity trade-off~\cite{stabilityPlasticity}, is crucial for continual learners. In continual federated learning (CFL), ~\cite{dupuy2023quantifying} highlights the role of ER-replay memory at each client. Existing methods with  ER-replay-based CFL include exemplar \cite{ER_exem} and prototypical-class \cite{ER_Prot} based methods. In \cite{ER_DP}, authors demonstrate that direct application of replay methods on temporal shifts in client behavior leads
to poor results, and as a solution, they propose data sharing between clients employing differential privacy. A drawback of existing approaches is that while the convergence behavior has been demonstrated through empirical results, theoretical guarantees on their convergence are unfounded. In contrast, C-FLAG effectively utilises the replay memory while maintaining the stability-plasticity trade-off by performing a single step of training on the memory data as a guide for multiple local steps on the current task dataset using adaptive learning rates. We also provide theoretical convergence guarantees by addressing catastrophic forgetting.

\textbf{Stale Gradients for Client Drift Mitigation:} The incrementally aggregated gradient (IAG)~\cite{iag} method significantly reduces computation costs but introduces stale gradients, which can deviate from the true gradient direction. In the federated setting, FedTrack \cite{fedtrack} uses IAG and demonstrates that it improves the convergence rate in comparison to non-IAG methods. Furthermore, stale server gradients from the last round are used to reduce the client drift between two communication rounds. C-FLAG takes inspiration from IAG-based updates to reduce client drift and jointly learn from memory and current data to mitigate forgetting. 

\textbf{Comparison with NCCL~(\cite{nccl}):}
NCCL is a centralized continual learning algorithm that greedily adapts learning rates at each iteration and trains a model jointly on memory and current data. NCCL is a strategy that uses adaptive step sizes for continual learning when the optimization problem is smooth and non-convex. NCCL overcomes several shortcomings in \cite{gem} and \cite{a-gem} in the context of adapting the learning rates. C-FLAG is a federated counterpart of NCCL as an adaptive learner, when data is available at the edge user. As FL involves local training, the global greedy adaptation of learning rates in every epoch is not feasible. In our approach, IAG helps to adapt the learning rates based on the accumulated gradients. This allows C-FLAG to optimize the learning rates at each client while reducing global catastrophic forgetting.
%Since C-FLAG performs a single gradient step on the memory data for every $E$ local training steps, {\color{red}{we expect the proposed approach to outperform NCCL with respect to the accuracy on the current task.}}
% \vspace{-4mm}
\section{EXPERIMENTAL RESULTS AND DISCUSSIONS}
\label{sec:ExperiResults}
% \vspace{-4mm}
In this section, we present the experimental results to demonstrate the efficacy of the proposed C-FLAG algorithm in task-incremental settings. 

\textbf{Benchmarks:} We demonstrate the experimental results on continual learning benchmarks such as Split-CIFAR10, Split-CIFAR100, and Split-TinyImageNet. Split-CIFAR10 and Split-CIFAR100 are derived from the CIFAR10 and CIFAR100~\cite{krizhevsky2009learning} datasets, where the entire dataset is split into 5 tasks. Each task in Split-CIFAR10 has two classes, while each task in Split-CIFAR100 contains 20 classes. In the TinyImageNet~\cite{Le2015TinyIV} dataset, data is divided into 10 tasks, each consisting of 20 classes. For non-IID splits, we employ the Dirichlet partitioning technique to distribute task-specific data among the participating clients. This allows us to regulate the data heterogeneity in FL using the Dirichlet parameter $\zeta$. In this work, we use $\zeta = 0.1$ and $\zeta = 10^5$ for simulating the non-IID and IID scenarios, respectively. We use a ResNet-18~\cite{he2015deepresiduallearningimage} backbone for classification on Split-CIFAR10 and Split-CIFAR100 datasets, whereas we use a ResNet-50 backbone for classification on Split-TinyImageNet dataset.  on all the datasets. Other details of the experimental configuration have been delegated to the appendix.

\textbf{Baselines:} We compare the proposed method in both task-incremental and class-incremental setups against the following baselines. The primary distinction between these two setups lies in the usage of task/class identity. In the task-incremental setup, task identity is required during both the training and testing phases, whereas in the class-incremental setup, task information is needed only during training. Another key difference is how the classifier heads are handled: in the class-incremental setup, classifier heads are appended progressively as new classes are introduced using training, while in the task-incremental setup, all classifier heads are initialized as a list at the beginning of training and selected based on the task identity for training and testing. We now proceed to discuss the baselines used for the different settings, marking each baseline with `TI' for task-incremental and `CI' for class-incremental, respectively. In all the cases, FedAvg~\cite{FedAvg} is used to adapt the centralized continual learning techniques to the federated setting. 
\begin{itemize}
\item ~\texttt{Fine-FL} (TI, CI): A naive baseline where a client model trains on the current data.
\item ~\texttt{EWC-FL} (Elastic weight consolidation in FL ~\cite{EWC})(TI, CI):~This is a Fisher information based regularization method implemented in a federated manner.
\item ~\texttt{NCCL-FL}~\cite{nccl}(TI, CI): We implement the centralized NCCL technique in a federated manner.
\item \texttt{Erg-FL}~\cite{tiny-er}(TI): Ering-FL is an experience-replay based continual learning method implemented in a federated manner. 
\item \texttt{FedTrk}~\cite{fedtrack}(TI, CI): FedTrack is an FL technique that uses IAG for local updates.
\end{itemize}

{\scriptsize
\begin{table}[h!]
% \vspace{-2mm}
\centering
\caption{Average accuracy and forgetting for non-IID setting on Split-CIFAR10, Split-CIFAR100, and Split-TinyImageNet (TinyIN) with $5$ clients.}

\begin{tabular}{@{\hspace{3pt}}l@{\hspace{3pt}}c@{\hspace{3pt}}:c|c@{\hspace{3pt}}:c|c@{\hspace{3pt}}:c@{}}
\toprule
 & \multicolumn{2}{c|}{{\small CIFAR10}} & \multicolumn{2}{c|}{\small CIFAR100} & \multicolumn{2}{c}{\small TinyIN}  \\
\cmidrule(lr){2-3} \cmidrule(lr){4-5} \cmidrule(lr){6-7}
 & Acc\tiny{($\uparrow$)} & Fgt\tiny{($\downarrow$)} & Acc\tiny{($\uparrow$)} & Fgt\tiny{($\downarrow$)} & Acc\tiny{($\uparrow$)} & Fgt\tiny{($\downarrow$)} \\
\midrule
\texttt{Fine-FL}  & $54.10$ & $6.08$ & $38.66$ & $20.87$ & $23.72$ & $23.85$ \\
\texttt{EWC-FL}       & $53.55$ & $\mf{5.22}$ & $38.83$ & $19.20$ & $26.94$ & $20.29$ \\
\texttt{NCCL-FL}     & $63.35$ & $12.37$ & $32.25$ & $29.49$ & ${28.49}$ & $\mf{8.73}$ \\
% \textcolor{red}{\texttt{FedTrk}}    & $\mf{80.50}$ & $13.58$ & $23.16$ & $\mf{15.54}$ & $6.72$ & $3.63$ \\
\texttt{Erg-FL}  & $\mf{79.11}$ & 8.32 & 31.25 & 32.40 & 21.90 & 19.67 \\
\texttt{C-FLAG}  & 65.02 & $5.82$ & $\mf{43.47}$ & $\mf{16.76}$ & $\mf{28.63}$ & $9.52$ \\
\bottomrule
\end{tabular}
\label{tab:performance_non_iid}
\end{table}
}

{\scriptsize
\begin{table}[h!]
\centering
\caption{Average accuracy and forgetting for IID setting on Split-CIFAR10, Split-CIFAR100, And Split-TinyImageNet (TinyIN) with $5$ clients.}

\begin{tabular}{@{\hspace{3pt}}l@{\hspace{3pt}}c@{\hspace{3pt}}:c|c@{\hspace{3pt}}:c|c@{\hspace{3pt}}:c@{}}
\toprule
 & \multicolumn{2}{c|}{{\small CIFAR10}} & \multicolumn{2}{c|}{\small CIFAR100} & \multicolumn{2}{c}{\small TinyIN}  \\
\cmidrule(lr){2-3} \cmidrule(lr){4-5} \cmidrule(lr){6-7}
 & Acc\tiny{($\uparrow$)} & Fgt\tiny{($\downarrow$)} & Acc\tiny{($\uparrow$)} & Fgt\tiny{($\downarrow$)} & Acc\tiny{($\uparrow$)} & Fgt\tiny{($\downarrow$)} \\
\midrule
\texttt{Fine-FL}  & $72.64$ & $26.51$ & $49.82$ & $30.00$ & $30.17$ & $31.91$ \\
\texttt{EWC-FL}       & $75.98$ & $22.34$ & $50.48$ & $30.16$ & $33.18$ & $28.38$ \\
\texttt{NCCL-FL}      & $83.43$ & $17.10$ & $41.65$ & $39.23$ & $28.93$ & $9.51$ \\
\texttt{FedTrk}      & $79.86$ & $17.62$ & $29.85$& $18.11$ & $33.32$ & $28.57$ \\
\texttt{Erg-FL}  & $84.01$ & $16.50$ & $38.70$ & $43.66$ & 31.25 &32.40  \\
\texttt{C-FLAG} & {$\mf{89.28}$} & $\mf{7.23}$ & $\mf{66.85}$ & $\mf{7.30}$ & $\mf{44.30}$ & $\mf{7.65}$ \\
\bottomrule
\end{tabular}
\label{tab:performance_iid}
\end{table}
}

We also benchmark the performance of C-FLAG in the class-incremental setup by employing baselines such as LwF-FL (Learning without forgetting)~\cite{LwF}(CI), iCARL-FL~\cite{icarl}(CI) and TARGET~\cite{target}(CI). However, due to lack of space, we defer the results of the class-incremental setup to the supplementary.

\textbf{Metrics:} The metrics used to evaluate the C-FLAG and the benchmarks are as follows: \\
\text{1.~Average accuracy (Acc)}: It is the average global model accuracy of all the tasks at the end of the final task in the CFL process.\\
\text{2.~Forgetting (Fgt):} It is the average forgetting evaluated on the global model. For $S$ tasks, the forgetting is defined as $\text{Forget}=\tfrac{1}{S-1}{\textstyle\sum}_{i=1}^{S-1}(a_{i,i} - a_{S,i})$,
where $a_{i,j}$ denotes the global model accuracy of task-$j$ after training on the $i$-th task.

\begin{figure}[h!]
    \centering
        \includegraphics[scale=0.19]{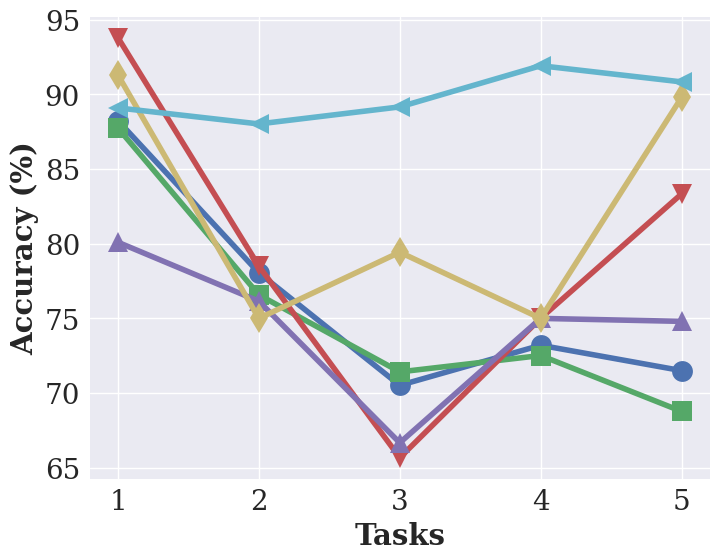}
    \includegraphics[scale=0.19]{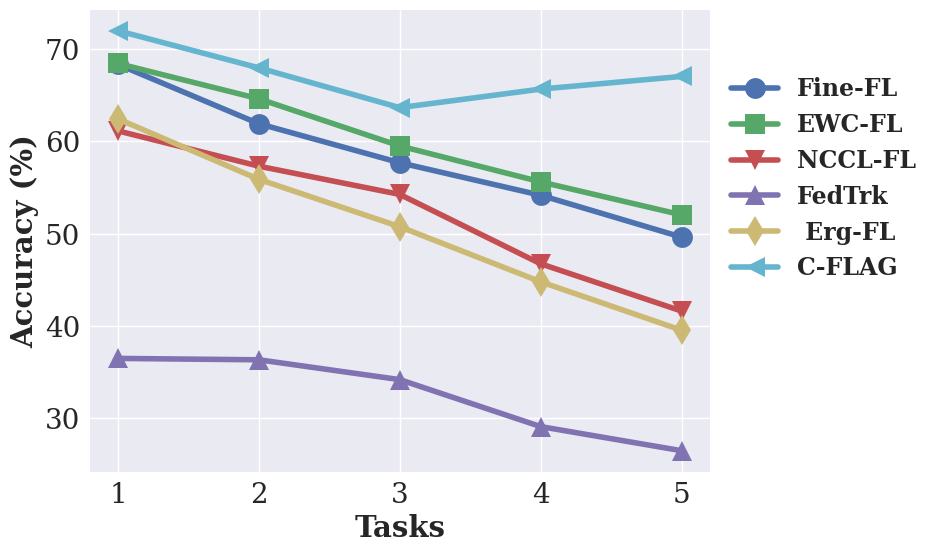}
\caption{ Average accuracy across tasks for IID splits of Split-CIFAR10 (Left) and Split-CIFAR100 (Right). }
    \label{fig:avg_task_iid}
    % \vspace{-4mm}
\end{figure}

\textbf{Comparison with Benchmarks:} We perform experiments to benchmark the performance of the proposed C-FLAG approach as compared to the task-incremental baselines listed above. From Table~\ref{tab:performance_non_iid} and Table~\ref{tab:performance_iid}, we observe that the proposed C-FLAG technique outperforms all the baselines for the Split-CIFAR10, Split-CIFAR100, and Split-TinyImageNet datasets with respect to accuracy. In particular, we see that our technique leads to a very low value of forgetting for all the datasets. All entries in Table~\ref{tab:performance_non_iid} and Table~\ref{tab:performance_iid} are averaged over $3$ seeds, with standard deviation values provided in the supplementary material.

\begin{figure}[h!]
    \centering
    \includegraphics[scale=0.29]{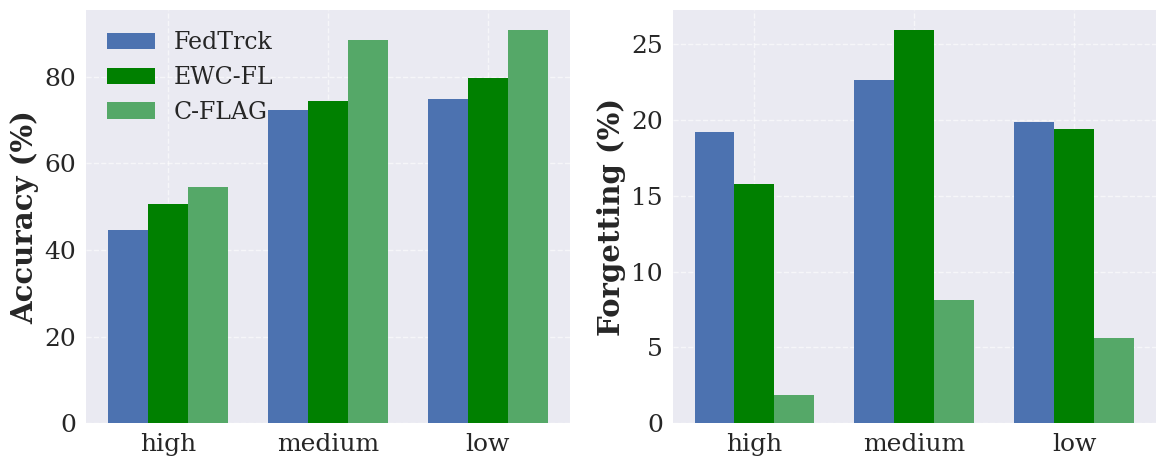}
    \includegraphics[scale=0.29]{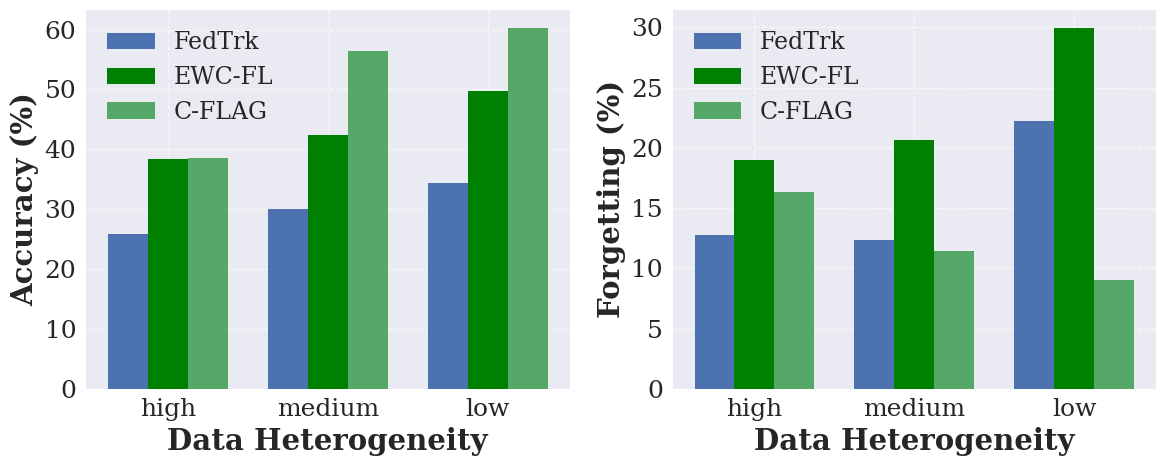}
\caption{Varying heterogeneity for C-FLAG, EWC-FL and FedTrk techniques on Split-CIFAR10 (Top) and Split-CIFAR100 (Bottom).}
    \label{fig:hetero_cifar100_cifar10}
    % \vspace{-5mm}
\end{figure}
In Fig.~\ref{fig:avg_task_iid}, we depict the average accuracy plots on Split-CIFAR10 and Split-CIFAR100 datasets for the proposed technique in the task-incremental setup as compared to the baselines. We observe that C-FLAG outperforms the baselines by demonstrating the highest average accuracy and lowest forgetting after each task. 

\textbf{Ablation Study:} We perform ablations on data heterogeneity, the number of clients, and varying memory sampling sizes. 

\begin{figure}[h!]
    \centering
    \includegraphics[scale=0.26]{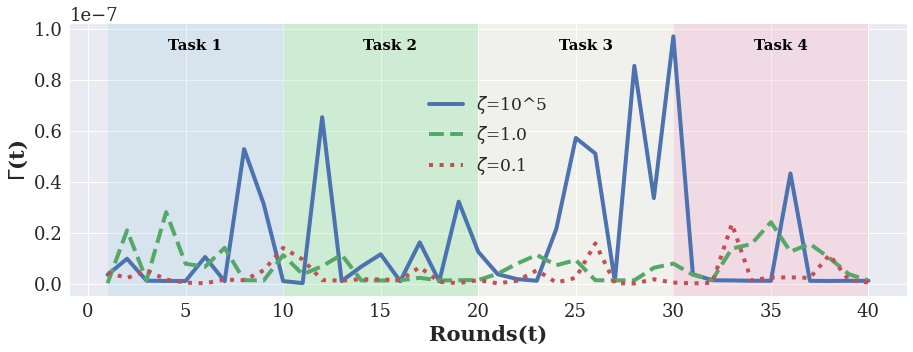}
\caption{Evolution of $\Gamma(t)$ against progressing tasks on Split-CIFAR10 dataset for varying heterogeneity.}
    \label{fig:forgetting}
\end{figure}

In Fig.~\ref{fig:hetero_cifar100_cifar10}, we illustrate the effectiveness of the proposed C-FLAG method across varying levels of data heterogeneity: low ($\zeta=10^5$), medium ($\zeta=1.0$) and high ($\zeta=0.1$), controlled by the Dirichlet parameter $\zeta$. The results clearly show that C-FLAG outperforms the EWC-FL and FedTrk baselines consistently across all levels of data heterogeneity. In Fig.~\ref{fig:forgetting}, we also examine the evolution of the forgetting term $\Gamma(t)$ as a function of communication rounds as the tasks progress. An interesting observation is that $\Gamma(t)$ in the IID case exhibits a larger variation as compared to the non-IID scenarios. This can be attributed to the fact that, in the IID case, the data distribution among clients is more uniform, which enables the global model to exhibit greater plasticity, thereby providing more opportunities for learning new patterns. Consequently, at the end of each round, the global model tends to forget more while assimilating new knowledge. In contrast, in non-IID scenarios, clients experience data drift, which limits the model's adaptability. From Fig.~\ref{fig:forgetting}, we also observe that whenever $\Gamma(t)$ shoots up, adaptive learning rates effectively help to mitigate this increase. 

%discussion on varying memeory sampling size
\begin{figure}[htp]
    \centering
    \includegraphics[scale=0.21]{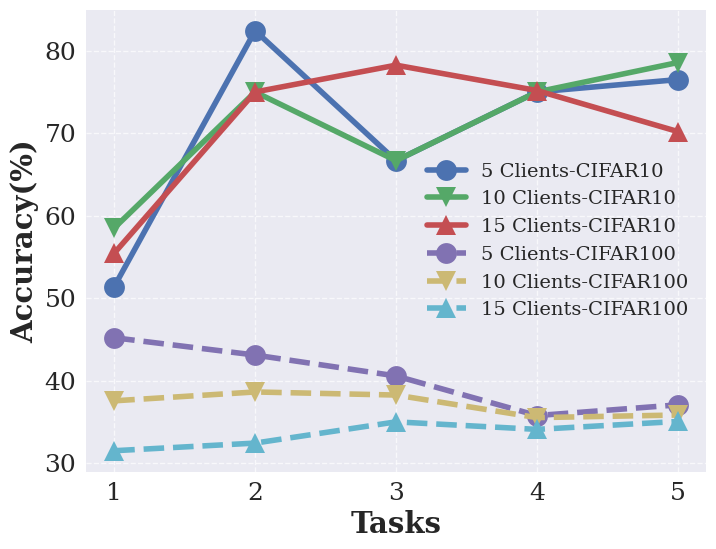}
    \includegraphics[scale=0.21]{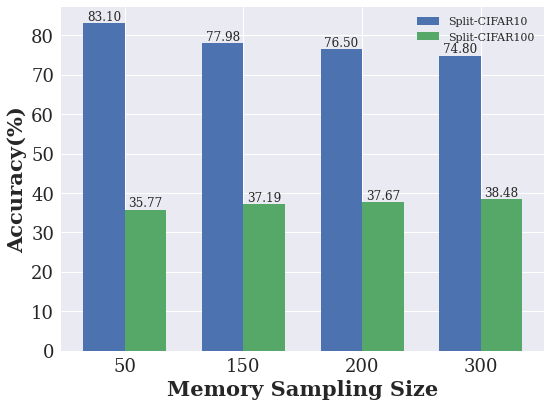}
\caption{Varying clients (left) and varying memory sample size (right) for C-FLAG on non-IID partitions of Split-CIFAR10 and Split-CIFAR100 dataset.}
    \label{fig:memory_hetero}
    % \vspace{-4mm}
\end{figure}

The scalability of C-FLAG with varying clients can be observed in Fig.~\ref{fig:memory_hetero} (Left) for both, Split-CIFAR10 and Split-CIFAR100 datasets since the accuracy performance of C-FLAG remains consistent. We analyze the effect of varying sizes of the memory sample in C-FLAG, on non-IID partitions of Split-CIFAR10 and Split-CIFAR100 datasets using sampling sizes of $50$, $150$, $200$, and $300$ per task. Note that the initial buffer size is fixed at $400$ for this experiment. We observe from Fig.~\ref{fig:memory_hetero} (Right), that for Split-CIFAR10, accuracy decreases as the sampling size increases, whereas in Split-CIFAR100, accuracy improves with larger sampling sizes. Split-CIFAR10 and Split-CIFAR100 have $2$ and $20$ classes per task respectively. As memory size increases, the model has to accommodate constraints from a larger memory data. Unlike CIFAR10, in CIFAR100, the model has enough latent representational space to accommodate the constraints due to the higher number of classes per task.

\section{CONCLUSIONS AND FUTURE WORK}
% \vspace{-3mm}
\label{sec:Conclusions}
We proposed the novel C-FLAG algorithm, which is a replay-memory based federated strategy that combines the edge-based gradient updates on memory and aggregated gradients on the current data. We also presented the convergence analysis of C-FLAG. In C-FLAG, the server initiates the transition between tasks at clients, while the previous task data is selectively sampled into a memory buffer for each client. The global parameter update of C-FLAG is an aggregation of a single step on memory and a delayed gradient-based local step on the current data at each client. We demonstrate that the convergence on the previous task is largely dependent on the suppression of the catastrophic forgetting term. We extend the analysis further by optimizing the catastrophic forgetting term and derive adaptive learning rates that ensure seamless continual learning. Empirically, we show that C-FLAG outperforms the baselines w.r.t. metrics such as accuracy and forgetting in a task-incremental setup. The limitations of C-FLAG include the need to store data in the memory. However, since each client maintains its memory buffer at the edge, privacy constraints are not violated. C-FLAG requires an additional communication for obtaining the average memory gradients $\mgrad{f}{}{}{}$ and the average gradients of the clients on current data $\grad{g}{}$ at the beginning of each server round. 
%In a typical FL setup, a large number of clients participate, making it challenging to extract client-specific information from these averages. Moreover, all clients receive the same model, which serves as the basis for computing these gradients. As a result, an adversarial client could attempt to reconstruct this information from the global model, even without access to specific gradient details.

\acknowledgments
% We acknowledge the grant from the iHub-Anubhuti-IIITD Foundation set up under the NM-ICPS scheme of the Department of Science and Technology, Govt. of India, Prime Minister's Fellowship for Doctoral Research (ANRF-FICCI) from ANRF and LightMetrics, Bengaluru, India, and the grant support from the DRDO CARS project.

We acknowledge the financial support provided by the iHub-Anubhuti-IIITD Foundation, set up under the NM-ICPS scheme of the Department of Science and Technology, Government of India. We also acknowledge the Prime Minister’s Fellowship for Doctoral Research (ANRF-FICCI) from ANRF and LightMetrics, Bengaluru, India, and the DRDO CARS project for their generous funding and support.

\bibliographystyle{plainnat}
\bibliography{citations}

\newpage

%%%%%%%%%%%%%%%%%%%%%%%%%%%%%%%%%%%%%%%%%%%%%%%%%%%%%%%%%%%%
\section*{Checklist}

 \begin{enumerate}

 \item For all models and algorithms presented, check if you include:
 \begin{enumerate}
   \item A clear description of the mathematical setting, assumptions, algorithm, and/or model. [Yes]
   \item An analysis of the properties and complexity (time, space, sample size) of any algorithm. [Yes]
   \item (Optional) Anonymized source code, with specification of all dependencies, including external libraries. [Yes]
 \end{enumerate}

 \item For any theoretical claim, check if you include:
 \begin{enumerate}
   \item Statements of the full set of assumptions of all theoretical results. [Yes]
   \item Complete proofs of all theoretical results. [Yes, in Supplementary]
   \item Clear explanations of any assumptions. [Yes]     
 \end{enumerate}

 \item For all figures and tables that present empirical results, check if you include:
 \begin{enumerate}
   \item The code, data, and instructions needed to reproduce the main experimental results (either in the supplemental material or as a URL). [Yes]
   \item All the training details (e.g., data splits, hyperparameters, how they were chosen). [Yes]
         \item A clear definition of the specific measure or statistics and error bars (e.g., with respect to the random seed after running experiments multiple times). [Yes]
         \item A description of the computing infrastructure used. (e.g., type of GPUs, internal cluster, or cloud provider). [Yes]
 \end{enumerate}

 \item If you are using existing assets (e.g., code, data, models) or curating/releasing new assets, check if you include:
 \begin{enumerate}
   \item Citations of the creator If your work uses existing assets. [Yes]
   \item The license information of the assets, if applicable. [Not Applicable]
   \item New assets either in the supplemental material or as a URL, if applicable. [Not Applicable]
   \item Information about consent from data providers/curators. [Not Applicable]
   \item Discussion of sensible content if applicable, e.g., personally identifiable information or offensive content. [Not Applicable]
 \end{enumerate}

 \item If you used crowdsourcing or conducted research with human subjects, check if you include:
 \begin{enumerate}
   \item The full text of instructions given to participants and screenshots. [Not Applicable]
   \item Descriptions of potential participant risks, with links to Institutional Review Board (IRB) approvals if applicable. [Not Applicable]
   \item The estimated hourly wage paid to participants and the total amount spent on participant compensation. [Not Applicable]
 \end{enumerate}

 \end{enumerate}

%%%%%%%%%% Supplementary %%%%%%%%%%%
\setcounter{section}{0}

\onecolumn
\aistatstitle{Supplementary Material for "On the Convergence of Continual Federated Learning Using Incrementally Aggregated Gradients"}

\section{PROOFS AND EXTENDED DISCUSSIONS}

\newstuff{

\begin{table}[!ht]
\centering
\caption{This table provides a summary of the key notations.}
% \begin{tabularx}{\textwidth}{|l|l|}
% \begin{tabularx}{\textwidth}{|@{}|XX|@{}}
% \begin{tabularx}{\textwidth}{|l|X|}
\begin{tabularx}{\textwidth}{|c|>{\centering\arraybackslash}X|}
\hline
\textbf{\centering Notation} & \textbf{\centering Description} \\ \hline
$\mathcal{P}^i$                                  & Data from all the previous tasks at the $i$-th client            \\ \hline
$\mathcal{M}^i $                                 & Memory buffer at the $i$-th client                               \\ \hline
$\mathcal{C}^i$                                  & Current dataset at the $i$-th client                             \\ \hline
$h_i(\mathbf{x})$                               & Local loss function at the $i$-th client                         \\ \hline
$f_i(\mathbf{x})$                               & Restriction of $h_i(\mf{x})$ on $\mathcal{P}^i$             \\ \hline
$g_i(\mathbf{x})$                               & Restriction of $h_i(\mf{x})$ on $\mathcal{C}^i$             \\ \hline
$ \mathbf{x}_{t,k}^{i}$ &
  Model weight at the $i$-th client after $t$-th communication round and at the end of the $k$-th local epoch \\ \hline
$\mathbf{x}_{t} $                                & Server model weight at the end of the $t$-th communication round \\ \hline
$ \nabla {g}_{i}^{\prime}(\mathbf{x}_{t,k}^{i}) $ &
  Delayed gradient at the $i$-th client after $t$-th communication round and $k$-th local epoch \\ \hline
$\tau_{k,j}^i $ &
  Index for the most recent local step at which gradient was computed for the $k$-th datapoint at the $i$-th client \\ \hline
$ {b_i}(\mathbf{x}_t) $                            & Bias at the $i$-th client                                        \\ \hline
$ \nabla{f}^{\dagger}(\mathbf{x}_t) $ & Gradient on the memory data                                      \\ \hline
$ \Gamma(t) $                                    & Global catastrophic forgetting term  \\ \hline
$ h_{|_{\mc{M}\cup \mc{C}}}(\mathbf{x}_t) $ &
  Restriction of $h(\mathbf{x}_t)$ on $\mathcal{M} \cup \mathcal{C} $ \\ \hline
\end{tabularx}
\end{table}
}% for color

In this section, we present the theoretical convergence analysis of C-FLAG, the proposed memory-based continual learning framework in a non-convex setting, along with the detailed proofs. 

First, we provide derivations of some results which will be used in the sequel to complete the proofs of the main theorems and lemmas. Notations-wise, we use the same notations as in the main manuscript. 

\textbf{Client drift after E local epochs}: At the beginning of the $t$-th global iteration, each client $i \in [N]$, receives current global model weight $\gw{t}$. Each client updates its weights for $E$ epochs to obtain $\lw{t,E}{i}$ resulting in a client drift from the current global model given as:
\begin{flalign}
     \lw{t,E}{i} - \gw{t} &= - \bet{t} \sumk{k=0}{E-1}\bigg( \grad{g}{t} - \grad{g_i}{t} + \dgrad{g}{i}{t,k}{i} \bigg) %\label{eq: LUR}\\
     \\
      &= - \bet{t} \sumk{k=0}{E-1}\bigg( \grad{g}{t} - \grad{g_i}{t} + \dgrad{g}{i}{t,k}{i} \bigg) \\
      &= - \bet{t} \sumk{k=0}{E-1}\bigg( \grad{g}{t} - \grad{g_i}{t} \bigg) - \bet{t} \sumk{k=0}{E-1} \dgrad{g}{i}{t,k}{i}  \\
     &= - \bet{t}E \bigg( \grad{g}{t} - \grad{g_i}{t} \bigg) - \bet{t} \sumk{k=0}{E-1} \dgrad{g}{i}{t,k}{i}. \label{eq: result1}
\end{flalign}
Summing up the weighted drifts across clients, we get
\begin{flalign}
    \sumk{i=1}{N}p_i (\lw{t,E}{i} - \gw{t}) &= -\bet{t}E \ub{\sumk{i=1}{N} p_i \bigg( \grad{g}{t} - \grad{g_i}{t} \bigg)}{=\, 0} -\bet{t}\sumk{i=1}{N}\sumk{k=0}{E-1}p_i \dgrad{g}{i}{t,k}{i} \\
    & = -\bet{t} \sumk{i=1}{N}\sumk{k=0}{E-1}p_i \dgrad{g}{i}{t,k}{i}. \label{eq: result2}
\end{flalign}
Hence, we see that the average client drift is a function of the aggregated delayed gradients at each client.

\textbf{Server Update Rule}:
At the end of the $t$-th global iteration, $i$-th client sends $(\lw{t,E}{i}- \alp{t}\mgrad{f}{}{t}{i})$ to the server and server aggregates its model weights to get $\gw{t+1}$ as follows:
\begin{flalign}
\gw{t+1} &= \sumk{i=1}{N}p_i (\lw{t,E}{i} - \alp{t}\mgrad{f}{}{t}{i}) \\
&\overset{\text{(a)}}{=}\gw{t} - \alp{t}  \mgrad{f}{}{t}{} - \bet{t} \sumk{i=1}{N}\sumk{k=0}{E-1}p_i \dgrad{g}{i}{t,k}{i} \label{eq: server_update-8}\\
& = \gw{t} - \alp{t} \sumk{i=1}{N}p_i \mgrad{f}{i}{t}{} - \bet{t} \sumk{i=1}{N}\sumk{k=0}{E-1}p_i \dgrad{g}{i}{t,k}{i} \label{eq: server_update} 
% \\
% & \overset{\text{(b)}}{=} \gw{t} - \alp{t} \sumk{i=1}{N}p_i \mgrad{f}{}{t}{} - \bet{t} \sumk{i=1}{N}\sumk{k=0}{E-1}p_i (\lgrad{g}{t,k}{i} + \error{t,k}{i} ),
\end{flalign}
where (a) follows from \eqref{eq: result2}.

% and (b) follows using the definition of delayed gradient $\error{t,k}{i} = \dgrad{g}{i}{t,k}{i} - \lgrad{g}{t,k}{i}$.

\subsection{Essential Lemmas and Proofs}

For purposes of clarity, we restate the assumptions as follows:

\begin{assump} %\label{asp: L-smoothness}
(L-smoothness). For all $i \in [N]$, $f_i, g_i, h_i$ are L-smooth. 
%For all $\mf{x}, \mf{y} \in \mathbb{R}^d, f_i(\mf{x}) \leq f_i(\mf{y}) + \inner{\grad{f_i}{}, \mf{x}-\mf{y}} + \frac{L}{2} \norm{\mf{x-y}}^2 $. 
\end{assump}
\begin{assump} %\label{asp: bounded-bias}
    (Bounded Bias). There exists constants $0 \leq m_i < 1 $ for all $\mf{x} \in \mathbb{R}^d$ such that $\norm{{b_i}(\mf{x})}^{2} \leq m_i\norm{\grad{f_i}{}}^2, \quad \forall i \in [N]$.
\end{assump}
\begin{assump} (Bounded memory gradient).
    There exists $r_{i} \in \mathbb{R}^{+}$, such that $\lVert{\mgrad{f}{i}{t}{}}\rVert \leq r_i \norm{\grad{g}{t}}$ for all $i \in [N]$. %\label{Asp:bounded-memory}
    %This is guaranteed by Archimedean property of $\mathbb{R}$ since $\lVert{\mgrad{f}{i}{t}{}}\rVert, \norm{\grad{g}{t}} \in \mathbb{R}^{+}$.  
\end{assump}

Additionally, we note that $g_i(\gw{})$ is defined as a sum of it component functions as $g_i(\gw{}) = \tfrac{1}{|\mc{C}^i|} {\textstyle\sum}_{j\in \mc{C}^i}{} g_{i,j} (\gw{})$, where each of the $g_{i,j}$s are also L-smooth.

In the following lemmas, we show that the expectation of the total bias term across clients, $b(\mf{x}) = \sumk{i=1}{N}p_i b_i(\mf{x})$, is zero. We also provide an upper bound on the squared norm of $b(\mf{x})$. 
\begin{lemma}
If $\mc{M}_0^i$ is uniformly sampled from $\mc{P}^i$ at each client $i \in [N]$, then
$\mathbb{E}_{\mc{M}_t}[ b(\mf{x})] = \sumk{i=1}{N}p_i\mathbb{E}_{\mc{M}_t^i}[b_i(\mf{x})] = 0 \; \forall \mf{x} \in \mathbb{R}^d$. \label{lemma:zero-bias}
\end{lemma}
\begin{proof}
From the definition of the bias function in the main manuscript~\ref{eq:biasedGrad}, we have $b_i(\mf{x}) = \mgrad{f}{i}{}{} - \grad{f_i}{} $, where the biased gradient is computed on the memory data $\mc{M}_t^i$ at $t$-th time-step for the $i$-th client. In the episodic memory (ring buffer) scheme, at each client $i \in [N]$ we construct the memory data as $\mc{M}_t^{i} = \mc{M}_0^{i} \; \forall t \in \{0,1,\cdots, T-1  \}$ and the initial memory $\mc{M}_0^{i}$ is uniformly chosen from the IID data stream of the past task dataset  $\mathcal{P}_i$. Using this, we obtain
\begin{flalign}
\mathbb{E}_{\mc{M}_t^i}[\mgrad{f}{i}{}{}] = \mathbb{E}_{\mc{M}_0^i}[\mgrad{f}{i}{}{}]  = \grad{f_i}{}, 
\end{flalign}
which leads to $\mathbb{E}_{\mc{M}_t^i}[b_i(\mf{x})] = 0 \; \forall \mf{x} \in \mathbb{R}^d$. Hence, $\mathbb{E}_{\mc{M}_t}[b(\mf{x})] = \sumk{i=1}{N}p_i  \mathbb{E}_{\mc{M}_t^i}[b_i(\mf{x})] = 0$.
\end{proof}

\begin{lemma} Given assumption~\ref{asp: bounded-bias} holds, there exists a $m \in \mathbb{R}^+ $, $\forall \mf{x} \in \mathbb{R}^d $  such that 
\begin{align}
 \norm{b(\mf{x})}^2 = \norm{\sumk{i=1}{N}p_i b_i(\mf{x})}^2 \leq  m \norm{\grad{f}{}}^2.
\end{align}
\label{lemma: 2}
    \begin{proof}
        From the assumption~\ref{asp: bounded-bias}, we have
        $\norm{b_i(x)}^2 \leq m_i \norm{\grad{f_i}{}}^2$. We assume there exists some $m \in \mathbb{R}^+$ such that 
        \begin{flalign}
        m_i \norm{\grad{f_i}{}}^2 \leq m \norm{\grad{f}{}}^2 \; \forall i \in [N]. \label{eq: uniform-bdd-bias}
        \end{flalign}
        Using this we obtain the following:
        \begin{flalign}
             \norm{b(\mf{x})}^2 &= \norm{\sumk{i=1}{N}p_i b_i(\mf{x})}^2 {\overset{\text{(a)}}{\leq}} \sumk{i=1}{N}p_i \norm{b_i(\mf{x})}^2 \overset{\text{(b)}}{\leq} \sumk{i=1}{N}p_i m \norm{\grad{f}{}}^2 = m \norm{\grad{f}{}}^2.
        \end{flalign}
        where (a) follows from the Jensen's inequality and (b) follows from \eqref{eq: uniform-bdd-bias}.
    \end{proof}
\end{lemma}

% \begin{corollary}
%     Suppose the assumptions of Lemma~\ref{lemma: 2} holds. Let $m_{min} = min \{m_i | i\in [N] \}$. Then $m \geq m_{min}$.
% \label{corollary:m}
% \begin{proof}
% % Let $m_{min} = min \{m_i | i\in [N] \}$.
% Using the definition of $\grad{f}{}$ and Jensen's inequality, we obtain
% \begin{align}
%     \norm{\grad{f}{}}^2 &\leq \sumk{i=1}{N}p_i\norm{\grad{f_i}{}}^2 \overset{(a)}{\leq} \sumk{i=1}{N}\frac{p_i m}{m_i}\norm{\grad{f}{}}^2,\label{eq:grad_f_m}
% \end{align}
% where (a) follows from Lemma~\ref{lemma: 2}. From \eqref{eq:grad_f_m} and considering $\norm{\grad{f}{}}^2 \neq 0$, we get
% \begin{align}
%    m\sumk{i=1}{N}\frac{p_i}{m_{min}} \geq m\sumk{i=1}{N}\frac{p_i}{m_i} \geq 1.
% \end{align}
% Using the last inequality, we get $m \geq m_{min}$.
% \end{proof}
% \end{corollary}

\begin{lemma}\label{lemma:avegage-memory-bound}
    Suppose that the bounded memory gradient assumption~\ref{Asp:bounded-memory} holds. Then there exists a constant $r \in \mathbb{R}^{+}$, such that $\norm{\mgrad{f}{}{t}{}} \leq r \norm{\grad{g}{t}}$ holds.
\end{lemma} 

\begin{proof}
Using the triangle inequality, the bounded memory gradient assumption, and choosing $r = max\{r_i | i \in [N]\}$, we obtain
    \begin{flalign}
    \norm{\mgrad{f}{}{t}{}} &= \norm{\sumk{i=1}{N}p_i \mgrad{f}{i}{t}{}}  \leq \sumk{i=1}{N}p_i \norm{\mgrad{f}{i}{t}{}} \leq \sumk{i=1}{N}p_i r_i \norm{\grad{g}{t}} \leq r \norm{\grad{g}{t}}
\end{flalign}
\end{proof}

\subsection{Proofs of the Lemmas and Theorems}

Since we use delayed gradients which is an estimate of the local gradient on the current task's dataset there is an associated gradient estimation error, $\error{t,k}{i}$, at each local epoch $k$ given by 
\begin{flalign}
    \error{t,k}{i} = \dgrad{g}{i}{t,k}{i} - \lgrad{g}{t,k}{i}. \label{eq: grad-error}
\end{flalign}
Using the above, we rephrase \eqref{eq: LUR} with the additional gradient error term as 
\begin{align}
    \lw{t,k+1}{i} = \lw{t,k}{i} - \bet{t} \left( \grad{g}{t} - \grad{g_i}{t} + \lgrad{g}{t,k}{i} + \error{t,k}{i} \right).\label{eq: LUR-error}
\end{align}

In this section, we provide proof of the lemmas and theorems stated in the main manuscript. First, we start with the Lem.~\ref{thm:grad-f} followed by Theorem~\ref{lemma:min-grad-f}.

\textbf{Lemma 1}: Suppose that the assumptions \ref{asp: L-smoothness}, \ref{asp: bounded-bias} hold and $\alp{t} < \frac{2}{L(1+m)}$. For the sequence $\{\gw{t}\}_{t=1}^{T}$ generated by the algorithm~\ref{alg: A} , we have
    \begin{flalign}
        \norm{\grad{f}{t}}^2  \leq \frac{1}{\alp{t}[1-\frac{L}{2}\alp{t}(1+m)]} \big( f(\gw{t}) - f(\gw{t+1}) + \bias{} +\Gamma(t) \big),
    \end{flalign} 
where $B(t)$ is the overfitting term defined as 
\begin{flalign}
    B(t) = {(L\alp{t}^2 - \alp{t})\langle{\grad{f}{t}, b(\gw{t})}\rangle + \bet{t}\langle{b(\gw{t}), {\textstyle\sum}_{i=1}^{N}{\textstyle\sum}_{k=1}^{E-1}p_i\dgrad{g}{i}{t,k}{i}}\rangle.}
\end{flalign}
Further, $\Gamma(t)$ is the forgetting term defined as
\begin{flalign}
   \Gamma(t) &= \frac{L}{2}\bet{t}^2\lVert{\textstyle\sum}_{i=1}^{N}{\textstyle\sum}_{k=1}^{E-1}p_i\dgrad{g}{i}{t,k}{i}\rVert^2 - \bet{t}(1-L\alp{t}) \langle{\mgrad{f}{}{t}{}, {\textstyle\sum}_{i=1}^{N}{\textstyle\sum}_{k=1}^{E-1}p_i\dgrad{g}{i}{t,k}{i}}\rangle. %\nonumber\\
   % &\bet{t}(1-L\alp{t}) \langle{\mgrad{f}{}{t}{}, {\textstyle\sum}_{k}^{} \avgerror{t,k}}\rangle+L\bet{t}^2\lVert{\textstyle\sum}_{i}^{N}{\textstyle\sum}_{k=1}^{E-1}p_i\error{t,k}{i}\rVert^2.
   %\label{eq:Lem1ForgettingTerm}
\end{flalign}

\textit{Proof.} We start our analysis using $L$- smoothness of $f(\cdot)$ and the server update rule \eqref{eq: server_update}. By $L$-smoothness of $f(\cdot)$, we have
\begin{flalign}
    f(\gw{t+1}) - f(\gw{t}) \leq  \ub{\inner*{\grad{f}{t}, \gw{t+1} - \gw{t}} }{A} + \ub{\frac{L}{2} \norm{\gw{t+1} - \gw{t}}^2}{B}. \label{eq:f-eqn}
\end{flalign}
We separately upper bound the terms $A$ and $B$. Using the update rule given in \eqref{eq: server_update-8} and separating the terms involving gradient over the memory and current data, we obtain
\begin{flalign}
A &= \inner*{\grad{f}{t}, \gw{t+1} - \gw{t}} \\
& = \inner{\grad{f}{t},  - \alp{t}  \sumk{i=1}{N}p_i\mgrad{f}{i}{t}{} - \bet{t} \sumk{i=1}{N}\sumk{k=0}{E-1}p_i \dgrad{g}{i}{t,k}{i} } \\
& = \ub{- \alp{t} \inner*{\grad{f}{t}, \sumk{i=1}{N}p_i\mgrad{f}{i}{t}{}}}{AI}  \ub{-\bet{t} \inner*{\grad{f}{t}, \sumk{i=1}{N}\sumk{k=0}{E-1}p_i \dgrad{g}{i}{t,k}{i}}}{AII}. \label{eq:A}
\end{flalign}
In the above, the first term ($AI$) measures the effect of the previous task's memory data gradients, and the second term ($AII$) measures the effect of the current task's delayed gradients. Considering $AI$ and using the biased gradient $\mgrad{f}{i}{}{} = \grad{f_i}{} + b_i(\mf{x})$ we have the following:
\begin{flalign}
    AI &=  - \alp{t} \inner*{\grad{f}{t}, \sumk{i=1}{N}p_i \mgrad{f}{i}{t}{}} \\
        &= -\alp{t}\inner*{\grad{f}{t}, \sumk{i=1}{N}p_i(\grad{f_i}{t}+b_i({\gw{t}}))} \\
        &= -\alp{t} \inner*{\grad{f}{t}, \grad{f}{t}} -\alp{t}\inner*{\grad{f}{t}, b(\gw{t})} \\
        &= -\alp{t}\norm{\grad{f}{t}}^2 - \alp{t}\inner*{\grad{f}{t}, b(\gw{t})}.
        \label{eq:AI}
\end{flalign}

% Additionally, for brevity we introduce four additional terms, ${\bm{\rho}}_{t,k}^i =\lgrad{g}{t,k}{i} $, $\bar{\bm{\rho}}_{t,k} = \sumk{i=1}{N}p_i {\bm{\rho}}_{t,k}^i $, $\error{t,k}{i} = \dgrad{g}{i}{t,k}{i}$
% and $\avgerror{t,k} = \sumk{i=1}{N}p_i \error{t,k}{i} $. 
Next, we simplify the term $AII$ as follows:
\begin{flalign}
    AII &= - \bet{t} \inner*{\grad{f}{t}, \sumk{i=1}{N}\sumk{k=0}{E-1}p_i \dgrad{g}{i}{t,k}{i}} \\
        &= - \bet{t} \inner*{\mgrad{f}{}{t}{} - b(\gw{t}), \sumk{i=1}{N}\sumk{k=0}{E-1}p_i \dgrad{g}{i}{t,k}{i}} \\
        & = - \bet{t} \inner*{\mgrad{f}{}{t}{} , \sumk{i=1}{N}\sumk{k=0}{E-1}p_i \dgrad{g}{i}{t,k}{i}} + \bet{t} \inner*{b(\gw{t}), \sumk{i=1}{N}\sumk{k=0}{E-1}p_i \dgrad{g}{i}{t,k}{i}}  \label{eq:AII}
        %
        % 
    % AII &= - \bet{t} \inner*{\grad{f}{t}, \sumk{i=1}{N}\sumk{k=0}{E-1} p_i \left(\lgrad{g}{t,k}{i} + \error{t,k}{i} \right)} \\
    % &= -\bet{t}\inner*{\mgrad{f}{}{t}{} - b(\gw{t}), \sumk{k=0}{E-1}\bar{\bm{\rho}}_{t,k}} - \bet{t}\inner*{\mgrad{f}{}{t}{} - b(\gw{t}), \avgerror{t,k}} \\
    % & = -\bet{t} \inner*{\mgrad{f}{}{t}{}, \sumk{k=0}{E-1}\bar{\bm{\rho}}_{t,k}} + \bet{t}\inner*{b(\gw{t}), \sumk{k=0}{E-1}\bar{\bm{\rho}}_{t,k}} - \bet{t}\inner*{\mgrad{f}{}{t}{}, \avgerror{t,k}} + \bet{t}\inner*{b(\gw{t}), \avgerror{t,k}}. \label{eq:AII}
\end{flalign}
\noindent Next we consider the term $B$, which represents the difference between two consecutive global model weights generated through our proposed algorithm~\ref{alg: A}. Using \eqref{eq: server_update}, we have the following:
\begin{flalign}
    B &=  \frac{L}{2} \norm{\gw{t+1} - \gw{t}}^2 \\
    &= \frac{L}{2} \norm{\alp{t} \sumk{i=1}{N}p_i \mgrad{f}{i}{t}{} + \bet{t} \sumk{i=1}{N}\sumk{k=0}{E-1}p_i \dgrad{g}{i}{t,k}{i}}  \\
    &= \ub{\frac{L}{2}\alp{t}^2 \norm{\sumk{i=1}{N}p_i \mgrad{f}{i}{t}{}}^2}{BI} + \ub{\frac{L}{2}\bet{t}^2\norm{\sumk{i=1}{N}\sumk{k=0}{E-1}p_i \dgrad{g}{i}{t,k}{i}}  ^2}{BII}\nonumber\\
    &+ \ub{L\alp{t}\bet{t}\inner*{\sumk{i=1}{N}p_i \mgrad{f}{i}{t}{}, \sumk{i=1}{N}\sumk{k=0}{E-1}p_i \dgrad{g}{i}{t,k}{i}  }}{BIII} \label{eq:B}
\end{flalign}
Similar to our previous approach, next we will bound each of the terms in $B$ separately. Using the definitions of biased gradients, $\grad{f}{t}$ and $b(\gw{t})$, we obtain
\begin{flalign}
    BI &= \frac{L}{2}\alp{t}^2 \norm{\sumk{i=1}{N}p_i \mgrad{f}{i}{t}{}}^2 \\
    &= \frac{L}{2}\alp{t}^2 \norm{\sumk{i=1}{N}p_i (\grad{f_i}{t} + b_i(\gw{t})) }^2 \\
    &= \frac{L}{2} \alp{t}^2 \norm{\grad{f}{t}+ b(\gw{t})}^2 \\
    &= \frac{L}{2} \alp{t}^2 \norm{\grad{f}{t}}^2 + \frac{L}{2} \alp{t}^2 \norm{b(\gw{t})}^2 + L\alp{t}^2\inner*{\grad{f}{t}, b(\gw{t})}. \label{eq:BI}
\end{flalign}
% Using the inequality $\norm{\sumk{i=1}{M}\mf{x}_i}^2 \leq M\sumk{i=1}{N}\norm{\mf{x}_i}^2$ for $M=2$, we derive a bound $BII$ as follows:
% \begin{flalign}
%     BII &= \frac{L}{2}\bet{t}^2\norm{\sumk{i=1}{N}\sumk{k=0}{E-1}p_i (\lgrad{g}{t,k}{i} + \error{t,k}{i}) }^2 \\
%     & \leq L\bet{t}^2 \norm{\sumk{i=1}{N}\sumk{k=0}{E-1}p_i \lgrad{g}{t,k}{i}}^2 + L\bet{t}^2\norm{\sumk{i=1}{N}\sumk{k=0}{E-1}p_i \error{t,k}{i} }^2  \label{eq:BII} \\
%     & \overset{\text{(a)}}{\leq} L\bet{t}^2\sumk{i=1}{N}p_i\norm{\sumk{k=0}{E-1} \lgrad{g}{t,k}{i}}^2 + L\bet{t}^2\sumk{i=1}{N}p_i\norm{\sumk{k=0}{E-1} \error{t,k}{i}}^2  \\
%     & \overset{\text{(b)}}{\leq} L\bet{t}^2E^2\sumk{i=1}{N}\sumk{k=0}{E-1}p_i\norm{\lgrad{g}{t,k}{i}}^2 + L\bet{t}^2E^2\sumk{i=1}{N}\sumk{k=0}{E-1}p_i\norm{\error{t,k}{i}}^2 \\
%     &= L\bet{t}^2E^2\sumk{i=1}{N}\sumk{k=0}{E-1}p_i\norm{\bm{\rho}_{t,k}^i}^2 + L\bet{t}^2E^2\sumk{i=1}{N}\sumk{k=0}{E-1}p_i\norm{\error{t,k}{i}}^2, \label{eq:BIIA}
% \end{flalign}
% where the last two inequalities in (a), (b) are obtained using the Jensen's inequality.
We keep the term $BII$ as it is. Finally, 
% using definitions of $\bar{\bm{\rho}}_{t,k}$ and $\avgerror{t,k}$,
we simplify $BIII$ as
\begin{flalign}
    BIII &= L\alp{t}\bet{t}\inner*{\sumk{i=1}{N}p_i \mgrad{f}{i}{t}{}, \sumk{i=1}{N}\sumk{k=0}{E-1}p_i\dgrad{g}{i}{t,k}{i}} = \sumk{k=0}{E-1} L\alp{t}\bet{t}\inner*{\mgrad{f}{}{t}{}, \sumk{i=1}{N}p_i\dgrad{g}{i}{t,k}{i}} 
\end{flalign}

Finally putting all the resulting terms in ~\eqref{eq:f-eqn}, we obtain the following:
\begin{flalign}
    f(\gw{t+1}) - f(\gw{t}) &\leq  -\alp{t}\norm{\grad{f}{t}}^2 - \alp{t}\langle{\grad{f}{t}, b(\gw{t})}\rangle - \bet{t} \inner*{\mgrad{f}{}{t}{} , \sumk{i=1}{N}\sumk{k=0}{E-1}p_i \dgrad{g}{i}{t,k}{i}} \nonumber \\
    &+ \bet{t} \inner*{b(\gw{t}), \sumk{i=1}{N}\sumk{k=0}{E-1}p_i \dgrad{g}{i}{t,k}{i}} + \frac{L}{2} \alp{t}^2 \norm{\grad{f}{t}}^2 + \frac{L}{2} \alp{t}^2 \norm{b(\gw{t})}^2 \nonumber \\
    &+ L\alp{t}^2\inner*{\grad{f}{t}, b(\gw{t})} + \frac{L}{2}\bet{t}^2\norm{\sumk{i=1}{N}\sumk{k=0}{E-1}p_i \dgrad{g}{i}{t,k}{i}}  ^2 \nonumber \\
    &+ \sumk{k=0}{E-1} L\alp{t}\bet{t}\inner*{\mgrad{f}{}{t}{}, \sumk{i=1}{N}p_i(\dgrad{g}{i}{t,k}{i}} \nonumber \\
    &\leq -\alp{t}(1-\frac{L}{2}\alp{t})\norm{\grad{f}{t}}^2 + \bias{} + \Gamma(t) + \frac{L}{2}\alp{t}^2 m\norm{\grad{f}{t})}^2 \text{ (using Lem. ~\ref{lemma: 2}) }\\
    &= -\alp{t}\big[1-\frac{L}{2}\alp{t}(1+m)\big] \norm{\grad{f}{t}}^2 + \bias{} + \Gamma(t),
\end{flalign}
where the overfitting term $B(t)$ is given as
\begin{flalign}
    B(t)= {(L\alp{t}^2 - \alp{t})\langle{\grad{f}{t}, b(\gw{t})}\rangle + \bet{t}\langle{b(\gw{t}), {\textstyle\sum}_{i=1}^{N}{\textstyle\sum}_{k=1}^{E-1}p_i\dgrad{g}{i}{t,k}{i}}\rangle}, 
\end{flalign}
and the forgetting term $\Gamma(t)$ is given as
\begin{flalign}
    \Gamma(t) = & \frac{L}{2}\bet{t}^2\lVert{\textstyle\sum}_{i=1}^{N}{\textstyle\sum}_{k=1}^{E-1}p_i\dgrad{g}{i}{t,k}{i}\rVert^2 - \bet{t}(1-L\alp{t}) \langle{\mgrad{f}{}{t}{}, {\textstyle\sum}_{i=1}^{N}{\textstyle\sum}_{k=1}^{E-1}p_i\dgrad{g}{i}{t,k}{i}}\rangle .
    % &+L\bet{t}^2\norm{\sumk{i=1}{N}\sumk{k=0}{E-1}p_i\error{t,k}{i}}^2.
\end{flalign}
Using $\alp{t} < \frac{2}{L(1+m)}$ and re-arranging the terms, we obtain
\begin{flalign}
    \norm{\grad{f}{t}}^2  \leq \frac{1}{\alp{t}[1-\frac{L}{2}\alp{t}(1+m)]} \left( f(\gw{t}) - f(\gw{t+1}) + \bias{} +\Gamma(t) \right). \label{eq:theorem-grad-f}
\end{flalign}

\textbf{Theorem 2}: Suppose that the assumptions \ref{asp: L-smoothness}, \ref{asp: bounded-bias} hold. Given $F = f(\gw{0}) - f(\gw{T})$, the sequence $\{\gw{t}\}_{t=1}^{T}$ generated by algorithm \ref{alg: A} with $\alp{t} = \alp{} = \frac{1}{L(m+1)} \; \forall t \in \{0,1, \cdots, T-1 \} $, and $m\in\mathbb{R}^{+}$, satisfies 
\begin{flalign*}
    \min\limits_{t}  \mathbb{E} \left[\norm{\grad{f}{t}}^2 \right] \leq \frac{2L(1+m)}{T} \bigg(F + \sumk{t=0}{T-1}\mathbb{E}[\Gamma(t)] \bigg).
\end{flalign*} 
% \label{lemma:min-grad-f}

\textit{Proof.} Taking expectation with respect to the choice of the memory $\mc{M}_t$ on both sides of \eqref{eq:theorem-grad-f}, we obtain
\begin{flalign}
    \mathbb{E}_{\mc{M}_t} \bigg[\norm{\grad{f}{t}}^2 \bigg]  \leq \mathbb{E}_{\mc{M}_t} \bigg[ \frac{1}{\alp{t}[1-\frac{L}{2}\alp{t}(1+m)]} \left( f(\gw{t}) - f(\gw{t+1}) + \bias{} +\Gamma(t) \right) \bigg].
\end{flalign}
Given the learning rate $\alp{t} = \alp{}$ for all $t \in \{0,1,\cdots, T-1  \}$ and taking average over $T$ iterations ($t=0$ to $T-1$), of the above, we obtain
\begin{flalign}
    \min\limits_{t}  \mathbb{E} \left[\norm{\grad{f}{t}}^2 \right] & \leq \frac{1}{T} \sumk{t=0}{T-1}\mathbb{E}\left[\norm{\grad{f}{t}}^2\right] \\
    & \leq \frac{1}{T} \sumk{t=0}{T-1} \frac{1}{\alp{}[1-\frac{L}{2}\alp{}(1+m)]} \bigg(f(\gw{t}) - f(\gw{t+1}) + \mathbb{E} [\bias{} +\Gamma(t)]\bigg) \\     
     & \leq \frac{1}{T\alp{}[1-\frac{L}{2}\alp{}(1+m)]} \bigg( f(\gw{0}) - f(\gw{T}) + \sumk{t=0}{T-1} \mathbb{E} [\bias{} +\Gamma(t)] \bigg) \\
     & \overset{\text{(a)}}{=} \frac{1}{T\alp{}[1-\frac{L}{2}\alp{}(1+m)]} \bigg( F + \sumk{t=0}{T-1} \mathbb{E} [\Gamma(t)] \bigg)  \\
     & \overset{\text{(b)}}{=} \frac{2L(1+m)}{T} \big(F + \sumk{t=0}{T-1}\mathbb{E}[\Gamma(t)] \big), \label{eq: 11}
\end{flalign}
where (a) follows from Lem.~\ref{lemma:zero-bias} and (b) follows using the step size $\alp{} = \frac{1}{L(m+1)}$.

Next, we provide the proof for Lem.~\ref{lemma:h_rate} as present in the main manuscript.

\textbf{Lemma 3}: Given assumption \ref{asp: L-smoothness}, the sequence $\{\gw{t}\}_{t=1}^{T}$ generated with learning rates $ \alp{t} = \bet{t} = \alp{} = \frac{1}{30LE}$, for the restriction of $h(\cdot)$ on $\mc{M}\cup \mc{C}$ given as $\mh(\gw{t}) = h_{|_{\mc{M}\cup \mc{C}}}(\gw{t})$ we have
\begin{flalign}
    \min\limits_{t} \lVert{\grad{\mh}{t}}\rVert^2 \leq \tfrac{60L}{T}\left(\mh(\gw{0}) - \mh(\gw{T})\right).
\end{flalign}

\textit{Proof.} We use a common learning rate $\bet{t} = \alp{t} \forall t \in \{0,1,\cdots, T-1  \}$. Using the update rules \eqref{eq: result2} and \eqref{eq: server_update} to update the global objective function on the restriction $\mc{M} \cup \mc{C}$, denoted as $\mh(\cdot)$, the update rule reduces into the following local and server update rules:
\begin{flalign}
    \text{Local:}\; \lw{t,k+1}{i} &= \lw{t,k}{i} - \bet{t} \bigg(\grad{\mh}{t} - \grad{\lmh}{t} + \lmhgrad{t,k}{i}  \bigg),\\
    \text{Server:}\; \gw{t+1} &= \gw{t} - \bet{t} \sumk{i=1}{N}p_i\grad{\mh}{t} - \bet{t} \sumk{i=1}{N}\sumk{k=0}{E-1}p_i \lmhgrad{t,k}{i},
\end{flalign}
where $\lmhgrad{t,k}{i}$ is the delayed gradient on $\mc{M}\cup \mc{C}$. Then we note that the problem is similar to Lem.~\ref{thm:grad-f} and can be proved using the L-smoothness of $h(.)$. To prove convergence we follow the same steps as in Theorem~2 of \cite{fedtrack} and obtain
%
% Then we note that, the problem reduces to a finite-sum unconstrained optimization problem present in theorem 2 of \cite{fedtrack} with an additional term of $- \bet{t}\grad{\mh}{t} = - \bet{t} \sumk{i=1}{N}p_i\grad{\lmh}{t}$. Subsequently, following the same proof steps as in \cite{fedtrack}, we obtain
%
\begin{flalign}
    \mh(\gw{t+1}) - \mh(\gw{t}) &\leq  (-\bet{t} + 10\bet{t}^2LE^2 + 80\bet{t}^4L^3E^4 -\bet{t} E)\norm{\mh{(\gw{t})}}^2  \\
    & < (-\bet{t} E + 10\bet{t}^2LE^2 + 80\bet{t}^4L^3E^4 )\norm{\mh{(\gw{t})}}^2  \label{eq:h_intermediate}\\
    &  \overset{\text{(a)}}{\leq} -\bet{t} E(1- 15\bet{t} LE)\norm{\mh{(\gw{t})}}^2,
\end{flalign}
where (a) follows from using $\bet{t} \leq \tfrac{1}{4LE}$. Rearranging the above equation and using $\bet{t} = \tfrac{1}{30LE}$, we obtain
\begin{flalign}
    \norm{\mh{(\gw{t})}}^2 &\leq \frac{1}{\bet{t} E(1- 15\bet{t} LE)}\big(\mh(\gw{t}) - \mh(\gw{t+1})\big)\\
    &= 60L \big(\mh(\gw{t}) - \mh(\gw{t+1})\big). \label{eq:norm-h}
\end{flalign}

Finally, using telescopic summing, we obtain the desired result as
\begin{flalign}
    \min\limits_{t} \norm{\grad{\mh}{t}}^2 \leq \frac{60L}{T} \sumk{t=0}{T-1} \big(\mh(\gw{0}) - \mh(\gw{T})\big).
\end{flalign}

Hence, our proposed algorithm provides convergence guarantees if it jointly learns from the current task and the replay-memory at each client. Next, we provide an upper bound on the squared norm of $\grad{g}{t}$.

We define $\mc{D}= \{\mc{D}^1, \mc{D}^2, \dots, \mc{D}^N\}$ as the collection of subsets of both the current task and memory data $\text{where each } \mc{D}^i \text{ satisfies } \mc{C}^i \subset \mc{D}^i \subset {\mc{M}^i \cup \mc{C}^i}$. Additionally, $h_{i_{|_{\mc{D}^i}}}$ is defined as the restriction of $h_i$ on $\mc{D}^i$.
\begin{lemma}\label{lemma:g-bound}
Given the sequence $\{\gw{t}\}_{t=1}^{T}$ as generated by algorithm~\ref{alg: A}, the upper bound for the global loss gradient for the current data $\mc{C} = \{\mc{C}^1, \mc{C}^2, \ldots, \mc{C}^N  \}$ satisfies
\begin{flalign}
    \norm{\grad{g}{t}}^2 = \norm{\sumk{i=1}{N}p_i \grad{g_i}{t}}^2 \leq 2 \norm{\grad{\mh}{t}}^2 + 2 \omega^2, 
\end{flalign}
where $\omega^2 := \sumk{i=1}{N}p_i \sup\limits_{\mc{C}^i \subset \mc{D}^i \subset \mc{M}^i \cup \mc{C}^i} \omega^2(h_i;\mc{D}^i) $ and $\omega^2 (h_i; \mc{D}^i) :=  \sup\limits_{\mf{x}} \norm{\grad{h_{i_{|_{\mc{D}^i}}}}{t} - \grad{\lmh}{t}}^2 $.

\begin{proof} 
Using the definitions of $\omega^2(h_i;\mc{D}^i)$ and Jensen's inequality, we obtain
\begin{flalign}
    &\sup\limits_{\mf{x}} \norm{\grad{g}{t} - \grad{\mh}{t} }^2 = \sup\limits_{\mf{x}}\norm{\grad{h_{|_{\mc{C}}}}{t} - \grad{\mh}{t}}^2 \\
    & = \sup\limits_{\mf{x}} \norm{\sumk{i=1}{N}p_i \big( \grad{{h_i}_{|_{\mc{C}^i}}}{t} - \grad{\lmh}{t} \big)}^2 \\
    & \leq \sup\limits_{\mf{x}} \sumk{i=1}{N}p_i \norm{\grad{{h_i}_{|_{\mc{C}^i}}}{t} - \grad{\lmh}{t}}^2. \label{eq:sup_intermediate}
\end{flalign}
Using $\sup\limits_{\mf{x}}\bigg(\sumk{i=1}{N} \psi_i(\mf{x}) \bigg) \leq \sumk{i=1}{N}\sup\limits_{\mf{x}}\psi_i(\mf{x})$
for $\psi_i: \mathbb{R}^d \rightarrow \mathbb{R}\; \forall i \in [N] $ in \eqref{eq:sup_intermediate}, we obtain
\begin{flalign}
    \sup\limits_{\mf{x}} \norm{\grad{g}{t} - \grad{\mh}{t} }^2 & \leq \sumk{i=1}{N}p_i \sup\limits_{\mf{x}} \norm{\grad{{h_i}_{|_{\mc{C}^i}}}{t} - \grad{\lmh}{t}}^2 \\
    & \leq \sumk{i=1}{N}p_i \sup\limits_{\mc{C}^i \subset \mc{D}^i \subset \mc{M}^i \cup C^i}  \sup\limits_{\mf{x}} \norm{\grad{{h_i}_{|_{\mc{D}^i}}}{t} - \grad{\lmh}{t}}^2 \\
    & \leq \sumk{i=1}{N}p_i \sup\limits_{\mc{C}^i \subset \mc{D}^i \subset \mc{M}^i \cup \mc{C}^i} \omega^2(h_i;\mc{D}^i) \\
    & = \omega^2
\end{flalign}
Finally, using the above result, we obtain
\begin{flalign}
    \norm{\grad{g}{t}}^2 &= \norm{\grad{g}{t} - \grad{\mh}{t} + \grad{\mh}{t}}^2 \\
    & \leq  2\norm{\grad{\mh}{t}}^2 + 2 \norm{\grad{g}{t} - \grad{\mh}{t}}^2 \\
    &\leq 2\norm{\grad{\mh}{t}}^2 + 2 \omega^2.
\end{flalign}
\end{proof}
\end{lemma}
\subsection{Convergence Analysis of $\Gamma(t)$}
To show the convergence of $\Gamma(t)$, we expand it using \eqref{eq: grad-error} to obtain
\begin{flalign}
    \Gamma(t) = & \frac{L}{2}\bet{t}^2\lVert\sumk{i=1}{N}\sumk{k=0}{E-1}p_i( \error{t,k}{i} + \lgrad{g}{t,k}{i})\rVert ^2 - \bet{t}(1-L\alp{t}) \langle{\mgrad{f}{}{t}{}, \sumk{i=1}{N}\sumk{k=0}{E-1}p_i( \error{t,k}{i} + \lgrad{g}{t,k}{i})}\rangle \label{eq:gamma-1}
\end{flalign}
% \end{proof}

Further for brevity, we denote ${\bm{\rho}}_{t,k}^i =\lgrad{g}{t,k}{i}, \bar{\bm{\rho}}_{t,k} = {\textstyle\sum}_{i=1}^{N}p_i {\bm{\rho}}_{t,k}^i, \text{ and } \avgerror{t,k} = {\textstyle\sum}_{i=1}^{N}p_i \error{t,k}{i} $. Using the inequality $\norm{\sumk{i=1}{M}\mf{x}_i}^2 \leq M\sumk{i=1}{N}\norm{\mf{x}_i}^2$ for $M=2$, the first term of $\Gamma(t)$ can be expanded as
\begin{flalign}
    &\frac{L}{2}\bet{t}^2\norm{\sumk{i=1}{N}\sumk{k=0}{E-1}p_i (\lgrad{g}{t,k}{i} + \error{t,k}{i}) }^2 \nonumber \\
    & \leq L\bet{t}^2 \norm{\sumk{i=1}{N}\sumk{k=0}{E-1}p_i \lgrad{g}{t,k}{i}}^2 + L\bet{t}^2\norm{\sumk{i=1}{N}\sumk{k=0}{E-1}p_i \error{t,k}{i} }^2  \\
    &= L\bet{t}^2\lVert\sumk{k=0}{E-1}\bar{\bm{\rho}}_{t,k}\rVert^2  + L\bet{t}^2\lVert\sumk{k=0}{E-1}\avgerror{t,k}\rVert^2. \label{eq:gamma-exp1}
    % & \overset{\text{(a)}}{\leq} L\bet{t}^2\sumk{i=1}{N}p_i\norm{\sumk{k=0}{E-1} \lgrad{g}{t,k}{i}}^2 + L\bet{t}^2\sumk{i=1}{N}p_i\norm{\sumk{k=0}{E-1} \error{t,k}{i}}^2  \\
    % & \overset{\text{(b)}}{\leq} L\bet{t}^2E^2\sumk{i=1}{N}\sumk{k=0}{E-1}p_i\norm{\lgrad{g}{t,k}{i}}^2 + L\bet{t}^2E^2\sumk{i=1}{N}\sumk{k=0}{E-1}p_i\norm{\error{t,k}{i}}^2 \\
    % &= L\bet{t}^2E^2\sumk{i=1}{N}\sumk{k=0}{E-1}p_i\norm{\bm{\rho}_{t,k}^i}^2 + L\bet{t}^2E^2\sumk{i=1}{N}\sumk{k=0}{E-1}p_i\norm{\error{t,k}{i}}^2, \label{eq:BIIA}
\end{flalign}
The second term in  $\Gamma(t)$ can be expanded as
\begin{flalign}
    &- \bet{t}(1-L\alp{t}) \langle{\mgrad{f}{}{t}{}, \sumk{i=1}{N}\sumk{k=0}{E-1}p_i( \error{t,k}{i} + \lgrad{g}{t,k}{i})}\rangle \nonumber\\
    & = - \bet{t}(1-L\alp{t}) \langle{\mgrad{f}{}{t}{}, {\textstyle\sum}_{k=1}^{E-1}\bar{\bm{\rho}}_{t,k}}\rangle - \bet{t}(1-L\alp{t}) \langle{\mgrad{f}{}{t}{}, \sumk{k=0}{E-1} \avgerror{t,k}}\rangle. \label{eq:gamma-exp2}
\end{flalign}

Using \eqref{eq:gamma-exp1} and \eqref{eq:gamma-exp2} in ~\eqref{eq:gamma-1}, we obtain
\begin{flalign}
    \Gamma(t) &= L\bet{t}^2\lVert\sumk{k=0}{E-1}\bar{\bm{\rho}}_{t,k}\rVert^2 - \bet{t}(1-L\alp{t}) \langle{\mgrad{f}{}{t}{}, \sumk{k=0}{E-1}\bar{\bm{\rho}}_{t,k}}\rangle \nonumber\\
   & - \bet{t}(1-L\alp{t}) \langle{\mgrad{f}{}{t}{}, \sumk{k=0}{E-1} \avgerror{t,k}}\rangle + L \bet{t}^2 \norm{\sumk{k=0}{E-1}\avgerror{t,k}}^2. \label{eq:gamma-expand}
\end{flalign}

\begin{lemma} 
\label{lemma:gradient-error bound}
Let assumption \ref{asp: L-smoothness} hold. For all $k \in \{0,\hdots,E-1\}$ and given any $t \in \{0,1,\cdots T-1 \}$, the gradient error~\eqref{eq: grad-error} due to the incrementally aggregated gradients in the proposed algorithm is bounded as follows:
\begin{flalign}
    \norm{\error{t,k}{i}} \leq \bet{t}LE\norm{\grad{g}{t}} + 3\bet{t}L^2E \max\limits_{0 \leq b \leq k-1} \norm{\lw{t,b}{i} - \gw{t}}.
\end{flalign}
\begin{proof}
% Using assumption \ref{asp: L-smoothness}, we have that $g_{i,j}$ are $L_j$ smooth $\forall j \in \mc{C}^i$ with $\max_{j\in \mc{C}^i} L_j \leq L$ .
Expanding $\error{t,k}{i}$ as given in \eqref{eq: grad-error}, we have

\begin{flalign}
    \norm{\error{t,k}{i}} &= \norm{{\dgrad{g}{i}{t,k}{i}} - \lgrad{g}{t,k}{i} } \\
    & = \norm{\frac{1}{| \mc{C}^i |} \sumk{j \in \mc{C}^i}{} \big\{ \nabla g_{i,j}\left(\mf{x}_{t,\tau_{k,j}^i}^{i}\right) - \nabla g_{i,j}\left(\mf{x}_{t,k}^{i} \right) \big\}}  \\
    & \overset{(a)}{\leq} \frac{1}{| \mc{C}^i |} \sumk{j \in \mc{C}^i}{} L \norm{\lw{t, \tau_{k,j}^i}{i} - \lw{t,k}{i}} %\\
    % &{{ = L \norm{\lw{t, \tau_{k,j}^i}{i} - \lw{t,k}{i}}}} 
    %%% to check %%%
    \label{eq:error_1}
\end{flalign}
where ($a$) follows from triangle inequality and L-smoothness of $g_i$.

Further simplifying, we note that
\begin{flalign}
    \frac{1}{| \mc{C}^i |} \sumk{j \in \mc{C}^i}{} L \norm{\lw{t, \tau_{k,j}^i}{i} - \lw{k,j}{i}} & \overset{(b)}{=} \frac{1}{| \mc{C}^i |} \sumk{j \in \mc{C}^i}{} L \norm{\sumk{l=\tau_{k,j}^i}{k-1}\lw{t, l}{i} - \lw{t,l+1}{i}} \\
    & \overset{(c)}{\leq}\frac{1}{| \mc{C}^i |} \sumk{j \in \mc{C}^i}{}  L \sumk{l=\tau_{k,j}^i}{k-1} \norm{\lw{t, l}{i} - \lw{t,l+1}{i}} \\
    & \overset{(d)}{\leq} \frac{1}{| \mc{C}^i |} \sumk{j \in \mc{C}^i}{} L\sumk{l=0}{k-1} \norm{\lw{t, l}{i} - \lw{t,l+1}{i}} \\
    & = L\sumk{l=0}{k-1} \norm{\lw{t, l}{i} - \lw{t,l+1}{i}},
\end{flalign}
where ($b$) follows from the properties of IAG with $ 0 \leq \tau_{k,j}^i \leq k$, ($c$) follows using triangle inequality, ($d$) follows since each entry in the summation is non-negative and $0 \leq \tau_{k,j}^i$. Using \eqref{eq: LUR}, we have 
\begin{flalign}
    & L\sumk{l=0}{k-1} \norm{\lw{t, l}{i} - \lw{t,l+1}{i}} {=} L \sumk{l=0}{k-1} \norm{\bet{t}\left(\grad{g}{t} -\grad{g_i}{t} + \dgrad{g}{i}{t,l}{i}  \right)} \\
    & = L \bet{t} \sumk{l=0}{k-1}\left[ \norm{\grad{g}{t} -\grad{g_i}{t} + \lgrad{g}{t,l}{i} + \error{t,l}{i}} \right] \\ 
    & \overset{(e)}{\leq} L \bet{t} \sumk{l=0}{k-1}\left[ \norm{\grad{g}{t}} + L \norm{\lw{t,l}{i} - \gw{t}} + \norm{\error{t,l}{i}}\right] \label{eq:error_2},
\end{flalign}
where ($e$) follows from triangle inequality and L-smoothness of $g_i$. Further, using \eqref{eq:error_1}, we obtain
\begin{flalign}
    \norm{\error{t,l}{i}}  &\leq \frac{1}{| \mc{C}^i |} \sumk{j \in \mc{C}^i}{} L \norm{\lw{t, \tau_{l,j}^i}{i} - \lw{k,j}{i}} \\
    &\overset{(a)}{\leq} \frac{1}{| \mc{C}^i |} \sumk{j \in \mc{C}^i}{} L \left[\norm{\lw{l,\tau_{l,j}^i}{i} - \gw{t}} + \norm{\lw{t,\tau_{l,j}^i}{i}  - \gw{t}}\right] \\
    & \overset{(b)}{\leq} L \max\limits_{0\leq b \leq l} \left[\norm{\lw{t,b}{i} - \gw{t}} + \norm{\lw{t,b}{i} - \gw{t}}  \right] \\
    %
    % \norm{\error{t,l}{i}}  \leq L \norm{\lw{t,\tau_{t,l}^i }{i} - \lw{t,l}{i}} &\overset{(a)}{\leq} L \left[\norm{\lw{t,l}{i} - \gw{t}} + \norm{\lw{t,\tau_{t,l}^i}{i}  - \gw{t}}\right] \\
    % & \overset{(b)}{\leq} L \max\limits_{0\leq b \leq l} \left[\norm{\lw{t,b}{i} - \gw{t}} + \norm{\lw{t,b}{i} - \gw{t}}  \right] \\
    & = 2L \max\limits_{0\leq b \leq l} \{ \norm{\lw{t,b}{i} - \gw{t}} \} \label{eq:error_3},
\end{flalign}
where $(a)$ follows using Triangle inequality and $(b)$ follows since $ 0\leq \tau_{l,j}^i\leq l$ leads to $\norm{\lw{t,\tau_{l,j}^i}{i} - \gw{t}} \leq \max\limits_{0\leq b \leq \tau_{l,j}^i} \{ \norm{\lw{t,b}{i} - \gw{t}} \} \leq \max\limits_{0\leq b \leq l} \{ \norm{\lw{t,b}{i} - \gw{t}} \}$. Substituting  \eqref{eq:error_3} in \eqref{eq:error_2}, we obtain 
\begin{flalign}
    \norm{\error{t,k}{i}} &\leq L \bet{t} \sumk{l=0}{k-1}\left[\norm{\grad{g}{t}} + L \norm{\lw{t,l}{i} - \gw{t}} + 2L \max\limits_{0\leq b \leq l} \{ \norm{\lw{t,b}{i} - \gw{t}} \} \right] \\
    & \leq L \bet{t} \sumk{l=0}{k-1} \left[ \grad{g}{t} +  L \max\limits_{0\leq b \leq l} \{ \norm{\lw{t,b}{i} - \gw{t}} + 2L \max\limits_{0\leq b \leq l} \{ \norm{\lw{t,b}{i} - \gw{t}} \} \right] \\
    & = L \bet{t} \sumk{l=0}{k-1} \left[ \grad{g}{t} +  3L \max\limits_{0\leq b \leq l} \{ \norm{\lw{t,b}{i} - \gw{t}} \right] \\
    & {\leq} \bet{t}LE \norm{\grad{g}{t}} + 3 \bet{t}L^2E \max\limits_{0\leq b \leq k-1} \{ \norm{\lw{t,b}{i} - \gw{t}}\},
\end{flalign}
where the last inequality in the above expression holds since $k \leq E$.
\end{proof}
\end{lemma}

\begin{lemma} \label{lemma:client-drift bound}
Let assumption \ref{asp: L-smoothness} hold and let $\bet{t} \leq \frac{1}{4LE}$. Then for $ k \in \{0,\hdots,E-1\}$ and given any $t \in \{0,1, \cdots T-1 \}$, the bound on the client drift is given by
\begin{flalign}
    \norm{\lw{t,k}{i} - \gw{t}} \leq 4\bet{t}k \norm{\grad{g}{t}} \leq 4\bet{t}E \norm{\grad{g}{t}}.
\end{flalign}

\begin{proof}
For any $i \in [N]$, from ~\eqref{eq: LUR} and expanding $\dgrad{g}{i}{t,k}{i}$ using \eqref{eq: grad-error}, we obtain
    \begin{flalign}
        \norm{\lw{t,k+1}{i} - \gw{t}} &= \norm{\lw{t,k}{i} - \gw{t} -\bet{t}(\grad{g}{t} - \grad{g_i}{t} + \dgrad{g}{i}{t,k}{i})} \\
        & = \norm{\lw{t,k}{i} - \gw{t} -\bet{t}(\grad{g}{t} - \grad{g_i}{t} + \lgrad{g}{t,k}{i} + \error{t,k}{i}) } \\
        & {\leq} \norm{\lw{t,k}{i} - \gw{t}} + \bet{t}\norm{\lgrad{g}{t,k}{i} - \grad{g_i}{t}} + \bet{t}\norm{\grad{g}{t}} +\bet{t}\norm{\error{t,k}{i}}.
        \label{eq:upperlm7-1}
    \end{flalign}
where the last inequality in the above follows from triangle inequality. Using L-smoothness of $g_i$ as given in assumption \ref{asp: L-smoothness} in \eqref{eq:upperlm7-1}, we have 
    \begin{flalign}
        \norm{\lw{t,k+1}{i} - \gw{t}} &{\leq} \norm{\lw{t,k}{i} - \gw{t}} + \bet{t}L\norm{\lw{t,k}{i} - \gw{t}} + \bet{t}\norm{\grad{g}{t}} +\bet{t}\norm{\error{t,k}{i}} \\
        &= (1+\bet{t}L)\norm{\lw{t,k}{i} - \gw{t}} + \bet{t}\norm{\grad{g}{t}} +\bet{t}\norm{\error{t,k}{i}} \\
        & \overset{\text{(a)}}{\leq} (1+\bet{t}L)\norm{\lw{t,k}{i} - \gw{t}} + \bet{t}\norm{\grad{g}{t}} + \bet{t}^2 LE \norm{\grad{g}{t}} + 3 \bet{t}^2L^2E \max\limits_{0\leq b \leq k-1} \{ \norm{\lw{t,b}{i} - \gw{t}} \\
        & \overset{\text{(b)}}{{\leq}} (1+\bet{t}L)\norm{\lw{t,k}{i} - \gw{t}} + 2\bet{t}\norm{\grad{g}{t}} +  \bet{t}L \max\limits_{0\leq b \leq k-1} \{ \norm{\lw{t,b}{i} - \gw{t}}\}, \label{eq:local-drift}
    \end{flalign}
where (a) follows from Lem.~\ref{lemma:gradient-error bound}, and (b) follows using $\bet{t} \leq \frac{1}{4LE} < \frac{1}{3LE}$.

For $k=0$, it is straightforward since $\lw{t,0}{i} = \gw{t}$.
Going further, we demonstrate that $\norm{\lw{t,k}{i} - \gw{t}} \leq 4\bet{t}k \norm{\grad{g}{t}}$ for $k \in \{1,2,\cdots, E-1 \}$ using mathematical induction. First, we verify for $k=1$. For $k = 1$, we have
\begin{flalign}
   \norm{\lw{t,1}{i} - \gw{t}} &= \norm{\bet{t} (\grad{g}{t} - \grad{g_i}{t} + \dgrad{g}{i}{t,0}{i})} \\
   & \overset{\text{(a)}}{\leq} \bet{t}\norm{\grad{g}{t}} + \bet{t}\norm{\grad{g_i}{t} - \dgrad{g}{i}{t,0}{i}} \\
   & \overset{\text{(b)}}{=} \bet{t}\norm{\grad{g}{t}} \label{eq:base-step-1} \leq 4\bet{t}\norm{\grad{g}{t}}.
\end{flalign}
where (a) follows using triangle inequality and (b) follows from $\dgrad{g}{i}{t,0}{i} = \tfrac{1}{|\mc{C}^i|} {\textstyle\sum}_{j\in \mc{C}^i}{} \nabla g_{i,j}(\mf{x}_{t,\tau_{0,j}^i}^{i}) = \tfrac{1}{|\mc{C}^i|} {\textstyle\sum}_{j\in \mc{C}^i}{} \nabla g_{i,j}(\gw{t}) = \grad{g_i}{t}$ because $\tau_{0,j}^i = 0$ and $\lw{t,0}{i} = \gw{t}$. Hence, the induction base step holds. 
% For $k = 1$, we have the following. Using \eqref{eq:local-drift}, we obtain
% \begin{flalign}
%     \norm{\lw{t,2}{i} - \gw{t}} & \leq (1+\bet{t}L)\norm{\lw{t,1}{i} - \gw{t}} + 2\bet{t}\norm{\grad{g}{t}} +  \bet{t}L \max\limits_{0\leq b \leq 0} \{ \norm{\lw{t,b}{i} - \gw{t}} \} \\
%     & \overset{\text{(a)}}{\leq} 2\bet{t}\norm{\grad{g}{t}} + 2\bet{t}\grad{g}{t} \\
%     & = 4 \bet{t} \norm{\grad{g}{t}},
% \end{flalign}
% where (a) follows using $\norm{\lw{t,1}{i} - \gw{t}} \leq \bet{t}\norm{\grad{g}{t}}$, the result from the case when $k=0$, and since $\bet{t}L \leq \frac{1}{4E} <1$.

We state the induction argument by first assuming that the condition is true for $k=l$, i.e.,
\begin{flalign}
    \norm{\lw{t,l}{i} - \gw{t}} \leq 4\bet{t}l \norm{\grad{g}{t}}.
\end{flalign}

Now, we need to demonstrate that the condition holds for $k=l+1$. From ~\eqref{eq:local-drift}, we obtain
\begin{flalign}
    \norm{\lw{t,l+1}{i} - \gw{t}} &\leq (1+\bet{t}L)\norm{\lw{t,l}{i} - \gw{t}} + 2\bet{t}\norm{\grad{g}{t}} +  \bet{t}L \max\limits_{0\leq b \leq l-1} \{ \norm{\lw{t,b}{i} - \gw{t}}\} \\
    & = \norm{\lw{t,l}{i} - \gw{t}} + 2\bet{t}\norm{\grad{g}{t}} +  \bet{t}L (\norm{\lw{t,l}{i} - \gw{t}} + \max\limits_{0\leq b \leq l-1} \{ \norm{\lw{t,b}{i} - \gw{t}}\}) \\
    & \overset{\text{(a)}}{\leq} 4\bet{t}l\norm{\grad{g}{t}} + 2\bet{t}\norm{\grad{g}{t}} +  \bet{t}L (4\bet{t}l + 4\bet{t}l )\norm{\grad{g}{t}} \\
    & \overset{\text{(b)}}{\leq} \big[ 4\bet{t}l +2\bet{t} +2\bet{t}(4\bet{t}LE)  \big]\norm{\grad{g}{t}} \\
    & \overset{\text{(c)}}{\leq} (4\bet{t}l +2\bet{t} + 2\bet{t} )\norm{\grad{g}{t}} = 4\bet{t}(l+1)\norm{\grad{g}{t}},
\end{flalign}
where (a) follows using the induction hypothesis, (b) follows since $l \leq E$, (c) follows from $\bet{t} \leq \frac{1}{4LE}$.

Since the statement is true for $k=1$ and true for $k=l+1$ when it is true for $k=l$. Then, using the principle of mathematical induction, the statement is true for all $k \in \{0, 1, \cdots, E-1 \}$. Finally, $\norm{\lw{t,k}{i} - \gw{t}} \leq 4\bet{t}k \norm{\grad{g}{t}} \leq 4\bet{t}E \norm{\grad{g}{t}}$ holds since $k < E$.
\end{proof}

\end{lemma}

\begin{corollary}\label{cor:error-bound}
    Suppose that $\bet{t} \leq \tfrac{1}{12LE}$. Then from the Lem.~\ref{lemma:gradient-error bound} and  Lem.~\ref{lemma:client-drift bound}, it follows that
    \begin{flalign}
        \norm{\error{t,k}{i}} \leq 2\bet{t}LE \norm{\grad{g}{t}}        
    \end{flalign} 

\begin{proof}
    
Using the result of Lem.~\ref{lemma:gradient-error bound} and Lem.~\ref{lemma:client-drift bound}, we obtain
\begin{flalign}
    \norm{\error{t,k}{i}} &\leq \bet{t}LE\norm{\grad{g}{t}} + 3\bet{t}L^2E \max\limits_{0 \leq b \leq k-1} \norm{\lw{t,b}{i} - \gw{t}} \\
    & \overset{\text{(a)}}{\leq} \bet{t}LE\norm{\grad{g}{t}} + 3\bet{t}L^2E (4 \bet{t}E \norm{\grad{g}{t}} \leq (\bet{t}LE\norm{\grad{g}{t}} + 12\bet{t}^2L^2E^2)\norm{\grad{g}{t}} \\
    & \overset{\text{(b)}}{\leq} 2\bet{t}LE \norm{\grad{g}{t}},
\end{flalign}
where (a) follows from Lem.~\ref{lemma:h_rate} and (b) follows from using $\bet{t} \leq \tfrac{1}{12LE}$.
    
\end{proof}
\end{corollary}

\begin{lemma} \label{lemma:rho-bound} %lemma:rho^2
Suppose assumption \ref{asp: L-smoothness} holds. Then at the $t$-th iteration at each client $i \in [N]$, the sequence of local updates, $\{\lw{t,k}{i}\}_{k=0}^{E-1}$, generated by our proposed algorithm satisfies the following:
\begin{flalign}
    &\text{(I) } \lVert{\sumk{k=0}{E-1}\Bar{\bm{\rho}}_{t,k}}\rVert \leq  4L\bet{t}E^2 \norm{\grad{g}{t}} + E\norm{\grad{g}{t}},\\
    % \intertext{and}
    &\text{(II) } \lVert{\sumk{k=0}{E-1}\bar{\bm{\rho}}_{t,k}}\rVert^2 \leq 32 \bet{t}^2L^2E^4\norm{\grad{g}{t}}^2 + 2E^2 \norm{\grad{g}{t}}^2.
\end{flalign}
\begin{proof}
To prove (I), we start from the definition of $\Bar{\bm{\rho}}_{t,k}$ as given after \eqref{eq:gamma-1}
% as stated after \eqref{eq:AI}
, to obtain the following: 
    \begin{flalign}
        \norm{\sumk{k=0}{E-1}\Bar{\bm{\rho}}_{t,k}} = \norm{\sumk{k=0}{E-1}\sumk{i=1}{N} p_i \lgrad{g}{t,k}{i}} {\leq} \sumk{k=0}{E-1}\norm{\sumk{i=1}{N}p_i \lgrad{g}{t,k}{i}},
    \end{flalign}
where the inequality follows from the triangle inequality. Adding and subtracting $\grad{g_i}{t}$, we have 
    \begin{flalign}
        \norm{\sumk{k=0}{E-1}\Bar{\bm{\rho}}_{t,k}} &\leq  \sumk{k=0}{E-1} \norm{\sumk{i=1}{N}p_i \big( \lgrad{g}{t,k}{i} - \grad{g_i}{t} \big)} +  \sumk{k=0}{E-1} \norm{\sumk{i=1}{N}p_i \big( \grad{g_i}{t} \big)} \\
        & \overset{\text{(b)}}{\leq}  \sumk{k=0}{E-1}\sumk{i=1}{N}p_i \norm{\lgrad{g}{t,k}{i} - \grad{g_i}{t}} + \sumk{k=0}{E-1} \norm{\grad{g}{t}} \\
        & \overset{\text{(c)}}{\leq} \sumk{k=0}{E-1}\sumk{i=1}{N}p_i L\norm{\lw{t,k}{i} - \gw{t}} + E\norm{\grad{g}{t}} \\
        & \overset{\text{(d)}}{\leq} \sumk{k=0}{E-1}\sumk{i=1}{N}p_i L (4\bet{t}E)\norm{\grad{g}{t}} + E\norm{\grad{g}{t}} \\
        & = 4L\bet{t}E^2 \norm{\grad{g}{t}} + E\norm{\grad{g}{t}},
    \end{flalign}
where (b) follows from triangle inequality and the definition of $\grad{g_i}{t}$, (c) follows from the $L$-smoothness assumption, and (d) follows from Lem.~\ref{lemma:client-drift bound}.

To prove (II), we start from the definition of $\Bar{\bm{\rho}}_{t,k}$ and use $\grad{g}{t} = \sumk{i=1}{N}p_i \grad{g_i}{t}$ to obtain the  following: 
    \begin{flalign}
    \norm{\sumk{k=0}{E-1}\bar{\bm{\rho}}_{t,k}}^2 
    & =  \norm{\sumk{k=0}{E-1}\sumk{i=1}{N} p_i {\bm{\rho}}_{t,k}^i}^2 \overset{\text{(a)}}{\leq} E\sumk{k=0}{E-1}\norm{\sumk{i=1}{N}p_i {\bm{\rho}}_{t,k}^i}^2\\
    & \leq E\sumk{k=0}{E-1}\norm{\sumk{i=1}{N}p_i \big( \lgrad{g}{t,k}{i} - \grad{g_i}{t} + \grad{g_i}{t} \big)}^2 \\
    & \leq 2E \sumk{k=0}{E-1}\norm{\sumk{i=1}{N}p_i \big( \lgrad{g}{t,k}{i} - \grad{g_i}{t} \big)}^2 + 2E \sumk{k=0}{E-1}\norm{\grad{g}{t}}^2 \\
    & \overset{\text{(b)}}{\leq} 2E \sumk{k=0}{E-1}\sumk{i=1}{N}p_i \norm{\lgrad{g}{t,k}{i} - \grad{g_i}{t}}^2 + 2E^2  \norm{\grad{g}{t}}^2\\
    & \overset{\text{(c)}}{\leq} 2E \sumk{k=0}{E-1}\sumk{i=1}{N}p_i\norm{L\big( \lw{t,k}{i} - \gw{t}\big)}^2 + 2E^2 \norm{\grad{g}{t}}^2 \\
    & \overset{\text{(d)}}{\leq} 2E \sumk{k=0}{E-1}\sumk{i=1}{N}p_i \big(4\bet{t}LE \norm{\grad{g}{t}} \big)^2 + 2E^2 \norm{\grad{g}{t}}^2 \\
    & = 32 \bet{t}^2L^2E^4\norm{\grad{g}{t}}^2 + 2E^2 \norm{\grad{g}{t}}^2,
    \end{flalign}
where (a),(b) follows from using Jensen's inequality, (c) follows from $L$-smoothness assumption, and (d) follows from Lem.~\ref{lemma:client-drift bound}.
\end{proof}
\end{lemma}

In the following lemma, we analyze the effect of the delayed gradient at the $t$-th iteration as observed on the average error across $E$ local steps.

\begin{lemma}
\label{lemma:e-bound}
The average error accumulated across $E$ local  steps due to the delayed gradient at the $t$-th iteration of the proposed algorithm satisfies
\begin{flalign}
    \norm{\sumk{k=0}{E-1}{\avgerror{t,k}}} \leq 2\bet{t}LE^2 \norm{\grad{g}{t}}.
\end{flalign}
\begin{proof}
% To prove (I), 
We start from the definition of $\avgerror{t,k}$ as given after \eqref{eq:gamma-1}, and then using triangle inequality and corollary~\ref{cor:error-bound}, we obtain
    \begin{flalign}
        \norm{\sumk{k=0}{E-1}{\avgerror{t,k}}} = \norm{\sumk{k=0}{E-1}\sumk{i=1}{N}p_i  \error{t,k}{i}} \leq \sumk{k=0}{E-1}\sumk{i=1}{N}p_i \norm{\error{t,k}{i}}  \leq \sumk{k=0}{E-1}\sumk{i=1}{N}p_i \big(2\bet{t}LE \norm{\grad{g}{t}} \big) = 2\bet{t}LE^2 \norm{\grad{g}{t}}.
    \end{flalign}
\end{proof}
\end{lemma}
Next we provide the proof of the Lem.~\ref{lemma:gammabound} present in the main manuscript.

\textbf{Lemma 4}~Suppose that the assumptions \ref{asp: L-smoothness}, \ref{asp: bounded-bias}, \ref{Asp:bounded-memory} hold and the step-sizes satisfy $\alp{t} = \alp{} < \frac{2}{L(1+m)}$ and {$\bet{t} = \bet{} < \frac{c}{\sqrt{T}} \forall t \in \{0,1,\cdots, T-1\} $} and for some $ c,m \in \mathbb{R}^{+}$. Then the following holds for the forgetting term $\Gamma(t)$:
\begin{flalign}
    \frac{1}{T}\sumk{t=0}{T-1}\mathbb{E}[\Gamma(t)] < \mc{O}\bigg(\frac{1}{T} + \frac{1}{\sqrt{T}} \bigg). 
\end{flalign}

\textit{Proof.} From \eqref{eq:gamma-expand} we obtain $\Gamma(t)$ as
    \begin{flalign}
        \Gamma(t) &= \ub{L\bet{t}^2\norm{\sumk{i=1}{N}\sumk{k=0}{E-1}\bar{\bm{\rho}}_{t,k}}^2}{T1} \ub{- \bet{t}(1-L\alp{t}) \inner*{\mgrad{f}{}{t}{}, \sumk{k=0}{E-1}\bar{\bm{\rho}}_{t,k}}}{T2} \ub{- \bet{t}(1-L\alp{t}) \inner*{\mgrad{f}{}{t}{}, \sumk{k=0}{E-1} \avgerror{t,k}}}{T3}\nonumber \\
    & \; \ub{+ L \bet{t}^2 \norm{\sumk{k=0}{E-1}\avgerror{t,k}}^2}{T4}. \label{eq:gamma-lemma}
    \end{flalign}
To complete the proof, we will bound each of the terms ($T1$, $T2$, $T3$, and $T4$) separately. Using Lem.~\ref{lemma:rho-bound} and \ref{lemma:e-bound}, we obtain
\begin{flalign}
    T1 &\leq 32 \bet{t}^4L^3E^4\norm{\grad{g}{t}}^2 + 2\bet{t}^2LE^2 \norm{\grad{g}{t}}^2 \label{eq:T1}\\
    T4 &\leq 4 \bet{t}^4L^3 E^4 \norm{\grad{g}{t}}^2 \label{eq:T4}.
\end{flalign}
To bound the other two terms ($T2$ and $T3$), we use the bounded bias assumption, Lem.~\ref{lemma:avegage-memory-bound} along with Lem.~\ref{lemma:rho-bound}, \ref{lemma:e-bound} and Cauchy-Schwartz inequality, to obtain
\begin{flalign}
    T2 &= - \bet{t}(1-L\alp{t}) \inner*{\mgrad{f}{}{t}{}, \sumk{k=0}{E-1}\bar{\bm{\rho}}_{t,k}} \leq \bet{t}(1 - L\alp{t}) \norm{\mgrad{f}{}{t}{}}\norm{\sumk{k=0}{E-1}{\Bar{\bm{\rho}}_{t,k}}} \\
    & \leq \bet{t}(1 - L\alp{t})r \norm{\grad{g}{t}} \big( 4\bet{t}LE^2 \norm{\grad{g}{t}} + E\norm{\grad{g}{t}} \big)\\
    &= \big[4\bet{t}^2(1 - L\alp{t})rLE^2 + \bet{t}(1 - L\alp{t})rE  \big] \norm{\grad{g}{t}}^2.\label{eq:T2} %\\
    % & \leq 4r\bet{t}^2L^2E^2\norm{\grad{g}{t}}^2 + r\bet{t}E\norm{\grad{g}{t}}^2 \text{ \big(assume $(1 - L\alp{t}) <1 $\big)} \label{eq: 15}
\end{flalign}

For the term $T3$, we obtain
\begin{flalign}
    T3 &= - \bet{t}(1-L\alp{t}) \inner*{\mgrad{f}{}{t}{}, \sumk{k=0}{E-1} \avgerror{t,k}} \leq \bet{t}(1-L\alp{t}) \norm{\mgrad{f}{}{t}{}}\norm{\sumk{k=0}{E-1}{\avgerror{t,k}}} \\
    & \leq \bet{t}(1-L\alp{t})r \norm{\grad{g}{t}} \big( 2\bet{t}LE^2 \norm{\grad{g}{t}} \big)= 2r \bet{t}^2(1-L\alp{t})LE^2 \norm{\grad{g}{t}}^2. \label{eq:T3} 
\end{flalign}

Using \eqref{eq:T1}, \eqref{eq:T4}, \eqref{eq:T2}, and \eqref{eq:T3}  in \eqref{eq:gamma-lemma}, we obtain

\begin{flalign}
    \Gamma(t) & \leq \ub{32 \bet{t}^4L^3E^4\norm{\grad{g}{t}}^2 + 2\bet{t}^2LE^2 \norm{\grad{g}{t}}^2}{T1}\nonumber\\
    &+ \ub{4\bet{t}^2(1 - L\alp{t})rLE^2\norm{\grad{g}{t}}^2 + \bet{t}(1 - L\alp{t})rE\norm{\grad{g}{t}}^2 }{T2} \nonumber \\
    & + \ub{2r \bet{t}^2(1-L\alp{t})LE^2 \norm{\grad{g}{t}}^2}{T3} +\ub{4 \bet{t}^4L^3 E^4 \norm{\grad{g}{t}}^2}{T4} \\
    & = \bigg[ 36\bet{t}^4L^3E^4 + 2\bet{t}^2LE^2 + r\bet{t}E(1 - L\alp{t}) \big(4\bet{t}LE + 2\bet{t}LE +1 \big) \bigg] \norm{\grad{g}{t}}^2  \\
    & \overset{\text{(a)}}{\leq} \bigg[ 36\bet{t}^4L^3E^4 + 2\bet{t}^2LE^2 + r\bet{t}E(1 - L\alp{t})\big(6\bet{t}LE +1 \big) \bigg] \bigg( 2 \norm{\grad{\mh}{t}}^2 + 2 \omega^2 \bigg)     \\
    & = \bigg[ 72\bet{t}^4L^3E^4 + 4\bet{t}^2LE^2 + 2r\bet{t}E(1 - L\alp{t}) \big(6\bet{t}LE +1 \big) \bigg] \bigg(\norm{\grad{\mh}{t}}^2 + \omega^2 \bigg) \label{eq:gamma-bound0}\\
    & \overset{\text{(b)}}{\leq} \frac{Q_t}{Z_t} [\mh(\gw{t}) - \mh(\gw{t+1}) ] + Q_t \omega^2, \label{eq: gamma-bound1}
\end{flalign}
where (a) follows using Lem.~\ref{lemma:g-bound}, (b) follows from \eqref{eq:h_intermediate}, $Q_t \triangleq 72\bet{t}^4L^3E^4 + 4\bet{t}^2LE^2 + 2r\bet{t}E(1 - L\alp{t}) \big(6\bet{t}LE +1 \big)$ and $Z_t \triangleq \bet{t}E(1 - 80\bet{t}^3L^3E^3 - 10\bet{t}LE)$.

Considering $\alp{t} = \alp{}$, $\bet{t} = \bet{}\; \forall t \in \{0,1,\cdots, T-1  \}$, we obtain $Q_t = Q = 72\bet{}^4L^3E^4 + 4\bet{}^2LE^2 + 2r\bet{}E(1 - L\alp{}) \big(6\bet{}LE +1 \big)$ and $Z_t = Z = \bet{}E(1 - 80\bet{}^3L^3E^3 - 10\bet{}LE)$. Telescoping the above from $t=0$ to $T-1$ we obtain
\begin{flalign}
    \sumk{t=0}{T-1} \Gamma(t) &\leq \sumk{t=0}{T-1} \bigg[ \frac{Q}{Z} \{\mh(\gw{t}) - \mh(\gw{t+1}) \} + Q \omega^2 \bigg] \\
    & = \frac{Q}{Z} \big[\mh(\gw{0}) - \mh(\gw{T}) \big] + QT \omega^2. %\label{eq:gamma-bound-sum}
\end{flalign}

Putting back the values of $Q, Z$ and using $\Delta_{\mh} = \mh(\gw{0}) - \mh(\gw{T})$, we obtain
\begin{flalign}
    \sumk{t}{}\Gamma(t) &\leq \frac{72\bet{}^4L^3E^4 + 4\bet{}^2LE^2 + 2r\bet{}E(1 - L\alp{}) \big(4\bet{}L^2E + 2\bet{}LE +1 \big)}{\bet{}E(1 - 80\bet{}^3L^3E^3 - 10\bet{}LE)}\Delta_{\mh} \nonumber \\
    &+ T \big[72\bet{}^4L^3E^4 + 4\bet{}^2LE^2 + 2r\bet{}E(1 - L\alp{}) \big(6\bet{}LE +1 \big)\big] \omega^2 \\
    &\leq \mc{O}\Bigg(\frac{\bet{}^3L^4 + \bet{}L + \bet{}L^2+ \bet{}L+1}{ 1-\bet{}^3L^3 -10\bet{}L} \Delta_{\mh} + T \big(\bet{}^4L^3 + \bet{}^2L + \bet{}\big(\bet{}L +1\big))\omega^2  \Bigg).
    \label{eq:GammatboundStep1}
    \end{flalign}
The $\omega$ term depends on the memory choice $\mc{D} = \{\mc{D}^1, \mc{D}^2, \ldots, \mc{D}^N\}$. To handle this randomness, we choose
\begin{flalign}
    \mc{D}^* = \argmax\limits_{C \subset D \subset P \cup C} \frac{\bet{}^3L^4 + \bet{}L^2+ 2\bet{}L+1}{1-\bet{}^3L^3 -10\bet{}L} \Delta_{\mh}
\end{flalign}
Taking expectation on both sides of \eqref{eq:GammatboundStep1}, averaging over $T$ and using $\bet{} < \frac{c}{\sqrt{T}}$ for some $c >0$, we obtain
\begin{flalign}
    \frac{1}{T}\sumk{t}{}\mathbb{E}[\Gamma(t)] &< \frac{1}{T} \mc{O}\Bigg(\frac{\bet{}^3 + \bet{} + 1 }{1-\bet{}^3 -\bet{}} \Delta_{h_{|_{\mc{D}^{*}}}} + T \big(\bet{}^4 + \bet{}^2 + \bet{}\big)\omega^2  \Bigg) < \mc{O}\Bigg( \frac{1}{T} + \frac{1}{\sqrt{T}} \Bigg).
\end{flalign}

Next, we provide the proof for the Theorem~\ref{theorem:f-rate} present in the main manuscript.

\textbf{Theorem 5}.~
Let $\{\gw{t}\}_{t=1}^T$ be the sequence generated by algorithm~\ref{alg: A}, and the step-sizes satisfy $\alp{t} = \frac{1}{L(m+1)}$ and $\bet{t} < \frac{c}{\sqrt{T}} \forall \{ 0,1, \cdots, T-1 \}$ and for some $ c,m \in \mathbb{R}^{+}$. Then, we obtain the following rate of convergence
\begin{flalign}
    \min\limits_{t}  \mathbb{E} \left[\norm{\grad{f}{t}}^2 \right] < \mc{O}\bigg(\frac{1}{\sqrt{T}}\bigg).
\end{flalign}

\textit{Proof.} From Theorem~\ref{lemma:min-grad-f} and Lem.~\ref{lemma:gammabound}, we have the following two results
\begin{flalign}
    \min\limits_{t}  \mathbb{E} \left[\norm{\grad{f}{t}}^2 \right] &\leq \frac{2L(1+m)}{T} \bigg(f(\gw{0}) - f(\gw{T}) + \sumk{t=0}{T-1}\mathbb{E}[\Gamma(t)] \bigg), \nonumber \\
    \frac{1}{T}\sumk{t}{}\mathbb{E}[\Gamma(t)] &< \mc{O}\Bigg( \frac{1}{T} + \frac{1}{\sqrt{T}} \Bigg). \nonumber
\end{flalign}
Using these together, we obtain
\begin{flalign}
    \min\limits_{t}  \mathbb{E} \left[\norm{\grad{f}{t}}^2 \right] &< \mc{O}\bigg( \frac{1}{\sqrt{T}} \bigg). \nonumber
\end{flalign}

Theorem~\ref{theorem:f-rate} gives the rate of convergence for the previous task. Next, we will analyze the IFO complexity of the C-FLAG algorithm~\ref{alg: A}.
\begin{corollary}\label{cor:ifo}
Let the step-sizes satisfy $\alp{t} = \frac{1}{L(m+1)}$ and $\bet{t} < \frac{c}{\sqrt{T}} \forall t \in \{ 0,1, \cdots, T-1 \}$, and for some $c \in \mathbb{R}^{+}$. Then the IFO complexity of the algorithm~\ref{alg: A} to obtain an $\epsilon$-accurate solution is:
\begin{flalign}
    IFO\;  calls = \mc{O}\bigg(\frac{1}{\epsilon^2}\bigg).
\end{flalign}
\begin{proof}
    Recall that the IFO complexity of an algorithm for an $\epsilon$-accurate solution is given as
    \begin{flalign}
        T(\epsilon) = \min \{T :  \min_{t}\; \mathbb{E}[\norm{\grad{f}{t}}^2] \leq \epsilon  \}
    \end{flalign}
For each step, an IFO call is utilised in calculating gradients for that step. From Theorem 2, we get

\begin{flalign}
    \min\limits_{t}  \mathbb{E} \left[\norm{\grad{f}{t}}^2 \right] \leq \frac{2L(1+m)}{T} \bigg(F + \sumk{t=0}{T-1}\mathbb{E}[\Gamma(t)] \bigg)
\end{flalign}
Hence, $ \mathbb{E} \left[\norm{\grad{f}{t}}^2 \right] \rightarrow \epsilon$ leads to $\mc{O}\bigg(\frac{1}{\epsilon^2} \bigg) $.
\end{proof}

\end{corollary}

Thus far, we presented the analysis with a constant step-size for $\alp{t}$ and diminishing step-size for $\bet{t}$. For completeness, in the sequel, we discuss the effect of constant learning rates on the overall convergence rate.

\begin{lemma}
\label{lemma:gamma-bound_avg}
Suppose that the assumptions \ref{asp: L-smoothness}, \ref{asp: bounded-bias}, \ref{Asp:bounded-memory} hold and the learning rates satisfy $\alp{t} < \frac{2}{L(1+m)}$ and $\bet{t}=\tfrac{1}{60LE}$ $\forall t \in \{ 0,1, \cdots, T-1 \}$. Then we obtain the following bound on $\Gamma(t)$:
\begin{flalign}
    \sumk{t=0}{T-1}\frac{\mathbb{E}[\Gamma(t)]}{T} < O\bigg(\frac{1}{T} + 1 \bigg). 
\end{flalign}

\begin{proof}
From \eqref{eq:gamma-expand} we obtain $\Gamma(t)$ as
    \begin{flalign}
        \Gamma(t) &= \ub{L\bet{t}^2\norm{\sumk{i=1}{N}\sumk{k=0}{E-1}\bar{\bm{\rho}}_{t,k}}^2}{T1} \ub{- \bet{t}(1-L\alp{t}) \inner*{\mgrad{f}{}{t}{}, \sumk{k=0}{E-1}\bar{\bm{\rho}}_{t,k}}}{T2} \ub{- \bet{t}(1-L\alp{t}) \inner*{\mgrad{f}{}{t}{}, \sumk{k=0}{E-1} \avgerror{t,k}}}{T3} \nonumber\\
    & \; \ub{+ L \bet{t}^2 \norm{\sumk{k=0}{E-1}\avgerror{t,k}}^2}{T4}.
    \end{flalign}

Next we proceed similar to the proof of Theorem~\ref{theorem:f-rate} provided above, and from \eqref{eq: gamma-bound1} we obtain,

\begin{flalign}
     \Gamma(t) < \frac{Q_t}{Z_t} [\mh(\gw{t}) - \mh(\gw{t+1}) ] + Q_t \omega^2,
\end{flalign}
where $Q_t = 72\bet{t}^4L^3E^4 + 4\bet{t}^2LE^2 + 2r\bet{t}E(1 - L\alp{t}) \big(6\bet{t}LE +1 \big)$ and $Z_t = \bet{t}E(1 - 80\bet{t}^3L^3E^3 - 10\bet{t}LE)$.

Considering $\alp{t} = \alp{}$, $\bet{t} = \bet{}\; \forall t \in \{0,1,\cdots, T-1  \}$, we obtain $Q_t = Q = 72\bet{}^4L^3E^4 + 4\bet{}^2LE^2 + 2r\bet{}E(1 - L\alp{}) \big( 6\bet{}LE +1 \big)$ and $Z_t = Z = \bet{}E(1 - 80\bet{}^3L^3E^3 - 10\bet{}LE)$. By summing up the above equation from $t=0$ to $T-1$ we obtain
\begin{flalign}
    \sumk{t=0}{T-1} \Gamma(t) &\leq \sumk{t=0}{T-1} \bigg[ \frac{Q}{Z} \{\mh(\gw{t}) - \mh(\gw{t+1}) \} + Q \sup\limits_{C \subset D \subset M \cup C} \omega_{h|_{D}}^2 \bigg] \\
    & = \frac{Q}{Z} \big[\mh(\gw{0}) - \mh(\gw{T}) \big] + QT \omega^2.\label{eq:gamma-bound-sum}
\end{flalign}

Assuming that $\bet{} \leq \frac{1}{2LE} $ and finally using $\bet{} = \frac{1}{60LE}$, we obtain the following inequalities
\begin{flalign}
    \frac{Q}{Z} \leq 16r(1-L\alp{}) + 22, \label{eq:Q-Z} \\
    Q \leq \frac{11}{60L} + \frac{2r}{15L}(1- L\alp{}). \label{eq:Q}
\end{flalign}

Using \eqref{eq:Q-Z} and \eqref{eq:Q} in \eqref{eq:gamma-bound-sum}, we obtain
\begin{flalign}
    \sumk{t=0}{T-1} \Gamma(t) &\leq \big[16r (1-L\alp{}) + 22)\big] \big[\mh(\gw{0}) - \mh(\gw{T}) \big]  + \bigg[\frac{11}{60L} + \frac{2r}{15L}(1- L\alp{})\bigg]T \omega^2.
\end{flalign}

Considering average over $T$ iterations and denoting $\Delta_{\mh} = \mh(\gw{0}) - \mh(\gw{T})$, we obtain
\begin{flalign}
    \frac{1}{T}\sumk{t=0}{T-1} \Gamma(t) &\leq \frac{16r (1-L\alp{}) + 22}{T} \Delta_{\mh[|_D]} + \bigg[\frac{11}{60L} + \frac{2r}{15L}(1- L\alp{})\bigg] \omega^2.
\end{flalign}
To handle the randomness of the memory choice, we choose
\begin{flalign}
    D^* = \argmax\limits_{C \subset D \subset P \cup C} \frac{16r (1-L\alp{}) + 22}{T} \Delta_{\mh}.
\end{flalign}
Taking expectations on both sides, we obtain
\begin{flalign}
    \frac{1}{T}\sumk{t=0}{T-1}\mathbb{E}[ \Gamma(t)] & < \mc{O}\Bigg(\frac{16r (1-L\alp{}) + 22}{T} \Delta_{h_{|_{D^{*}}}} + \bigg[\frac{11}{60L} + \frac{2r}{15L}(1- L\alp{})\bigg] \omega^2 \Bigg) \\
    & < \mc{O} \bigg(\frac{1}{T} +1\bigg).
\end{flalign}
So, if we take constant step-sizes for both $\alp{t}$ and $\bet{t}$, then in our proposed method, the average of the cumulative forgetting terms still converges with constant time complexity $\mc{O}(1)$. 
%Finally, Using this rate of convergence in Theorem~\ref{theorem:f-rate}, our proposed algorithm can only converge to a region where $\norm{\grad{f}{t}}^2$ = $\mc{O}(\omega^2)$.
\end{proof}
\end{lemma}

\section{PROPOSED ALGORITHM}
The proposed C-FLAG algorithm was presented in Sec.~\ref{sec:algorithm} of the main manuscript. Here, we complete the algorithm by presenting the pseudocode for \texttt{AdapLR} where $AdapFlag$ is a boolean variable with $AdapFlag = True$ implies that the algorithm employs adaptive learning rates.

\begin{minipage}{0.48\textwidth}
\begin{algorithm}[H]
	\caption{C-FLAG: Continual Federated Learning with Aggregated Gradients}
	%\label{alg: A}
	\begin{algorithmic}[1]
		\REQUIRE Step-size $\alp{}$, $\bet{}$, initial model $\gw{0}$, $Adap Flag$
  %and  initial full gradient $\grad{g}{0}$
		\ENSURE $\gw{t}$ for $t = 1,\hdots, T$
		\FOR{$t = 0, \ldots, T-1$}
            \STATE For client $i = 1, \hdots, N$, compute $\grad{g_i}{t}$ and $\mgrad{f}{i}{t}{}$, and transmit to the server. 
            \STATE Server computes $\grad{g}{t}$ and $\mgrad{f}{}{t}{}$, and broadcasts to each client.
            \FOR{client $i = 1, 2, \ldots N $ in parallel}
                % \STATE compute $\grad{g_i}{t}$ and transmit to the server
                
                \STATE Set $\lw{t,0}{i} = \gw{t} $%, compute $\mgrad{f}{i}{t,0}{i}$.
                \FOR{$k = 0, \hdots, E-1$}
                    \STATE Compute $\dgrad{g}{i}{t,k}{i}$ using \eqref{eq: def-memory-g-grad} and $\lw{t,k+1}{i}$ using \eqref{eq: LUR}.
                \ENDFOR
                % \IF{$adap\_flag$}
                \STATE $\Delta \lw{t}{i} = \texttt{AdapLR}(\lw{t,E}{i},\mgrad{f}{}{t}{}, \alp{},\bet{}, AdapFlag ) $.
                % \ELSE
                %     \STATE $\Delta \lw{t}{i} = \lw{t,E}{i} -\alp{}\mgrad{f}{}{t}{} $
                % \ENDIF
                \STATE Transmit $\Delta \lw{t}{i}$ to the server.
            \ENDFOR
            \STATE Server computes and broadcasts $\gw{t+1}$ using $\gw{t+1} = \gw{t}- \sumk{i=1}{N}p_i\Delta \lw{t}{i}$.
        \ENDFOR
	\end{algorithmic}
\end{algorithm}
\end{minipage}
\hfill
\begin{minipage}{0.48\textwidth}

\begin{algorithm}[H]
    \caption{AdapLR}
    \label{alg:adaptive}
    \begin{algorithmic}[1]
        \REQUIRE $ \lw{t,E}{i},\mgrad{f}{}{t}{}, \alp{},\bet{}, AdapFlag $
        \ENSURE $\Delta \lw{t}{i}$
        \IF{not $AdapFlag$}
        \RETURN $\Delta \lw{t}{i} = \alp{t}\mgrad{f}{}{t}{}+ \lw{t,E}{i}$
        \ENDIF
        \STATE Compute $\mf{a}_i = \gw{t} - E\bet{t}(\grad{g}{t} - \grad{g_i}{t}) - \lw{t,E}{i}$.
        \STATE Compute $ \Lambda_{t,i} = \langle{\mgrad{f}{}{t}{}, \mf{a}_i}\rangle $.
        \IF{$ \Lambda_{t,i} > 0$}
            \STATE $\alp{t,i} = \alp{} $
            \STATE $\bet{t,i} = \frac{(1-L\alp{})\Lambda_{t,i}}{LN p_i\norm{\mf{a}_i}^2} $ // Worst case
        \ELSE
            \STATE $\alp{t,i} =\alp{}(1- \frac{\Lambda_{t,i}}{\norm{\mgrad{f}{}{t}{}}^2})$
            \STATE $\bet{t,i} = \bet{} $
        \ENDIF
        \STATE $\mf{a}_i = \frac{\bet{t,i}}{\bet{}}*\mf{a}_i $ // re-scaling with adaptive rate
        % \STATE $\lw{t,E}{i} =  E\bet{t,i}(\grad{g}{t} - \grad{g_i}{t}) + \mf{a}_i$
        \STATE $\Delta \lw{t}{i} = \alp{t,i}\mgrad{f}{}{t}{} + E\bet{t,i}(\grad{g}{t} - \grad{g_i}{t}) + \mf{a}_i$ 
        % \STATE $\Delta \lw{t}{i} = \lw{t,E}{i} - \alp{t,i}\mgrad{f}{}{t}{} $
        \RETURN $\Delta \lw{t}{i}$
    \end{algorithmic}
\end{algorithm}
\end{minipage}

\section{C-FLAG WITH ADAPTIVE LEARNING RATES}

In Sec.~\ref{sec:adapLR} of the main manuscript, we presented the analysis regarding the effect of choosing adaptive learning rates in C-FLAG. Essentially, this translated the convergence analysis into an easily implementable solution where learning rates are adapted to achieve lower forgetting at each iteration. In the following subsections, we provide detailed discussions on the average and the worst-case scenarios. The results obtained here are encapsulated in Table~\ref{tab:adaptive_rates} of the main manuscript. 

\subsection{Average Case}
In the average case, we denote the catastrophic forgetting as $\Gamma_{av}(t)$ and $\Gamma_{i, av}(t)$ for the server and the $i$-th client respectively. Using the average case condition given as $\sumk{i=1}{N}\sumk{j=1, j \neq i}{N} p_i p_j C_{i,j} =0$, the forgetting terms can be rewritten as 
\begin{align}
    \Gamma(t) &= \Gamma_{av}(t) = \frac{L\bet{t}^2}{2} \sumk{i=1}{N}p_i^2 \norm{\mf{a}_i}^2  - \bet{t}(1-L\alp{t})\sumk{i=1}{N}p_i \Lambda_{t,i}  \nonumber\\
    &= \sumk{i=1}{N}p_i \bigg[ \frac{L\bet{t}^2}{2} p_i \norm{\mf{a}_i}^2  - \bet{t}(1-L\alp{t}) \Lambda_{t,i} \bigg]  =  \sumk{i=1}{N}p_i \Gamma_{i, av}(t),\label{eq:gamma_avg}
\end{align}
where $\Gamma_{i, av}(t) = \frac{L\bet{t}^2}{2} \sumk{i=1}{N}p_i \norm{\mf{a}_i}^2  - \bet{t}(1-L\alp{t})\sumk{i=1}{N} \Lambda_{t,i}$. Subsequently, taking expectation on both sides of \eqref{eq:gamma_avg}, we obtain
\begin{align}
    \mathbb{E}[\Gamma_{av}(t)] = \sumk{i=1}{N}p_i \mathbb{E}[\Gamma_{i, av}(t)].\label{eq:gamma_avg_expected}
\end{align}
We aim to minimize $\Gamma_{av}(t)$ by minimizing individual clients' contribution, $\Gamma_{i, av}(t)$. We achieve this by obtaining client-specific adaptive learning rates $\alp{t,i}$ and $\bet{t,i}$ for $i \in [N]$. Since $\Gamma_{i,av}(t)$ is a quadratic polynomial in $\bet{t}$, we obtain its solution $\bet{t,i}^{*} = \frac{(1-L\alp{t})\Lambda_{t,i}}{L p_i\norm{\mf{a}_i}^2}$ which leads to $\mathbb{E}[\Gamma_{i,av}^*] = -\frac{[(1-L\alp{t})\Lambda_{t,i}]^2}{2L p_i \norm{\mf{a}_i}^2}$ for a fixed $\alp{t}$. For the transference case ($\Lambda_{t,i} >0$), we choose $\bet{t,i}= \bet{t,i}^*$ and $\alp{t,i}=\alp{t} = \alp{}$, which is a feasible solution to our optimization problem~\ref{eq:gamma-min} as in the main manuscript, and consequently, $i$-th client's contribution towards the global forgetting becomes negative since $\mathbb{E}[\Gamma_{i,av}^*] <0 $. However, in the case of interference ($\Lambda_{t,i}\leq 0$), $\bet{t,i}^*$ is negative, which violates the constraint of our optimization problem. Moreover, $\Gamma_{i,ad}(t)$ is monotonically increasing in $\bet{t}$ for any feasible $\bet{t}$. Hence, we propose to adapt the $\alp{t}$ to find $\alp{t,i}$. Recall that each client takes $E$ local gradient steps on the current data before server aggregation. Consequently, the overall effect of $i$-th client's accumulated gradients, $\mathbf{a}_i$, on the direction of $\mgrad{f}{}{t}{}$, is given by $\inner{\frac{\mgrad{f}{}{t}{}}{\norm{\mgrad{f}{}{t}{}}}, \mf{a}_i }\frac{\mgrad{f}{}{t}{}}{\norm{\mgrad{f}{}{t}{}}} = \frac{\Lambda_{t,i}}{\norm{\mgrad{f}{}{t}{}}^2} \mgrad{f}{}{t}{} $. To negate this effect, we propose to adapt $\alp{t}$, at the $i$-th client as $\alp{t,i} = \alp{}(1- \frac{\Lambda_{t,i}}{\norm{\mgrad{f}{}{t}{}}^2})$ and keep $\bet{t,i} = \alp{}$.
%We draw the motivation behind this approach from the scheme proposed in NCCL~\cite{nccl} in the centralized setting. 
%{\color{red}Additionally, in subsection~\ref{sec:adaptive_alpha} we discuss how the adaptive choice of $\alp{t}$ is equivalent to the proposed surrogate function for non-increasing catastrophic forgetting in A-GEM~\cite{a-gem}}.

\subsection{Worst Case: $C_{i,j} >0$}

In the worst case, we denote the catastrophic forgetting as $\Gamma_{w}(t)$ and $\Gamma_{i, w}(t)$ for the server and the $i$-th client respectively. The forgetting terms can be rewritten as follows:
\begin{align}
    \Gamma(t) &= \Gamma_{w}(t) \\
    & = \frac{L\bet{t}^2}{2} \bigg[\sumk{i=1}{N}p_i^2 \norm{\mf{a}_i}^2 + \sumk{i=1}{N}\sumk{{\substack{j=1 \\ j \neq i}} }{N} p_i p_j C_{i,j} \bigg]  - \bet{t}(1-L\alp{t})\sumk{i=1}{N}p_i \Lambda_{t,i} \\
    &\overset{\text{(a)}}{\leq} \frac{L\bet{t}^2}{2} \sumk{i=1}{N}p_i^2 N \norm{\mf{a}_i}^2  - \bet{t}(1-L\alp{t})\sumk{i=1}{N}p_i \Lambda_{t,i} \\
    &= \sumk{i=1}{N} p_i \bigg[\frac{L\bet{t}^2}{2} p_i N \norm{\mf{a}_i}^2  - \bet{t}(1-L\alp{t}) \Lambda_{t,i}   \bigg] \\
    & = \sumk{i=1}{N}p_i \Gamma_{i, w}(t)
 \label{eq:gamma_worst},
\end{align}
where (a) follows from the inequality
$2 p_i p_j C_{i,j} = 2\inner{p_i a_i, p_j a_j} \leq p_i^2\norm{a_i}^2 + p_j^2 \norm{a_j}^2$ and we denote $\Gamma_{i, w}(t) = \frac{L\bet{t}^2}{2} \sumk{i=1}{N}p_i N \norm{\mf{a}_i}^2  - \bet{t}(1-L\alp{t})\sumk{i=1}{N} \Lambda_{t,i}$. Subsequently, taking expectation on both sides of \eqref{eq:gamma_worst}, 
% and noting that the memory choice of each client is independent of each other
we obtain
\begin{align}
    \mathbb{E}[\Gamma_{w}(t) &] \leq \sumk{i=1}{N}p_i \mathbb{E}[\Gamma_{i, w}(t)].\label{eq:gamma_worst_expected}
\end{align}
Similar to the previous section, we analyze and minimize the contribution of each client individually. 
In the worst case, for the $i$-th client, $\mathbb{E}[\Gamma_{i, w}(t)] = \mathbb{E}[ \frac{L\bet{t}^2}{2} p_i N \norm{\mf{a}_i}^2  - \bet{t}(1-L\alp{t}) \Lambda^i ]$, and the optimal learning rate $\bet{t,i}^{*} = \frac{(1-L\alp{t})\Lambda_{t,i}}{LN p_i\norm{\mf{a}_i}^2}$ and $\mathbb{E}[\Gamma_{i,w}^*] = -\frac{[(1-L\alp{t})\Lambda_{t,i}]^2}{2LN p_i \norm{\mf{a}_i}^2}$. Observe that the optimal choice of $\bet{t,i}^{*}$  leads to $\mathbb{E}[\Gamma_{i,w}^*] <0 $. Similar to the average case, in the case of interference ($\Lambda_{t,i}\leq 0$), we propose $\alp{t,i} = \alp{}(1- \frac{\Lambda_{t,i}}{\norm{\mgrad{f}{}{t}{}}^2})$. Then clients can scale their learning rates by either $\alp{t,i} = \alp{}$ or $\bet{t,i} = \alp{}$ at the end of every $E$ local epochs. 
% In the main manuscript, we discuss how this choice is similar to the surrogate function proposed in A-GEM~\cite{a-gem}.

\subsection{Analysis of Adaptive Rates}

From the previous sections, we see that clients having transference effect by adapting $\bet{t}$ to $\bet{t,i}^{*}$ leads to minimized catastrophic forgetting in both, the average and the worst case. In this section, first, we show that in the interference case our proposed choice of adaptive $\alp{t,i}$ leads to lesser forgetting, that is $\mathbb{E}[\Gamma_{i, ad}(t) ] \leq \mathbb{E}[\Gamma_{i}(t)]$. This holds for both the average and worst-case scenarios.

Next, we provide the proofs for Lem.~\ref{lemma:less-inter} and Lem.~\ref{lemma:less-gamma}.

\textbf{Lemma 6}. In the case of interference at the $i$-th client, for both the average and worst cases, adaptive rates $\alp{t,i} = \alp{}(1- \frac{\Lambda_{t,i}}{\norm{\mgrad{f}{}{t}{}}^2})$ and $\bet{t,i} = \alp{}$ lead to smaller forgetting, that is $\mathbb{E}[\Gamma_{i, ad}(t) ] \leq \mathbb{E}[\Gamma_{i}(t)]$.

\textit{Proof.} In the case of interference we have $\Lambda_{t,i} \leq 0$ and $\alp{t,i} = \alp{}(1- \frac{\Lambda_{t,i}}{\norm{\mgrad{f}{}{t}{}}^2})$. Using these we obtain 
         \begin{align}
            \alp{t,i} = \alp{t}(1- \frac{\Lambda_{t,i}}{\norm{\mgrad{f}{}{t}{}}^2}) \geq \alp{t}
        \end{align}
        From this, we get
        \begin{align}
            1 - L\alp{t,i } \leq 1 - L\alp{t}
        \end{align}
        Next, multiplying both sides by $-\bet{t}\Lambda_{t,i}$, we get 
        \begin{align}
            -\bet{t}(1 - L\alp{t,i})\Lambda_{t,i} \leq -\bet{t}(1 - L\alp{t})\Lambda_{t,i},
         \end{align}
         Finally, adding $\frac{L\bet{t}^2}{2} \sumk{i=1}{N}p_i^2 \norm{\mf{a}_i}^2$ or  $\frac{L\bet{t}^2}{2} \sumk{i=1}{N}p_i^2 N \norm{\mf{a}_i}^2$ on both sides based on the average or worst case it leads to $\mathbb{E}[\Gamma_{i, ad}(t) ] \leq \mathbb{E}[\Gamma_{i}(t)]$. 

\textbf{Lemma 7}. Client-wise adaptive rates lead to improved forgetting, $\mathbb{E}[\Gamma_{ad}(t)] \leq \mathbb{E}[\Gamma(t)] $.

\textit{Proof.} Suppose that out of $N$ clients, for $r$ clients we observe interference($\Lambda_{t,i} \leq 0$), and for the rest $(N-r)$ clients, we observe transference ($\Lambda_{t,i} >0 $). Without loss of generality, consider the first $r$ clients are interfering, that is $\Lambda_{t,i} \leq 0 $ for $i \in \{1, 2, \cdots, r \}$ and $\Lambda_{t,i} > 0 $ for $i \in \{r+1, r+2, \cdots, N \}$. Hence, we obtain
        \begin{align}
           \mathbb{E}[\Gamma_{ad}(t)] &=  \sumk{i=1}{N}p_i \mathbb{E}[\Gamma_{i,ad}(t)] \\
           % &= \sumk{i=1}{r}p_i \mathbb{E}[\Gamma_{i, int}(t)] + \sumk{i=r+1}{N}p_i \mathbb{E}[\Gamma_{i, tran}(t)] \\
           &= \sumk{i=1}{r} p_i \mathbb{E}[\Gamma_{i,ad}(t)] + \sumk{i=r+1}{N}p_i \mathbb{E}[\Gamma_{i,ad}(t)] \\
            & \overset{(a)}{\leq} \sumk{i=1}{r}p_i \mathbb{E}[\Gamma_{i}(t)] + \sumk{i=r+1}{N}p_i \mathbb{E}[\Gamma_{i}^{*}(t)] \\
            % & \overset{(b)}{<} \sumk{i=1}{N}p_i \mathbb{E}[\Gamma_i(t)] + \sumk{i=1}{N}p_i \mathbb{E}[\Gamma_i(t)] \\
            & \overset{(b)}{\leq} \sumk{i=1}{r}p_i \mathbb{E}[\Gamma_{i}(t)] + \sumk{i=r+1}{N}p_i \mathbb{E}[\Gamma_{i}(t)] \\
            &=  \sumk{i=1}{N}p_i \mathbb{E}[\Gamma_i(t) ]\\
            & = \mathbb{E}[\Gamma(t)],
            % &= \mathbb{E}[\Gamma (worst)],
        \end{align}
        where (a) follows from Lem.~\ref{lemma:less-inter} and noting that for both the average and worst cases, $\mathbb{E}[\Gamma_{i,ad}(t)] = \mathbb{E}[\Gamma_i^{*}(t)]$, and (b) follows from $\mathbb{E}[\Gamma_{i}^{*}(t)] \leq \mathbb{E}[\Gamma_{i}(t)]$ for $i \in \{r+1, r+2, \cdots, N \}$. This proves that for both the average and the worst cases, using adaptive learning rates is better than constant learning rates.

% Lemma\ref{lemma:less-gamma} shows that with the adaptive rates the overall forgetting gets reduced than without the adaptive rates in the worst case and the average case follows similarly.

\section{ADDITIONAL EXPERIMENTS}

In this section, we present additional experimental results. Unless specified otherwise, we run all tasks for $20$ communication rounds each, allowing each client to perform $2$ epochs of local updates. We consistently use a federated setup with $5$ clients, maintaining a $100\%$ client participation rate. For Split-CIFAR10 and Split-CIFAR100, we use a mini-batch size of $128$, whereas for Split-TinyImageNet, the mini-batch size is set to $32$. In this work, the theory of C-FLAG applies to training using gradient descent (GD). However, employing mini-batch GD is a common practice in machine learning to reduce the computational load on edge devices. For this, we adopt the ADAM optimizer for every local client. We use a class-balanced ring buffer memory with an initial size of $400$ samples per task for each client while sampling $200$ data points from the buffer at each step. Please note that the values in all the plots in this section and the main manuscript correspond to results from a single seed, while the values in the table represent averages over three seeds. The details of the hyperparameters used are provided in Table~\ref{tab:hyperparams}. We utilised the Avalanche library~\cite{avalanche} to generate different task sequences from the Split-CIFAR10, Split-CIFAR100, and Split-TinyImageNet datasets, which consist of 5, 5, and 10 tasks, respectively. The value of $L$, which is used for calculating adaptive learning rates, is chosen to be 5. We trained our models on NVIDIA A100-PCIE-40GB GPU. In the following subsections, we compare the performance of the proposed C-FLAG with the baselines presented in the main manuscript for the task-incremental and class-incremental settings.

\begin{table*}[h!]
\centering
\caption{Training hyperparameters for different datasets.}\label{tab:hyperparams}
\begin{tabular}{|c|c|c|c|}
\hline
Hyperparameter                    & Split-CIFAR10 & Split-CIFAR100 & Split-TinyImageNet  \\ \hline
Learning rate                 & 0.0001    & 0.0001    & 0.0001    \\ \hline
Seed                              &  1234, 1235, 1236 &  1234, 1235, 1236 & 1234, 1235, 1236 \\ \hline
\end{tabular}
\end{table*}

% \tiny{($\uparrow$)}
% \tiny{($\downarrow$)}

\begin{table}[!h]
\centering
\caption{\textbf{Task incremental}: Performance metrics on different datasets for non-IID and IID settings with $5$ clients and $5$, $5$, $10$ tasks, respectively. \emph{Acc} and \emph{Forget} denote average classification accuracy and average forgetting.}
\begin{tabularx}{\textwidth}{l*{8}{>{\centering\arraybackslash}X}}
\toprule
 & \multicolumn{2}{c}{Split-CIFAR10} & \multicolumn{2}{c}{Split-CIFAR100} & \multicolumn{2}{c}{TinyImageNet}  \\
\cmidrule(lr){2-3} \cmidrule(lr){4-5} \cmidrule(lr){6-7}
 & Acc\tiny{($\uparrow$)} & Forget\tiny{($\downarrow$)} & Acc\tiny{($\uparrow$)} & Forget\tiny{($\downarrow$)} & Acc\tiny{($\uparrow$)} & Forget\tiny{($\downarrow$)} \\
\midrule
\multicolumn{7}{c}{\textbf{Non-IID Setting}} \\
\midrule
Finetune-FL  & $54.10\pm5.78$ & $6.08\pm8.60$ & $38.66\pm1.60$ & $20.87\pm0.82$ & $23.72\pm1.82$ & $23.85\pm0.92$ \\
EWC-FL       & $53.55\pm4.56$ & $5.22\pm7.46$ & $38.83\pm1.33$ & $19.20\pm1.33$ & $26.94\pm1.91$ & $20.29\pm0.26$ \\
% LwF-FL       & $55.81\pm5.81$ & $6.30\pm6.30$ & $50.77\pm3.00$ & $7.20\pm1.04$  & $32.07\pm3.09$ & $15.80\pm2.47$ \\
NCCL-FL       & $63.35\pm7.62$             & $12.37\pm0.83$            & $32.25\pm0.95$             & $29.49\pm1.46$             & $28.49\pm1.19$             & $\mf{8.73\pm0.28}$             \\
% FedTrack       & $80.50\pm6.77$             & $13.58\pm3.98$           & $23.16\pm1.87$             & $15.54\pm1.97$             & $6.72\pm0.21$            & $3.63\pm1.14$\\
Erg-FL  & $\mf{79.11\pm 7.9}$ & $8.32\pm 7.42$ & $31.25\pm1.05$ & $32.40\pm 1.69$ & $21.90\pm 0.96$ & $19.67\pm 0.98$ \\
C-FLAG       & $65.02\pm6.22$            & $\mf{5.82\pm0.9}$            & $\mf{43.47\pm1.64}$             & $\mf{16.76\pm1.85} $            & $2\mf{8.63\pm0.91}$            & $9.52\pm0.89$             \\
\midrule
\multicolumn{7}{c}{\textbf{IID Setting}} \\
\midrule
Finetune-FL  & $72.64\pm0.98$  & $26.51\pm1.31$ & $49.82\pm0.51$ & $30.00\pm1.32$ & $30.17\pm0.69$ & $31.91\pm0.22$ \\
EWC-FL       & $75.98\pm5.10$  & $22.34\pm4.19$ & $50.48\pm1.10$ & $30.16\pm0.81$ & $33.18\pm0.99$ & $28.38\pm0.48$ \\
% LwF-FL       & $92.16\pm2.18$  & $1.75\pm1.08$  & $70.50\pm0.67$ & $2.60\pm0.22$  & $52.59\pm0.46$ & $6.38\pm0.65$  \\
NCCL-FL       & $83.43\pm5.38$             & $17.10\pm4.64$            & $41.65\pm0.58$             & $39.23\pm0.62$             & $28.93 \pm 0.73$             & $9.51\pm0.98$             \\
FedTrack       & $79.86\pm 7.16$             & $17.62\pm4.46 $          & $29.85\pm3.32$             & $18.11\pm3.15$             & $33.32\pm1.43 $         & $28.57\pm4.04$            \\
Erg-FL  & $84.01\pm9.96$ & $16.50\pm 11.57$ & $38.70\pm1.35$ & $43.66\pm1.48$ & $31.25\pm 1.05$ & $32.40 \pm 1.69$ \\
C-FLAG       & $\mf{89.28\pm4.98}$              & $\mf{7.23\pm5.14}$             & $\mf{66.85 \pm 0.77}$             & $\mf{7.30 \pm 0.81} $            &    $\mf{44.3\pm 0.71}$        &  $\mf{7.65\pm0.78}$          \\
\bottomrule
\end{tabularx}
% \vspace{-2mm}
\label{tab:performance_task}
\end{table}

% \subsection{Class-incremental Setting}

\begin{table}[H]
\centering
\caption{\textbf{Class incremental}: Performance metrics on different datasets for non-IID and IID settings with $5$ clients and $5$, $5$, $10$ tasks, respectively. \emph{Acc} and \emph{Forget} denote average classification accuracy and average forgetting.}
\begin{tabularx}{\textwidth}{l*{6}{>{\centering\arraybackslash}X}}
\toprule
 & \multicolumn{2}{c}{Split-CIFAR10} & \multicolumn{2}{c}{Split-CIFAR100}  \\
\cmidrule(lr){2-3} \cmidrule(lr){4-5}
 & Acc\tiny{($\uparrow$)} & Forget\tiny{($\downarrow$)} & Acc\tiny{($\uparrow$)} & Forget\tiny{($\downarrow$)} \\
\midrule

\multicolumn{5}{c}{\textbf{Non-IID Setting}} \\
\midrule
Finetune-FL  & $14.18\pm3.22$            & $\mf{2.31\pm4.82}$            & $16.24\pm1.03$ & $45.50\pm1.80$  \\
EWC-FL       & $12.92\pm2.53$            & $6.12\pm5.20$            & $17.25\pm1.29$ & $40.26\pm2.09$  \\
LwF-FL       & $12.45\pm3.54$            & $5.62\pm5.17$            & $22.38\pm3.66$ & $31.07\pm2.58$  \\
iCARL-FL     & $15.84\pm4.23$            & $50.75\pm12.95$          & $21.40\pm1.53$ & $32.46\pm2.23$  \\
TARGET       & $13.68\pm3.48$            & $9.73\pm7.04$            & $\mf{23.22\pm1.93}$ & $27.65\pm5.01$  \\
NCCL-FL      & $14.42\pm1.22$            & $76.93\pm2.46$           & $9.96\pm0.40$  & $47.76\pm1.05$  \\
FedTrack     & $12.21\pm1.04$            & $61.89\pm5.16$           & $3.53\pm0.18$  & $\mf{17.83\pm0.29}$  \\
C-FLAG       & $\mf{17.06\pm2.63}$                       & $43.04\pm13.27$                       & $13.94\pm1.13$             & $51.49\pm0.72$              \\
\midrule
\multicolumn{5}{c}{\textbf{IID Setting}} \\
\midrule
Finetune-FL  & $24.25\pm4.89$            & $55.54\pm13.09$          & $20.04\pm1.30$ & $65.67\pm1.61$  \\
EWC-FL       & $25.28\pm4.07$            & $57.04\pm5.44$           & $22.20\pm1.00$ & $62.69\pm1.35$  \\
LwF-FL       & $43.01\pm8.21$            & $32.07\pm16.54$          & $\mf{35.86\pm0.98}$ & $35.72\pm0.81$  \\
iCARL-FL     & $\mf{47.98\pm2.43}$            & $49.92\pm3.16$           & $24.13\pm1.72$ & $48.89\pm1.47$  \\
TARGET       & $20.08\pm3.12 $           & $21.56\pm10.38$          & $31.04\pm2.94$ & $26.75\pm5.61$  \\
NCCL-FL      & $16.91\pm0.50$            & $93.15\pm2.10$           & $12.40\pm0.39$ & $60.61\pm0.83$  \\
FedTrack     & $14.54\pm1.33$            & $81.48\pm3.64$           & $5.49\pm0.26$  & $\mf{22.02\pm1.73}$  \\
C-FLAG       & $26.37\pm3.80$   & $\mf{7.66\pm7.55}$                       & $31.68\pm1.54$             & $47.97\pm 2.64$              \\
\bottomrule
\end{tabularx}
% \vspace{-2mm}
\label{tab:performance_class}
\end{table}

Tables \ref{tab:performance_task} and \ref{tab:performance_class} present the average accuracy performance for both task-incremental and class-incremental settings, along with standard deviation values. Our findings consistently show that the proposed approach effectively mitigates catastrophic forgetting while achieving higher average accuracy than the baselines in the task-incremental setup. In the class-incremental setup, the results are comparable overall, with our approach outperforming the baselines in select cases. 
\begin{figure}[H]
    \centering
    \includegraphics[scale=0.3]{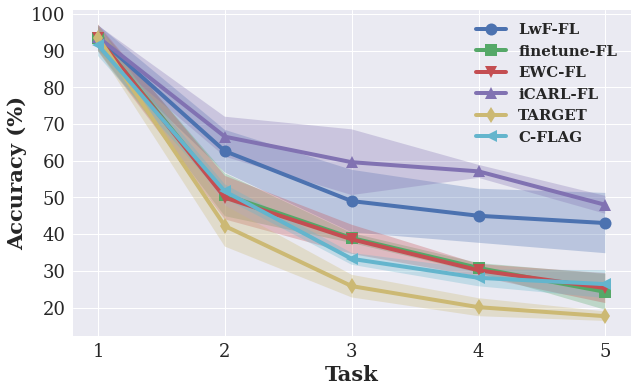}
    \includegraphics[scale=0.3]{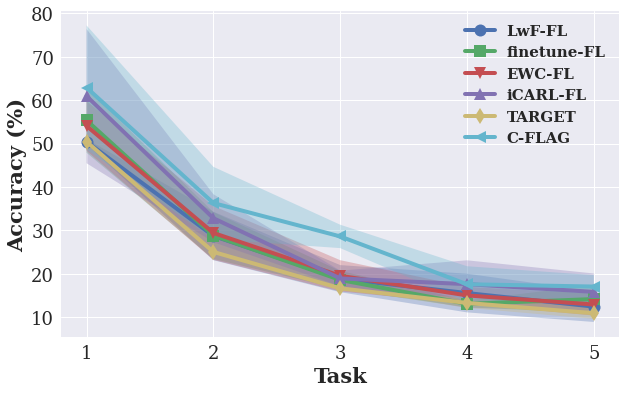}
\caption{ \centering Average accuracy plots on Split-CIFAR10 (Left: IID, Right: non-IID) in the class-incremental setup.}
    \label{fig:cifar100_avgaccuracy_class} 
\end{figure}
Fig.~\ref{fig:cifar100_avgaccuracy_class} illustrates the average accuracy performance of our method on the Split-CIFAR10 dataset in the class-incremental setup, demonstrating its superiority over the baselines in the non-IID setting. The baselines considered include class-incremental techniques such as TARGET, iCARL-FL, and LwF-FL.

\section{ADDITIONAL ABLATION STUDIES}

In this section, we provide additional ablation studies on the task-incremental setting.

\subsection{Varying Memory Sizes}

We provide the results for varying memory sizes ($m_0$) in Fig.~\ref{fig:memory_rebuttal}. We have analysed $6$ combinations of memory sample sizes on Split-CIFAR10 and Split-CIFAR100 datasets {{for both IID and non-IID settings.}} {Note that we have fixed the sampling size as half of the corresponding memory size.} It is observed that varying memory size affects the final average accuracy and forgetting. The general trend indicates an improvement in accuracy as we increase the size of the memory buffer. This trend is more evident in the non-IID partitioning scenario. The Split-CIFAR100 dataset shows greater sensitivity to memory size variations compared to Split-CIFAR10 in both IID and non-IID settings. This is likely due to the higher complexity and larger number of classes in CIFAR100, which demand more memory for mitigating catastrophic forgetting.

\begin{figure}[H]
    \centering
    \includegraphics[scale=0.28]{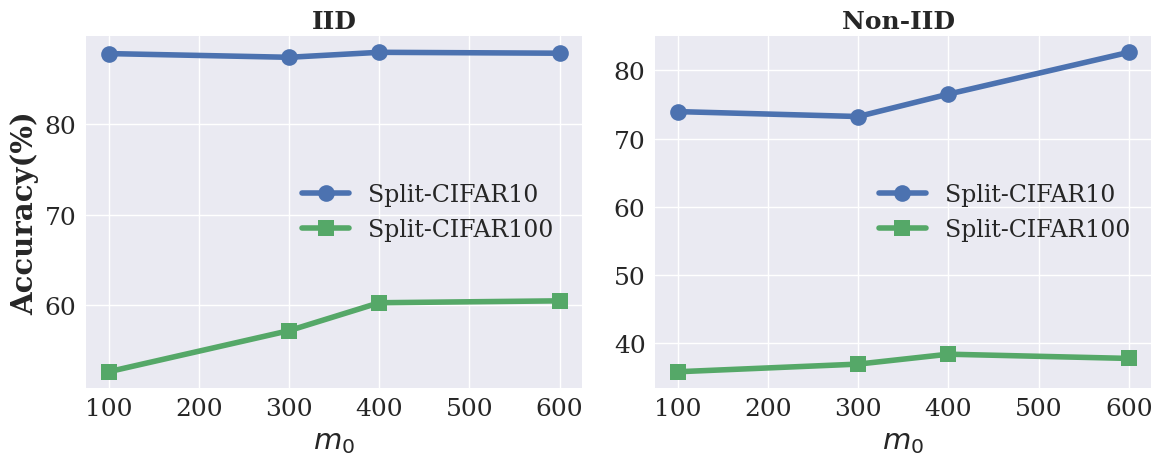}
    \includegraphics[scale=0.28]{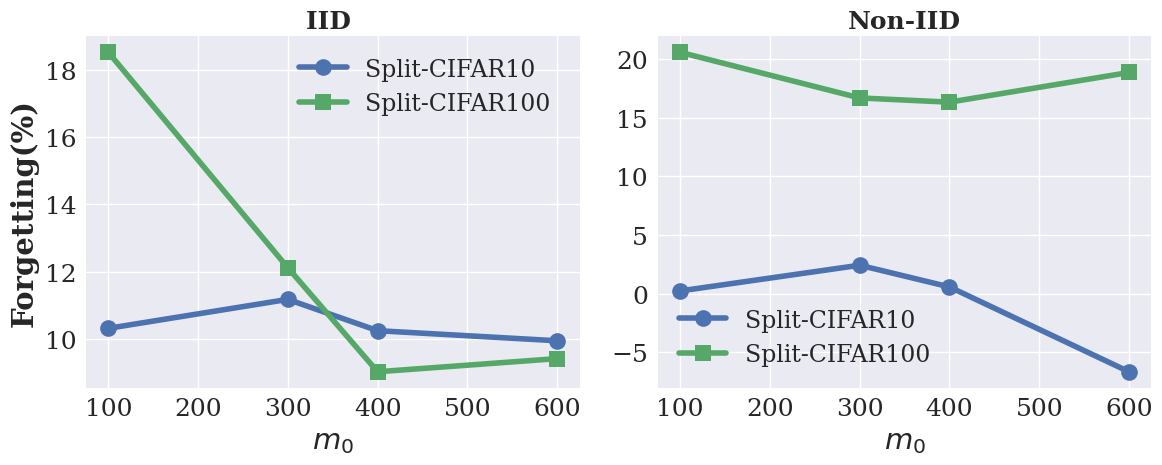}
\caption{\centering Average accuracies (left) and forgetting (right) for varying memory sizes ($m_0$).}
    \label{fig:memory_rebuttal}
    % \vspace{-4mm}
\end{figure}

From Fig.~\ref{fig:memory_rebuttal}, we observe that in the non-IID Split-CIFAR10 setting, forgetting becomes slightly negative for a larger memory size ($600$), indicating that a large memory size reinforces patterns that encourage improved transference. Note that the negative value of forgetting indicates that the current model performs extremely well. We refer to this phenomenon as transference in Sec.~\ref{sec:adapLR}.

\subsection{Varying Number of Clients}

\begin{table}[H]
\centering 
\caption{\centering Varying number of clients for the Split-CIFAR10 and Split-CIFAR100 datasets}
\begin{tabular}{|l|l|l|l|l|l|}
\hline
Dataset                  & Clients-50   & Clients-25   & Clients-15   & Clients-10   & Clients-5     \\ \hline
Split-CIFAR10 (IID)      & (90.67,2.39) & (90.77,3.53) & (91.05,4.28) & (90.52,5.22) & (87.89,10.25) \\ \hline
Split-CIFAR100 (Non-IID) & (35.20,2.28) & (32.60,4.27) & (35.48,6.38) & (35.95,8.58) & (38.43,16.33) \\ \hline
\end{tabular}
\label{table:cross_device}
\end{table}

Our proposed C-FLAG framework is agnostic to the number of clients and can be applied in cross-silo and cross-device settings. In the main manuscript, we have demonstrated the effect of different numbers of clients in Fig.~\ref{fig:memory_hetero} (left) of Sec.~\ref{sec:ExperiResults}, where we observed that the proposed technique is robust to the change in the number of clients. In other words, varying the number of clients does not deteriorate the continual learning abilities of the proposed method. Here, we are extending the analysis by accommodating $25$ and $50$ clients to demonstrate that C-FLAG can be implemented in cross-device settings as well. For instance, in the case of IID partitioning of the Split-CIFAR10 dataset and non-IID partitioning of the Split-CIFAR100 dataset, we obtain the average accuracies and forgetting as mentioned in Table~\ref{table:cross_device}.

\end{document}